\documentclass[10pt,journal,compsoc]{IEEEtran}
\usepackage{amssymb, amsmath, amsthm, graphicx, color, url, hyperref}
\usepackage{multirow, diagbox, makecell, subcaption}

\newcommand{\diag}{\operatorname{diag}}
\newcommand{\argmax}{\operatornamewithlimits{argmax}}
\newcommand{\argmin}{\operatornamewithlimits{argmin}}

\newtheorem{prop}{Proposition}

\usepackage[dvipsnames,table,xcdraw]{xcolor}

\ifCLASSOPTIONcompsoc
  \usepackage[nocompress]{cite}
\else
  \usepackage{cite}
\fi
\ifCLASSINFOpdf
\else
\fi

\begin{document}
%
\title{Free Energy Node Embedding via Generalized Skip-gram with Negative Sampling}
%
%
%
%

\author{Yu Zhu,~\IEEEmembership{Student Member,~IEEE,}
        Ananthram Swami,~\IEEEmembership{Fellow,~IEEE,}
        and Santiago Segarra,~\IEEEmembership{Senior Member,~IEEE}
\IEEEcompsocitemizethanks{\IEEEcompsocthanksitem Y. Zhu and S. Segarra are with the Dept. of ECE, Rice University.
\protect\\
Email: \{yz126, segarra\}@rice.edu.
\IEEEcompsocthanksitem A. Swami is with the U.S. Army’s DEVCOM Army Research Laboratory.
\protect\\
Email: ananthram.swami.civ@army.mil.
}
}

\IEEEtitleabstractindextext{%
\begin{abstract}
A widely established set of unsupervised node embedding methods can be interpreted as consisting of two distinctive steps: i) the definition of a similarity matrix based on the graph of interest followed by ii) an explicit or implicit factorization of such matrix.
Inspired by this viewpoint, we propose improvements in both steps of the framework. 
On the one hand, we propose to encode node similarities based on the free energy distance, which interpolates between the shortest path and the commute time distances, thus, providing an additional degree of flexibility. 
On the other hand, we propose a matrix factorization method based on a loss function that generalizes that of the skip-gram model with negative sampling to arbitrary similarity matrices.
Compared with factorizations based on the widely used $\ell_2$ loss, the proposed method can better preserve node pairs associated with higher similarity scores. 
Moreover, it can be easily implemented using advanced automatic differentiation toolkits and computed efficiently by leveraging GPU resources. 
Node clustering, node classification, and link prediction experiments on real-world datasets demonstrate the effectiveness of incorporating free-energy-based similarities as well as the proposed matrix factorization compared with state-of-the-art alternatives.
\end{abstract}

}

\maketitle

\IEEEdisplaynontitleabstractindextext

%
\IEEEpeerreviewmaketitle

\IEEEraisesectionheading{\section{Introduction}}

%
%
%
%

\IEEEPARstart{T}{he} field of node embedding (also called network representation learning) has attracted a lot of interest and achieved significant progress in recent years. 
The goal of node embedding is to encode nodes in a low-dimensional space so that similarity in the embedding space (e.g., the inner product of embedding vectors) approximates similarity in the original graph (e.g., homophily or structural equivalence)~\cite{survey1}. 
The produced low-dimensional embeddings can be used in several network analysis tasks such as node classification, link prediction, community detection, and visualization. 
They can also be used as the input to downstream graph neural networks with applications in varied fields including wireless communications~\cite{chowdhury2020unfolding, zhao2020distributed}, traffic prediction~\cite{jia2019graph, roddenberry2019hodgenet}, and neuroscience~\cite{ma2017multi, yue2019graph}. 
In this paper, we focus on the popular subclass of unsupervised node embedding, where only the graph structure is given and no extra information about the nodes is available. 

Motivated by the empirical success of some of the earlier methods~\cite{deepwalk, node2vec}, a number of approaches have been proposed for this problem over the last five years along with several survey papers that provide a taxonomy of these approaches under unified frameworks~\cite{survey1, survey2, survey3, survey4}. 
For example, the paper \cite{survey1} develops an encoder-decoder framework consisting of four methodological components (a pairwise proximity function, an encoder function, a decoder function, and a loss function), where different existing methods boil down to making distinct choices for these components.      

In this paper, we provide an alternative unified framework to encompass existing methods. 
The proposed framework consists of two steps: i) Encoding pairwise node similarities into a matrix $\mathbf{S}$, and ii) Factorizing $\mathbf{S}$ to obtain node embeddings. 
For the first step, the similarity matrix $\mathbf{S}$ can be directly computed from the adjacency matrix as in  the case of the personalized PageRank~\cite{hope}, or it can also be implicitly built through sampling procedures such as those based on random walks~\cite{deepwalk, node2vec}. 
For the second step, the matrix factorization can be implemented explicitly using variations of the singular value decomposition (SVD)~\cite{hope, grarep, netmf, AROPE, prone} or implicitly (as in \cite{deepwalk, node2vec}) by leveraging the framework of word2vec~\cite{word2vec_1, word2vec_2} or its derivatives, namely the skip-gram model together with negative sampling (NS) or noise contrastive estimation~\cite{nce1, nce2}. 
The proposed unified framework motivates two directions for designing better node embeddings by finding: i) More meaningful and expressive similarity matrices, and ii) Efficient matrix factorizations that preserve key aspects of those matrices.

In order to find better similarity measures, we rely on the free energy (FE) distance~\cite{fe_dist}. 
Using a single parameter, the FE distance interpolates between the shortest path (SP) and the commute time (CT) distances, which represent two extreme notions of distances on graphs. 
Indeed, while the SP distance only considers the path of minimum length between every pair of nodes, the CT distance takes into account all the (hitting) paths between nodes and has been used in various network analysis tasks including node embeddings~\cite{simnet}. 
However, it has been found that the CT distance might be misleading for large graphs~\cite{ct_problem, ct_problem2, ct_problem3}. 
This motivates us to consider the FE distance which is more flexible and can be tailored for graphs of different types and sizes. 
The FE distance has several beneficial properties -- it is a graph geodetic and can be computed in closed form --, and it has shown superior performance in capturing information useful for node clustering and classification compared to other measures that interpolate between the SP and the CT distances~\cite{fe_dist}. 

To improve the matrix factorization step, we modify the loss function in the framework of skip-gram with NS to accommodate arbitrary similarity matrices, including those with unbounded negative entries such as the similarity matrix associated with DeepWalk~\cite{deepwalk, netmf}.
There is a series of embedding methods including but not limited to~\cite{deepwalk,line,node2vec,walklets,app,metapath2vec,mines,verse} that leverage this framework while considering different similarity measures such as node2vec~\cite{node2vec} (considers higher-order proximities) and LINE~\cite{line} (considers the first-order and second-order proximities). 
An important work in this thread is VERSE~\cite{verse}, which generalizes the framework to arbitrary similarity matrices satisfying the condition that each row is amenable to being interpreted as a probability distribution. Our proposed matrix factorization method gets rid of this condition, thus further extending VERSE.
Moreover, we demonstrate that the proposed factorization can better preserve node pairs having higher similarity scores compared to SVD-related methods based on the $\ell_2$ loss function.
Finally, relying on the use of TensorFlow \cite{tf}, we provide an efficient implementation of the proposed factorization by leveraging GPU resources.

We propose a new node embedding method that combines these two improvements, namely the FE distance based similarities and the generalized matrix factorization. 
We evaluate the proposed method considering three downstream tasks (node clustering, node classification, and link prediction) on five real-world datasets and compare the attained performance with that of multiple state-of-the-art alternatives.
The experimental results validate the superiority of the proposed method. We also show that it can scale to large and sparse networks.

We summarize our contributions as follows.
\begin{itemize}
\item[1.] We put forth a new unified framework which can better organize and help understand a wide range of unsupervised node embedding methods.
\item[2.] We propose the use of free energy distance in computing the similarity matrix which is able to adapt to graphs of different sizes and types as well as different downstream tasks.
\item[3.] We develop a matrix factorization method that works for arbitrary similarity measures, better preserves node pairs having higher proximity scores, can be easily implemented using automatic differentiation toolkits, and can be computed efficiently by leveraging GPU resources.
\item[4.] We validate the effectiveness of combining the FE distance based similarities and the generalized matrix factorization via numerical experiments, and demonstrate the scalability of the proposed embedding method.
\end{itemize}

The rest of this article is structured as follows. 
The problem statement is given in Section~\ref{s:ps}. 
Section~\ref{s:exist} summarizes existing matrix factorization related approaches in a unified framework. 
The proposed node embedding method is discussed in Section~\ref{s:propose}, where the FE distance and the similarity measure based on it are introduced in Section~\ref{ss:fe} and the generalized matrix factorization method is discussed in Section~\ref{ss:vmf}. 
Numerical experiments are presented in Section~\ref{s:experiments} together with an ablation study and discussions on parameter sensitivity.
Algorithm scalability is discussed in Section~\ref{s:scalability}.
Closing remarks are included in Section~\ref{s:conclusion}.


\section{Problem Statement}\label{s:ps}

We consider an undirected and weighted graph $\mathcal{G}=(\mathcal{V},\mathcal{E},\mathbf{A})$ where $\mathcal{V}=\{v_1,v_2,\cdots,v_n\}$ is the node set of cardinality $n$, $\mathcal{E}$ is the edge set such that the unordered pair $(v_i,v_j)$ belongs to $\mathcal{E}$ iff there exists an edge between $v_i$ and $v_j$, and $\mathbf{A}\in\mathbb{R}^{n \times n}$ is the weighted adjacency matrix with the element $A_{ij}=A_{ji}>0$ indicating the edge weight if $(v_i,v_j)\in\mathcal{E}$ and $A_{ij}=A_{ji}=0$ otherwise. 
The edge weights can be interpreted as a measure of similarity between connected nodes.
We also introduce two other graph-related matrices that will be instrumental throughout the paper, namely, the (diagonal) degree matrix $\mathbf{D}=\diag(\mathbf{A1})$, where $\mathbf{1}$ refers to the all-ones vector, and the random-walk transition probability matrix $\mathbf{P}=\mathbf{D}^{-1}\mathbf{A}$.
With this notation in place, a formal problem statement follows.

\vspace{0.3em}

\noindent\textbf{Problem 1 (Unsupervised node embedding).} 
Given a graph $\mathcal{G}=(\mathcal{V},\mathcal{E},\mathbf{A})$, the goal is to learn a node embedding function $f: \mathcal{V} \to \mathbb{R}^d$ such that $d\ll n$ and $f$ preserves some proximity measure of interest on $\mathcal{G}$.


\section{Unified Framework for Existing Work}\label{s:exist}

We propose a unified framework to summarize and help better understand existing unsupervised node embedding methods. 
The proposed framework can be viewed as a two-step procedure.
The first step is to \emph{encode pairwise node similarities into a matrix $\mathbf{S}$}, whose entry $S_{ij}$ reflects the similarity between $v_i$ and $v_j$. 
This similarity is typically related to homophily, structural equivalence, or a mix of both. 
The second step is to \emph{factorize the constructed similarity matrix $\mathbf{S}$}. 
More precisely, the goal is to identify two matrices $\mathbf{U},\mathbf{V}\in\mathbb{R}^{n\times d}$ such that $\mathbf{S}\approx \mathbf{UV}^{\top}$, and then use the rows of $\mathbf{U}=[\mathbf{u}_1,\cdots,\mathbf{u}_n]^{\top}$ as node embeddings, i.e., $f(v_i) = \mathbf{u}_i$.

\begin{figure}
\centering
\includegraphics[width=0.48\textwidth]{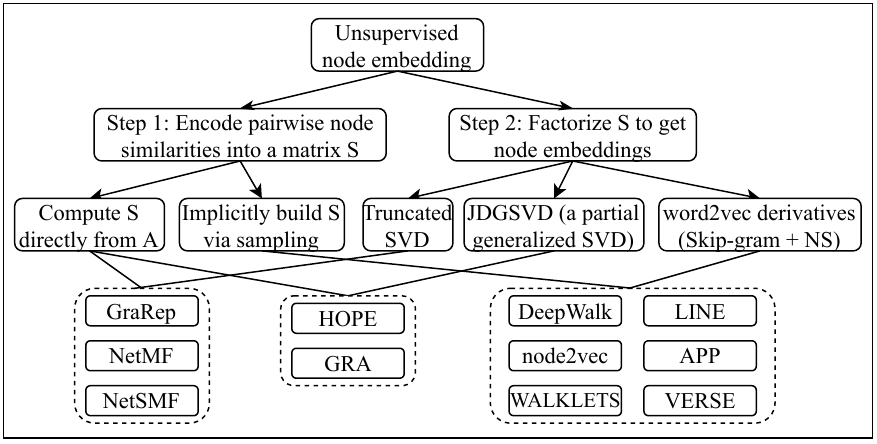}	
\caption{\small A unified framework for existing node embedding methods based on matrix factorizations.}
\label{fig:framework}
\end{figure}

We categorize related work according to their adopted methods in each step; see Fig.~\ref{fig:framework}. For the first step, existing works can be divided into the two following categories.
\begin{itemize}	
\item \emph{Explicitly compute the similarity matrix $\mathbf{S}$ from $\mathbf{A}$.}

\textbf{HOPE} \cite{hope} considers four common similarity measures (such as the Katz Index and the personalized PageRank) that can be readily computed from $\mathbf{A}$, whereas \textbf{GraRep} \cite{grarep} adopts positive $k$-step log probability matrices for different values of $k$.
\textbf{GRA} \cite{gra} proposes a Global Resource Allocation similarity measure, which is a variation over Katz that selectively assigns a high similarity score to pairs of nodes that have a large number of paths between them. 
\textbf{NetMF} \cite{netmf} and \textbf{NetSMF} \cite{netsmf} consider the similarity matrix  implicitly factorized in \cite{deepwalk} [cf.~\eqref{E:dw}].   

\item \emph{Implicitly build a similarity matrix via sampling.} 

There are two main sampling methods considered in the literature. The first sampling method is to generate random walks on the graph. In this context, for a node $v_i$, the nodes that appear within a window centered at $v_i$ are considered as the context of $v_i$. Two nodes are deemed as being similar to each other if they have similar contexts.
Indeed, \textbf{DeepWalk} \cite{deepwalk} was the first to propose to view random walks in graphs as `sentences', so that the problem of node embedding can be recast as a word embedding problem and the word2vec framework can be leveraged. 
Variations of this notion include \textbf{node2vec} \cite{node2vec}, which considers random walks with tunable drift parameters to enable smooth interpolation between breadth first search and depth first search, and
\textbf{WALKLETS} \cite{walklets}, which proposes to subsample random walks for capturing multi-scale relationships between nodes.

The second method involves sampling positive and negative node pairs directly.
For a pair of nodes $(v_i,v_j)$, we consider it as a positive sample if $v_j$ is likely to appear as the context of $v_i$ in random walks, otherwise we view it as a negative sample.
\textbf{LINE} \cite{line} draws positive samples based on second-order proximity weights whereas negative samples are drawn from a degree-dependent probability distribution.
\textbf{APP} \cite{app} uses the Monte-Carlo end-point sampling method \cite{SampleEndPoint} to generate positive samples in order to approximate the personalized PageRank value between every node pair and negative samples are drawn from a uniform distribution.
\textbf{VERSE} \cite{verse} generalizes the above ideas to any similarity matrix whose rows define a valid probability distribution. 
Hence, for a node $v_i$, another node $v_j$ is sampled from the distribution defined by the $i$th row of the similarity matrix for generating positive samples and negative samples are drawn from a uniform distribution.
\end{itemize}

\noindent
For the second step, the presented prior work falls into one of the following three categories.
\begin{itemize}
\item \emph{Truncated SVD.} 

SVD is a popular matrix factorization method for node embedding. 
In order to make the algorithms scalable for large graphs, existing works such as \textbf{GraRep} and \textbf{NetMF} set negative entries in the original similarity matrix to zero to construct a sparse matrix, and \textbf{NetSMF} also leverages a random-walk matrix-polynomial sparsification technique to further sparsity the similarity matrix. 
Efficient truncated SVD methods \cite{svd1, svd2, svd4} can be applied to these sparse matrices.  
All these methods deal with a fundamental trade-off where sparser matrices are computationally preferable but also discard more information from the original similarity matrix.

\item \emph{A partial generalized SVD (JDGSVD).} 

It has been proposed in \textbf{HOPE} that if the similarity matrix can be written in the form of $\mathbf{S}=\mathbf{M}_g^{-1}\mathbf{M}_l$ where $\mathbf{M}_g$ and $\mathbf{M}_l$ are both matrix polynomials, the original SVD problem can be transformed into a generalized SVD problem for fast computation \cite{svd3, svd5}. \textbf{GRA} also adopts this method.
\item \emph{word2vec derivatives.} 

Existing works including \textbf{DeepWalk}, \textbf{node2vec}, \textbf{WALKLETS}, \textbf{LINE}, \textbf{APP}, and \textbf{VERSE} leverage the framework of word2vec or its derivatives. 
Most of them adopt the skip-gram model proposed in word2vec together with noise contrastive estimation or NS for efficient computation. 
An iterative process is adopted where, in each step, a positive node pair $(v_i,v_j)$ and one or multiple negative pairs $\{(v_i,v_{j'})\}$ are sampled, and gradient ascent method is applied to optimize the objective $\log\sigma(\mathbf{u}_i^\top \mathbf{v}_j) + \sum_{j'}\log\sigma(-\mathbf{u}_i^\top \mathbf{v}_{j'})$ where $\sigma(x)=1/(1+e^{-x})$ is the sigmoid function. 
In this way, the sought embeddings $f(v_i) = \mathbf{u}_i$ are found; we expand on this idea in Section~\ref{ss:vmf}. 
Moreover, it has been shown that this process implicitly factorizes interpretable similarity matrices \cite{word_mf, netmf}.
\end{itemize}


\section{Proposed Method}\label{s:propose}

We leverage the framework introduced in Section~\ref{s:exist} to describe our proposed modifications to both steps in Sections~\ref{ss:fe} and~\ref{ss:vmf}, respectively.

\subsection{Similarity based on the free energy distance}\label{ss:fe}

In defining the similarity matrix $\mathbf{S}$, we consider the FE distance on graphs.
In order to define this distance, let us first introduce the following notions. 
We assume that every pair of nodes $(v_i,v_j)$ is associated with a transition cost $C_{ij}$ such that $0<C_{ij}<\infty$ if $(v_i,v_j)\in\mathcal{E}$ and $C_{ij}=\infty$ otherwise.
Denote by $\mathcal{P}_{st}$ the set of \emph{hitting} (or \emph{absorbing}) paths from node $s$ to node $t$, where node $t$ can only appear in a hitting path as the terminal node. 
For a path $p=(v^0=s,v^1,\cdots,v^{\ell-1},v^{\ell}=t)\in\mathcal{P}_{st}$, its cost and \emph{reference probability} are respectively given by $c(p)=C_{sv^1}+C_{v^1v^2}+\cdots+C_{v^{\ell-1}t}$ and $\mathbb{P}^{\mathrm{ref}}_{st}(p)=P_{sv^1}P_{v^1v^2}\cdots P_{v^{\ell-1}t}$ where, we recall, the transition probability matrix $\mathbf{P}$ was defined in Section~\ref{s:ps}.

The edge costs $C_{ij}$ can be defined according to attributes of the edges or their endpoints in order to bias the probability distribution of selecting a path~\cite{franccoisse2017bag}.
For example, the cost of jumping to a node can be set proportional to its degree in order to penalize paths visiting hubs.
When there are no natural costs assigned to the edges, it is common to set $C_{ij}=1/A_{ij}$ (the edge weights and the costs are analogous to conductance and resistance in an electric network, respectively)~\cite{fe_dist, franccoisse2017bag, devooght2014random, fe_fast}.
We adopt this convention in the paper.

The FE distance is proposed as a trade-off between the SP and the CT distances. 
The SP distance between two nodes $s$ and $t$ is the minimum cost of a path between the two nodes, i.e., $\Delta^{\mathrm{SP}}_{st}=\min_{p\in\mathcal{P}_{st}} c(p)$. 
It only considers the minimum cost path between these two nodes and does not integrate the information of other paths, thus it cannot capture the global structure of the graph. 
In many practical problems, for a constant SP distance, nodes should be considered to be closer to each other if they are connected by more paths~\cite{franccoisse2017bag}. 
In other words, the SP distance ignores the number or density of paths between two nodes.
The CT distance between two nodes is defined as $\Delta^{\mathrm{CT}}_{st}=H_{st}+H_{ts}$ where the expected hitting time $H_{st}=\sum_{p\in\mathcal{P}_{st}}\mathbb{P}^{\mathrm{ref}}_{st}(p)c(p)$ is the expected cost that it takes a random walk to travel from $s$ to $t$ for the first time. 
It has been shown in~\cite{ct_problem3} that as the graph size increases, $H_{st}$ approaches the reciprocal of the degree of node $t$ up to some constant factor. 
Hence, $\Delta^{\mathrm{CT}}_{st}$ becomes only dependent on trivial local properties of the graph, i.e., the degrees of $s$ and $t$. 
An intuitive explanation of this phenomenon is that in large graphs, it takes a long time for the random walk to travel through a substantial part of the graph. 
Before the random walk comes close to the target node, it has already `forgotten' the starting node. 
Therefore, the hitting time $H_{st}$ becomes only dependent on the degree of the target node $t$, which can be understood as the likelihood that the random walk hits $t$ once it is in its neighborhood~\cite{ct_problem3}.
By introducing an extra parameter, the FE distance is able to interpolate between the SP and the CT distances and overcome their drawbacks. 

The FE distance is obtained as the solution to the following optimization problem over the probability distributions on the paths in $\mathcal{P}_{st}$, 
%
%
\begin{align}\label{E:FE_opt}
\mathbb{P}^{\mathrm{FE}}	_{st} &= \argmin_{\mathbb{P}_{st}}  \sum_{p\in\mathcal{P}_{st}} \mathbb{P}_{st}(p)c(p)+\frac{1}{\eta}D_{\mathrm{KL}}(\mathbb{P}_{st} \parallel \mathbb{P}^{\mathrm{ref}}_{st}), \nonumber \\
&\mathrm{s.t.} \sum_{p\in\mathcal{P}_{st}}\mathbb{P}_{st}(p)=1,
\end{align}
where $\eta>0$ is a tunable parameter and 
$$
D_{\mathrm{KL}}(\mathbb{P}_{st} \parallel \mathbb{P}^{\mathrm{ref}}_{st}) = \sum_{p\in\mathcal{P}_{st}} \mathbb{P}_{st}(p)\log\frac{\mathbb{P}_{st}(p)}{\mathbb{P}^{\mathrm{ref}}_{st}(p)}
$$
is the K\"{u}llback-Leibler divergence (also called relative entropy) from the reference probability distribution $\mathbb{P}^{\mathrm{ref}}_{st}$ to $\mathbb{P}_{st}$. 
The measure $D_{\mathrm{KL}}(\mathbb{P}_{st} \parallel \mathbb{P}^{\mathrm{ref}}_{st})$ quantifies how similar $\mathbb{P}_{st}$ is to $\mathbb{P}^{\mathrm{ref}}_{st}$.
It is non-negative, and equals zero if and only if $\mathbb{P}_{st}$ and $\mathbb{P}^{\mathrm{ref}}_{st}$ are the same distribution.

Denote the objective function of~\eqref{E:FE_opt} by $\phi(\mathbb{P}_{st})$, which is the free energy of a thermodynamical system with temperature $1/\eta$ and state transition probabilities $\mathbb{P}_{st}$~\cite{fe_dist, peliti2011statistical}. 
The FE distance between node $s$ and node $t$ is defined as $\Delta^{\mathrm{FE}}_{st}=(\phi(\mathbb{P}^{\mathrm{FE}}_{st})+\phi(\mathbb{P}^{\mathrm{FE}}_{ts}))/2$. When $\eta\to 0^{+}$, the second term in the objective dominates, thus $\mathbb{P}^{\mathrm{FE}}_{st}\to\mathbb{P}^{\mathrm{ref}}_{st}$ and the FE distance converges to the CT distance (divided by 2); when $\eta\to\infty$, the first term in the objective dominates, hence $\mathbb{P}^{\mathrm{FE}}_{st}$ will be more and more peaked around the shortest path and the FE distance converges to the SP distance. 


In fact, these three types of distances (CT, SP, and FE) can be understood under the bag-of-paths framework~\cite{franccoisse2017bag}, in which a probability distribution $\mathbb{P}_{st}$ is assigned over all hitting paths between any node pair $(s,t)$.
The SP distance assigns probability $1$ to the minimum cost path and ignores other paths; the CT distance considers all the hitting paths and the probability distribution is determined by the natural random walk process (corresponding to $\mathbb{P}^{\mathrm{ref}}_{st}$). 
Compared with the SP distance, the FE distance also considers suboptimal paths apart from the minimum cost path.
Compared with the CT distance, the parameter $\eta$ in the FE distance can adjust the probability distribution to find a trade-off between exploitation and exploration.
As $\eta$ increases, low-cost paths will be assigned a higher probability while high-cost paths will have a lower probability of being sampled from the bag~\cite{franccoisse2017bag}.
Hence, the aforementioned issues existing in the SP and the CT distances can be alleviated.

There are a few other measures that generalize the SP and the CT distances including the SP-CT combination distance (a simple convex combination of these two distances), the logarithmic forest distance~\cite{chebotarev2011class}, the $p$-resistance distance~\cite{alamgir2011phase} and the randomized shortest path (RSP) dissimilarity~\cite{yen2008family, saerens2009randomized}.
The FE distance and the RSP dissimilarity do not always lie between the SP and the CT distances for intermediate values of $\eta$, thus offering more flexibility than the other mentioned measures (cf. Fig. 2 in~\cite{fe_dist}).
We favor the use of the FE distance because, unlike the RSP dissimilarity, it satisfies the triangle inequality and, thus, defines a bona fide metric.
A detailed comparison of these measures can be found in~\cite{fe_dist}.

We denote the FE distance matrix by $\pmb{\Delta}_\eta^{\mathrm{FE}}$, where we have made explicit the dependence on the tunable parameter $\eta$.  
To compute the FE distances on a graph, we do not have to solve the optimization problem~\eqref{E:FE_opt} for every node pair $(s,t)$. 
The entire matrix $\pmb{\Delta}_\eta^{\mathrm{FE}}$ can be computed in closed form using the algorithm proposed in~\cite{fe_dist}, the computational complexity of which is $\mathcal{O}(n^3)$.
Moreover, an efficient method has been proposed in~\cite{fe_fast}, which reduces the complexity to $\mathcal{O}(n|\mathcal{E}|)$ and thus is able to scale on large and sparse graphs.
A detailed discussion regarding the issue of scalability can be found in Section~\ref{s:scalability}.

We convert the distance matrix $\pmb{\Delta}_\eta^{\mathrm{FE}}$ into a similarity matrix via
\begin{equation}\label{e:fe_sim}
\mathbf{S}_{\eta,b,\gamma}^{\mathrm{FE}} = \gamma (-\pmb{\Delta}_\eta^{\mathrm{FE}} + b),
\end{equation}
where parameters $b$ and $\gamma$ respectively control the shift and the scale of the similarity values.  
We would like to clarify that we do not require $\mathbf{S}_{\eta,b,\gamma}^{\mathrm{FE}}$ to be a non-negative matrix, and the choice of parameter $b$ is not motivated by this concern.
We have further studied the influence of different choices of parameters $\eta$, $b$ and $\gamma$ on performance in Section~\ref{Ss:parameter_sensitivity}.

\subsection{Generalized matrix factorization}\label{ss:vmf}

Based on the FE similarity matrix in~\eqref{e:fe_sim}, we propose the following solution to Problem 1
\begin{align}\label{e:fe_embed}
&f_{\eta,b,\gamma,d}^{\mathrm{FE}} (v_i) = \mathbf{u}^*_i, \qquad \text{with} \\
& \{\mathbf{u}^*_i\} \! = \! \argmax_{\{\mathbf{u}_i\in\mathbb{R}^d\}_{i=1}^n} \sum_{i \neq j} e^{[\mathbf{S}_{\eta,b,\gamma}^{\mathrm{FE}}]_{ij}} \log\sigma(\mathbf{u}_i^\top\mathbf{u}_j) + \log\sigma(-\mathbf{u}_i^\top\mathbf{u}_j), \nonumber
\end{align} 
where, we recall, $\sigma(\cdot)$ is the sigmoid function and $[\mathbf{S}_{\eta,b,\gamma}^{\mathrm{FE}}]_{ij}$ denotes the $(i,j)$th entry of $\mathbf{S}_{\eta,b,\gamma}^{\mathrm{FE}}$.
The embedding $f_{\eta,b,\gamma,d}^{\mathrm{FE}} (v_i)$ depends on four hyperparameters $\eta,b,\gamma,d$ where $\eta$ is inherited from the FE distance in~\eqref{E:FE_opt}, $b$ and $\gamma$ help control the shift and the scale of the similarity matrix in~\eqref{e:fe_sim}, and $d$ is the embedding dimension.
By forming a matrix $\mathbf{U}^* \in \mathbb{R}^{n \times d}$ whose rows are $\mathbf{u}^{*\top}_i$, we obtain an implicit factorization of $ \mathbf{S}_{\eta,b,\gamma}^{\mathrm{FE}} \approx \mathbf{U}^* \mathbf{U}^{*\top}$.
We show that this is the case through the following more general result.

\begin{prop}\label{P:main_result}
Consider the following optimization problem 
\begin{align}\label{E:obj}
& \mathbf{U}^*, \mathbf{V}^*  = \argmax_{\mathbf{U}, \mathbf{V} \in \mathbb{R}^{n \times d}} \psi (\mathbf{U}, \mathbf{V}) \qquad \text{with} \\
& \psi (\mathbf{U}, \mathbf{V}) = \sum_{1\leq i , j\leq n} S^{+}_{ij}\log\sigma(\mathbf{u}_i^\top\mathbf{v}_j) + S^{-}_{ij}\log\sigma(-\mathbf{u}_i^\top \mathbf{v}_j), \nonumber
\end{align}
where $\mathbf{U}=[\mathbf{u}_1,\cdots,\mathbf{u}_n]^{\top}$ and $\mathbf{V}=[\mathbf{v}_1,\cdots,\mathbf{v}_n]^{\top}$.
Then, $\mathbf{S} \approx \mathbf{U}^* \mathbf{V}^{* \top}$ where the entries of $\mathbf{S}$ are given by $S_{ij}=\log(S^{+}_{ij}/S^{-}_{ij})$ and, for large enough $d$, we have that $\mathbf{S} = \mathbf{U}^* \mathbf{V}^{* \top}$.
\end{prop}
\begin{proof}
Inspired by \cite{word_mf, Hilbert_MLE}, we take the derivative of the objective $\psi (\mathbf{U}, \mathbf{V})$ with respect to $\mathbf{u}_i^\top\mathbf{v}_j$ and get
\begin{align*}\label{E:derivative}
\frac{\partial\,\psi (\mathbf{U}, \mathbf{V})}{\partial\,\mathbf{u}_i^\top\mathbf{v}_j} &= S^{+}_{ij} [1-\sigma(\mathbf{u}_i^\top\mathbf{v}_j)] - S^{-}_{ij} \sigma(\mathbf{u}_i^\top\mathbf{v}_j) \nonumber \\
&= (S^{+}_{ij}+S^{-}_{ij}) \left[ \sigma\left( \log(S^{+}_{ij}/S^{-}_{ij}) \right) - \sigma(\mathbf{u}_i^\top\mathbf{v}_j) \right]\!,
\end{align*}
where the second equality follows from the identity $a\equiv(a+b)\sigma(\log(a/b))$. 
By setting the derivative to zero, we get $\mathbf{u}_i^\top\mathbf{v}_j = \log(S^{+}_{ij}/S^{-}_{ij})$. Therefore, problem \eqref{E:obj} implicitly factorizes the matrix whose $(i,j)$th entry is $\log(S^{+}_{ij}/S^{-}_{ij})$. Finally, for sufficiently large $d$, each inner product $\mathbf{u}_i^\top\mathbf{v}_j$ can be set independently and $\mathbf{S}$ can be reconstructed perfectly.
\end{proof}

\begin{figure}
\centering
\includegraphics[width=0.48\textwidth]{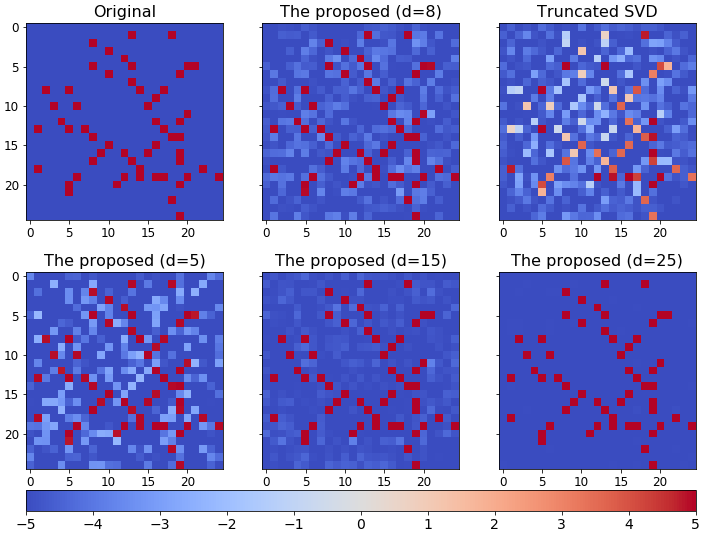}	
\caption{\small (top) One example matrix and the reconstructed ones obtained by GMF and truncated SVD. The embedding dimension $d$ is set to $8$. The proposed method can better preserve the edges (red). (bottom) The reconstructed matrices obtained by the proposed method with $d=5,15,25$ respectively. The recovery accuracy increases as $d$ increases, and a perfect recovery is achieved when $d$ is large enough.}
\label{fig:example}
\end{figure}

By setting $S^{+}_{ij} = e^{[\mathbf{S}_{\eta,b,\gamma}^{\mathrm{FE}}]_{ij}}$ and $S^{-}_{ij} = 1$ in Proposition~\ref{P:main_result}, it follows that~\eqref{e:fe_embed} is implicitly factorizing $\mathbf{S}_{\eta,b,\gamma}^{\mathrm{FE}}$.
However, note that other choices of $S^{+}_{ij}$ and $S^{-}_{ij}$ can lead to the same implicit factorization, and different choices might result in different embedding qualities.
From an empirical standpoint, two modifications can be observed when comparing~\eqref{e:fe_embed} with the generic decomposition in~\eqref{E:obj}.
First, in~\eqref{e:fe_embed} we exclude the diagonal entries ($i=j$) in the sum, since we are interested in encoding the similarities between (distinct) nodes.
Second, we set $\mathbf{V}=\mathbf{U}$ since we are factorizing a symmetric similarity matrix associated with an undirected graph. 
We refer to~\eqref{E:obj} as GMF for generalized matrix factorization and~\eqref{e:fe_embed} as GMF-FE.


\begin{table}
\caption{\small Summary of datasets considered.}\label{table-datasets}
\begin{center}
\begin{tabular}{c c c c c c c} 
 \Xhline{0.8pt}
 Dataset  & $|\mathcal{V}|$ & $|\mathcal{E}|$ & \#Labels & Multi-label  \\ 
 \hline
 CiteSeer  & $2,110$ & $3,668$ & $6$ & False  \\
 Cora  & $2,485$ & $5,069$ & $7$ & False  \\
 PPI  & $3,852$ & $37,841$ & $50$ & True  \\ 
 BlogCatalog  & $10,312$ & $333,983$ & $39$ & True  \\ 
 CoCit & $42,452$ & $194,410$ & $15$ & False \\
 \Xhline{0.8pt}
\end{tabular}
\end{center}
\end{table}

\begin{table}
\caption{\small Choices of operators for embedding node pairs.}\label{table-operator}
\begin{center}
\begin{tabular}{l l} 
 \Xhline{0.8pt}
 Operator     &   Definition  \\ 
 \hline
 Average      &   $\mathbf{e}_{ij}[k]=(\mathbf{u}_i[k] + \mathbf{u}_j[k])/2$   \\
 Hadamard     &   $\mathbf{e}_{ij}[k]=\mathbf{u}_i[k]\cdot \mathbf{u}_j[k]$   \\
 Weighted-L1  &   $\mathbf{e}_{ij}[k]=\left|\mathbf{u}_i[k]-\mathbf{u}_j[k]\right|$   \\ 
 Weighted-L2  &   $\mathbf{e}_{ij}[k]=\left|\mathbf{u}_i[k]-\mathbf{u}_j[k]\right|^2$   \\ 
 \Xhline{0.8pt}
\end{tabular}
\end{center}
\end{table}

From a more general perspective, the proposed GMF can be used to implicitly factorize arbitrary similarity matrices.
In this way, if the FE similarity is not the right choice for a specific application, the practitioner can still use GMF for their similarity measure of preference.
GMF has the following advantages. 
First, it generalizes existing methods under the framework of skip-gram plus NS to arbitrary similarity matrices. 
Especially, compared with VERSE, which requires each row of the similarity matrix to be a valid probability distribution, GMF does not have such limitations and even works for matrices with unbounded negative entries such as the implicit similarity of DeepWalk [cf.~\eqref{E:dw}]. 
Second, it can be easily implemented in automatic differentiation toolkits such as TensorFlow~\cite{tf} and PyTorch~\cite{pytorch}, facilitating the use of state-of-the-art optimization algorithms and GPU resources for efficient computation.

To build intuition about the difference between GMF and more traditional SVD-related methods, consider the $\ell_2$ loss typically considered in this latter class,
\begin{equation}\label{E:l2}
\|\mathbf{S}-\mathbf{U}\mathbf{V}^{\top}\|_2^2 = \sum_{1\leq i,j\leq n} (S_{ij} - \mathbf{u}_i^\top\mathbf{v}_j)^2.
\end{equation}
GMF adopts the loss~\eqref{E:obj} in the form of cross-entropy which is suitable for the task of edge detection, while the $\ell_2$ loss~\eqref{E:l2} is a natural choice for regression~\cite{hamilton2020graph}.
Moreover,~\eqref{E:l2} treats all entries in $\mathbf{S}$ (all node pairs) equally in an unweighted way. 
By contrast, \eqref{E:obj} implicitly allocates larger weights to the entries associated with larger $S^{+}_{ij}$ when we set identical values for all $S^{-}_{ij}$. 
In other words, it tends to preserve node pairs having higher proximity values. 
This is especially important for networks with positive bias (such as protein-protein interaction networks) where the presence of an edge is a strong indication about the relation between two nodes but the absence of an edge is a weaker indication that the nodes are not directly related. 
To illustrate this effect, we consider an unweighted Erd\H{o}s-R\'{e}nyi random graph with $25$ nodes and edge-formation probability $0.1$. 
We adopt the similarity matrix $\mathbf{S}$ where $S_{ij}=5$ if $A_{ij}=1$ and $S_{ij}=-5$ otherwise.  
We set $d=8$ and use both GMF and the truncated SVD to factorize $\mathbf{S}$ and reconstruct it\footnote{To use \eqref{E:obj} to factorize the matrix $\mathbf{S}$ considered in this example, we set $S^{+}_{ij}=e^{S_{ij}}$ and $S^{-}_{ij}=1$.}; see the first row of Fig.~\ref{fig:example}. 
We can see that GMF can better discriminate the edges present in the original graph from all node pairs, and incurs most of its error in the reconstruction of the low similarity (less relevant) pairs.
The second row of Fig.~\ref{fig:example} presents the reconstructed matrices obtained by GMF with different embedding dimensions. 
It can be observed that the recovery quality improves as $d$ increases. 
In addition, a perfect reconstruction is obtained when $d$ is large enough, which verifies Proposition~\ref{P:main_result}.


\begin{figure}
\centering
\includegraphics[scale=0.27]{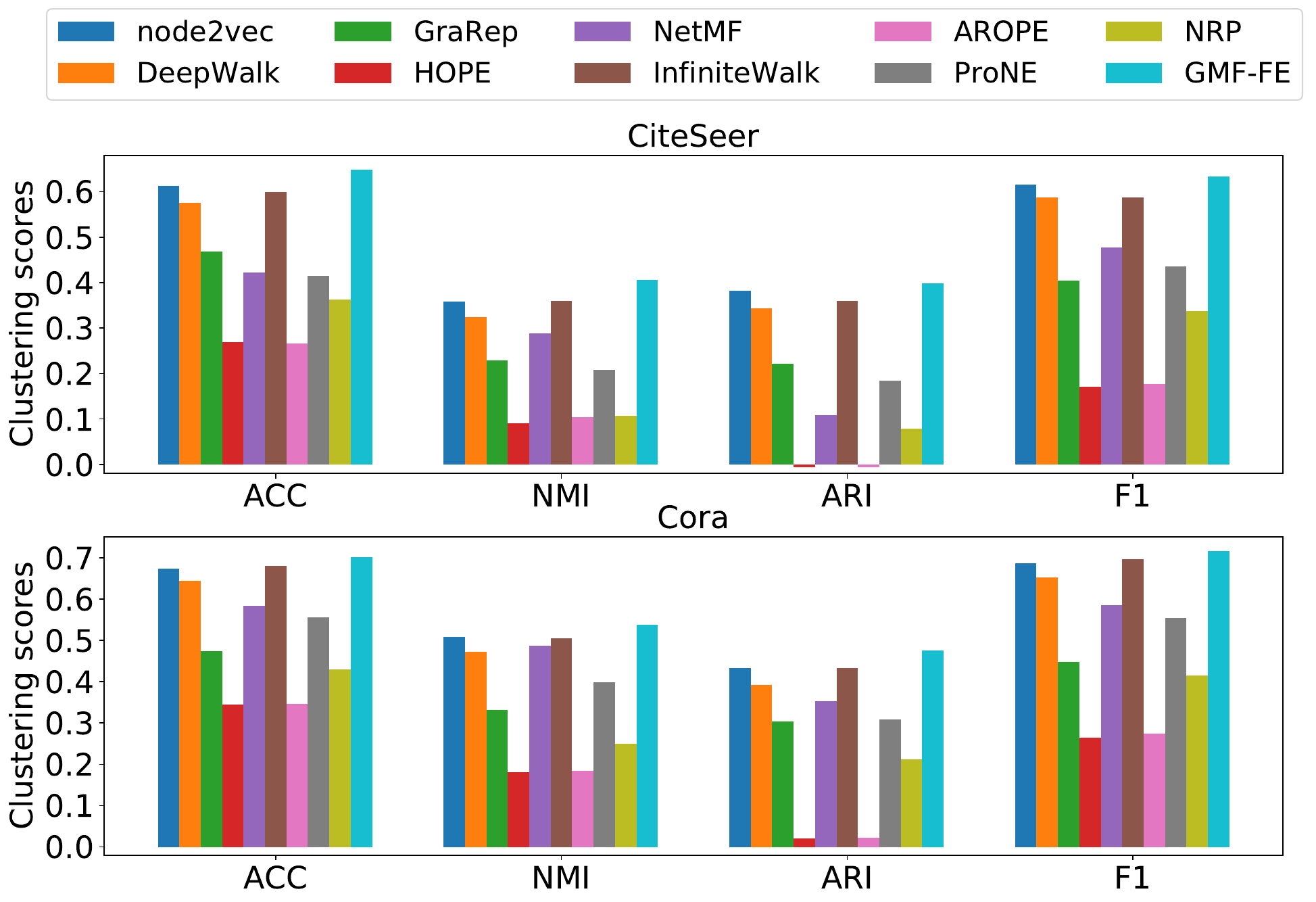}
\caption{\small Node clustering results for single-label datasets CiteSeer and Cora. The proposed GMF-FE method outperforms the alternatives on the four performance metrics considered: clustering accuracy (ACC), normalized mutual information (NMI), adjusted Rand index (ARI) and weighted-F1 score (F1).}
\label{fig:cluster}
\end{figure}

\section{Experiments}\label{s:experiments}

We conduct numerical experiments\footnote{The code needed to replicate numerical experiments presented in this paper can be found at \url{https://github.com/yuzhu2019/fe_embed}.} on real-world datasets to answer the following questions:
\begin{itemize}
\item[Q1] Does the proposed GMF-FE perform better than existing matrix factorization related methods?
\item[Q2] What is the impact of the two steps (similarity computation and matrix factorization) in GMF-FE if considered separately?
\item[Q3] How do the parameters $\eta, b, \gamma$ as well as the embedding dimension $d$ affect the embedding quality?
\end{itemize}


We answer Q1-Q3 in Sections~\ref{Ss:performance_comparison}-\ref{Ss:parameter_sensitivity}, respectively.
Before that, we present the datasets, baselines, parameter settings, and downstream tasks considered.

\vspace{1mm}
\noindent\textbf{Datasets}

\noindent We consider the following well-established networks in this section. 
CiteSeer~\cite{citeseer} and Cora~\cite{cora} are citation networks, where the nodes represent publications and the edges are citations between them. 
Each publication is assigned one label indicating its topic. 
PPI~\cite{node2vec} is a subgraph of the protein-protein interaction network for Homo sapiens. 
The labels represent biological states.
BlogCatalog~\cite{BlogCatalog} is a social network, where the nodes are blog authors on the BlogCatalog website and the edges indicate their relationships. 
Each blogger is assigned one or multiple labels that indicate the topic categories provided by the author.
We remove self-loops, convert directed graphs to undirected ones, and keep the largest connected component for each network. 
These networks are all unweighted.
Descriptive statistics of these networks after pre-processing are summarized in Table~\ref{table-datasets}. 

\begin{figure*}[!htbp]
\centering
\includegraphics[scale=0.42]{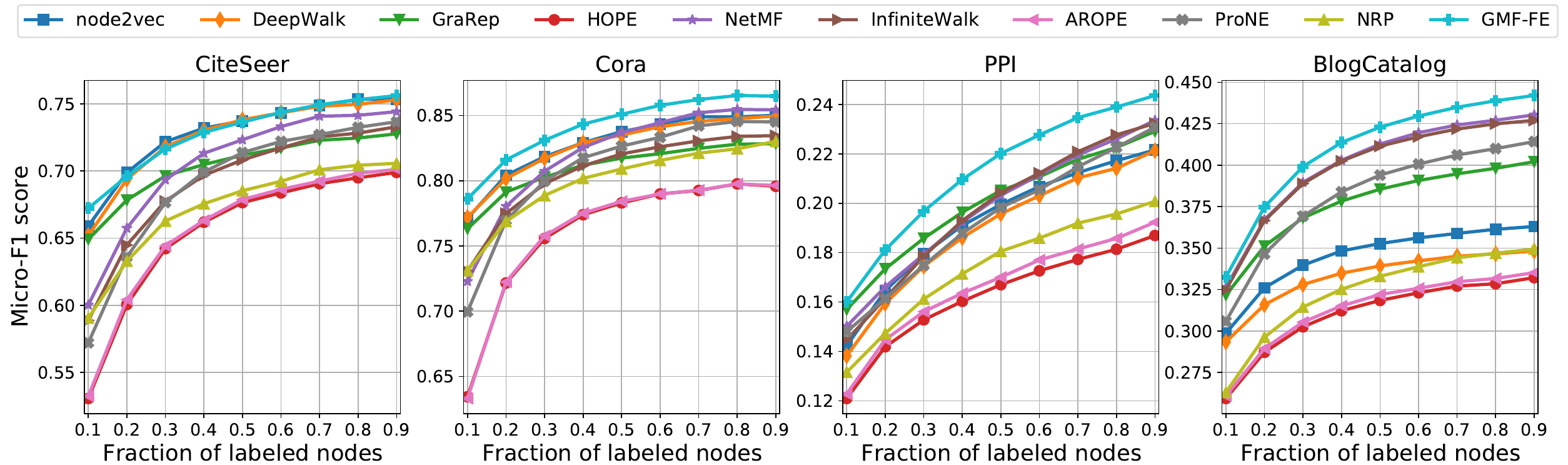}
\caption{\small Node classification results for different methods and datasets as a function of the fraction of labeled nodes. The proposed GMF-FE achieves state-of-the-art performance for CiteSeer and improves upon all baselines for the other three datasets across all fractions of labeled nodes.} 
\label{fig:nc}
\end{figure*}

\vspace{1mm}
\noindent\textbf{Baseline methods}

\noindent We compare the proposed approach with the following node embedding methods, which are representatives of matrix factorization related methods.
The parameter settings (described below) were selected in accordance with the guidelines in the corresponding publications.

\begin{itemize}
\item \textbf{node2vec}~\cite{node2vec}. We set the number of walks per node to $10$, the walk length to $80$, the window size to $10$, the number of negative samples to $1$, and the optimization is run for a single epoch. We obtain the best in-out and return parameters via a grid search over the set $\{0.25, 0.5, 1, 2, 4\}$.  
\item \textbf{DeepWalk}~\cite{deepwalk}. We adopt NS instead of hierarchical softmax, thus, DeepWalk can be viewed as a special case of node2vec when both of the in-out and return parameters are set to $1$. 
\item \textbf{GraRep}~\cite{grarep}. We set the maximum transition step to $4$ and set the log shifted factor to $1$ divided by the graph size.
\item \textbf{HOPE}~\cite{hope}. We adopt Katz as the similarity matrix and set the decay parameter to $0.95$ divided by the spectral radius of $\mathbf{A}$. 
For consistency with the other methods, we adopt the source embedding vectors as the final node embeddings and ignore the target embedding vectors.
\item \textbf{NetMF}~\cite{netmf}. Recall here the similarity matrix implicitly factorized in DeepWalk as derived in \cite{netmf}
\begin{equation}\label{E:dw}
\mathbf{S}^{\mathrm{DW}} = \log\left( \frac{\mathrm{vol}_\mathcal{G}}{bT} \sum\nolimits_{t=1}^T \mathbf{P}^t \mathbf{D}^{-1} \right),	
\end{equation}
where $\mathrm{vol}_{\mathcal{G}}\!=\!\sum_{1\leq i,j\leq n}A_{ij}$ is the volume of the graph $\mathcal{G}$, $T$ is the context window size, and $b$ is the number of negative samples in the skip-gram. 
This method applies the truncated SVD to the matrix obtained by setting the negative entries of \eqref{E:dw} to zero. 
We set $T=10$, $b=1$, and the number of top eigenpairs $h$ to $256$ [cf. Algorithm $4$ in \cite{netmf}]. 
\item \textbf{InfiniteWalk}~\cite{chanpuriya2020infinitewalk}. We set window size to $10$.
\item \textbf{AROPE}~\cite{AROPE}. This method considers similarity matrices in the form of polynomials of the adjacency matrix, namely $\mathbf{S}=\sum_{i=1}^q w_i\mathbf{A}^i$. We adopt one of the default setting where $q=3$ and $w_1=1, w_2=0.1, w_3=0.01$.
\item \textbf{ProNE}~\cite{prone}. We set the term number of the Chebyshev expansion to $10$, $\mu=0.2$ and $\theta=0.5$.
\item \textbf{NRP}~\cite{yang2020homogeneous}. We set the number of iterations for personalized PageRank approximation to $20$, the number of iterations for learning weights to $10$, and the random walk decay factor to $0.15$. Following Section 5.4 in~\cite{yang2020homogeneous}, for each node, we normalize its forward and backward vectors (of length $d/2$ respectively) and then concatenate them as the final embedding. 
\end{itemize}

The setting of hyperparameters for the proposed approach is as follows.
\begin{itemize}
	\item \textbf{GMF-FE.} We set $b$ to the $70$th percentile of the entries in the FE distance matrix $\mathbf{\Delta}_{\eta}^{\mathrm{FE}}$ so that $70\%$ of entries in the similarity matrix $\mathbf{S}_{\eta,b,\gamma}^{\mathrm{FE}}$ will be positive [cf.~\eqref{e:fe_sim}]. We set $\gamma$ to make the largest entry in $\mathbf{S}_{\eta,b,\gamma}^{\mathrm{FE}}$ equal to $6$. We perform a line search\footnote{{In node classification, we select the value of $\eta$ optimal for most fractions of labeled nodes considered. We do the same for the in-out and return parameters in node2vec. For the proposed GMF-FE, the optimal $\eta$ is relatively stable for different fractions of labeled nodes.}} of $\eta$ over the set $\{10^{-4}, 10^{-3}, \cdots, 10^{1}\}$.
\end{itemize}	

Finally, in solving the optimization problems in~\eqref{e:fe_embed} and~\eqref{E:obj}, we use the Adam optimizer with learning rate $0.1$ and forgetting factors $\beta_1=0.9$ and $\beta_2=0.999$. 
The number of iterations is set to $300$. 

\vspace{1mm}
\noindent\textbf{Tasks}

\noindent To evaluate the quality of the produced node embeddings, we consider three common downstream tasks, namely, node clustering, node classification and link prediction.  

\emph{Node clustering} detects groups of nodes with similar attributes. 
We run $k$-means~\cite{lloyd1982least} on node embeddings to cluster nodes then map the predicted labels to real labels using the Kuhn-Munkres algorithm~\cite{kuhn1955hungarian}. 
To evaluate the clustering performance, we compute four metrics, namely, clustering accuracy score (ACC), normalized mutual information (NMI), adjusted Rand index (ARI) and weighted-F1 score~\cite{emmons2016analysis}. 
For all of them, a larger value indicates a better performance.
For each set of node embeddings, we run $k$-means $10$ times with different centroid initializations. 
Moreover, for each algorithm, we repeat the embedding procedure $5$ times to take into consideration the randomness in embeddings generated by (some of) these algorithms.
For the proposed GMF-FE, the randomness comes from the optimizer.
The results are shown in Fig.~\ref{fig:cluster}, which are averaged over $50$ realizations in total. 
We set the embedding dimension $d$ to $8$ in this task\footnote{We follow the convention in spectral clustering~\cite{von2007tutorial}, in which the embedding dimension is usually set to the number of clusters. Since the baseline method GraRep requires the embedding dimension to be a multiple of the maximum transition step ($4$ in this paper), we set the embedding dimension to $8$ which is a multiple of $4$ while also being similar to the number of clusters in both networks considered ($6$ in CiteSeer and $7$ in Cora; see Table~\ref{table-datasets}).}.

Given the labels of a portion of the nodes, in \emph{node classification} our goal is to infer the labels of the remaining nodes. 
To be specific, we randomly split the nodes (i.e., their embeddings and labels) into a training set and a test set. 
We train a one-vs-rest logistic regression classifier using the training set and evaluate the performance using the test set. 
We adopt micro-F1 scores to evaluate the classification accuracy (similar results were obtained for macro-F1 scores; see Fig.~S1 in the supplementary material).
We repeat this procedure $10$ times for each set of node embeddings and repeat the embedding procedure $5$ times for each algorithm.
The results (averaged over a total of $50$ realizations) are shown in Fig.~\ref{fig:nc}.   
For node classification (as well as link prediction discussed below), we set the embedding dimension $d$ to $128$ which is widely adopted in existing works.

\begin{table*}
\caption{\small AUC scores for link prediction. The four operators in Table~\ref{table-operator} are abbreviated as Avg, Hada, L1, and L2, respectively. The best result for each method-dataset combination is underlined and the best performance for each dataset is bolded. GMF-FE attains this best performance for every dataset considered.}\label{table-lp}
\begin{center}
\begin{tabular}{l|p{0.57cm}p{0.57cm}p{0.57cm}p{0.63cm}|p{0.57cm}p{0.57cm}p{0.57cm}p{0.63cm}|p{0.57cm}p{0.57cm}p{0.57cm}p{0.63cm}|p{0.57cm}p{0.57cm}p{0.57cm}p{0.63cm}} 
 \Xhline{0.8pt}
	\multirow{2}{*}{\diagbox[width=5em]{Method}{Dataset}} & \multicolumn{4}{c}{CiteSeer} \vline& \multicolumn{4}{c}{Cora} \vline& \multicolumn{4}{c}{PPI} \vline&\multicolumn{4}{c}{BlogCatalog} \\\ 
 	                       & Avg & Hada & L1 & L2 & Avg & Hada & L1 & L2 & Avg & Hada & L1 & L2 & Avg & Hada & L1 & L2 \\ [0.2em]
 	\hline
 	node2vec   & 0.640 & \underline{0.942} & 0.929 & 0.931 & 0.596 & \underline{0.908} & 0.896 & 0.898 & 0.810 & \underline{0.817} & 0.680 & 0.680 & 0.885 & 0.819 & 0.920 & \underline{0.921} \\ [0.2em]
 	DeepWalk   & 0.648 & \underline{0.935} & 0.923 & 0.925 & 0.598 & \underline{0.903} & 0.886 & 0.888 & 0.803 & \underline{0.807} & 0.650 & 0.649 & 0.866 & 0.815 & \underline{0.915} & \underline{0.915} \\ [0.2em]
 	GraRep     & 0.681 & 0.897 & \underline{0.914} & 0.905 & 0.649 & 0.871 & \underline{0.885} & 0.879 & 0.859 & \underline{0.882} & 0.829 & 0.843 & 0.939 & \underline{0.950} & 0.941 & 0.947 \\ [0.2em]
 	HOPE       & \underline{0.724} & 0.622 & 0.581 & 0.514 & \underline{0.687} & 0.672 & 0.546 & 0.515 & \underline{0.860} & 0.568 & 0.796 & 0.729 & \underline{0.941} & 0.680 & 0.930 & 0.915 \\ [0.2em]
 	NetMF      & 0.668 & \underline{0.897} & \underline{0.897} & 0.882 & 0.585 & 0.864 & \underline{0.865} & 0.850 & 0.804 & 0.840  & \underline{0.843} & \underline{0.843} & \underline{0.878} & 0.877 & \underline{0.878} & 0.870 \\ [0.2em]
 	InfiniteWalk & 0.635 & \underline{0.910} & 0.908 & 0.902 & 0.571 & \underline{0.911} & 0.893 & 0.896 & 0.824 & \underline{0.861} & 0.823 & 0.824 & \underline{0.929} & 0.873 & 0.764 & 0.748 \\ [0.2em]
 	AROPE & \underline{0.724} & 0.624 & 0.579 & 0.509 & \underline{0.686} & 0.672 & 0.542 & 0.509 & \underline{0.860} & 0.583 & 0.832 & 0.807 & \underline{0.941} & 0.678 & 0.931 & 0.916 \\ [0.2em]
 	ProNE & 0.690 & 0.858 & \underline{0.864} & 0.860 & 0.621 & 0.817 & \underline{0.828} & 0.818 & 0.863 & \underline{0.877} & 0.796 & 0.812 & \underline{0.945} & 0.934 & 0.897 & 0.908 \\ [0.2em]
 	NRP          & 0.674 & 0.891 & \underline{0.897} & 0.894 & 0.618 & \underline{0.871} & 0.864 & 0.858 & 0.865 & 0.876 & 0.872 & \underline{0.877} & \underline{0.949} & 0.942 & 0.946 & 0.947 \\ [0.2em]
 	GMF-FE     & 0.702  & \textbf{\underline{0.951}} & 0.932 & 0.936 & 0.665  &\textbf{\underline{0.924}} & 0.892 & 0.895 & 0.885 & \textbf{\underline{0.912}} & 0.781 & 0.786 & 0.953 & \textbf{\underline{0.957}} & 0.896  & 0.885 \\ [0.2em]
 \hline 
   GMF-DW   & 0.682 & 0.928 & 0.930 & \underline{0.931} & 0.609 & \underline{0.894} & 0.874 & 0.877 & \underline{0.842} & 0.761 & 0.814 & 0.817 & \underline{0.928} & 0.672 & 0.677 & 0.675 \\ [0.2em]
   EIG-FE   & 0.717 & \underline{0.942} & 0.920 & 0.913 & 0.652 & \underline{0.917} & 0.887 & 0.886 & 0.880 & \underline{0.901} & 0.785 & 0.785 & \underline{0.952} & 0.934 & 0.783 & 0.766 \\ 
  \Xhline{0.8pt}
\end{tabular}
\end{center}
\end{table*}

Given a graph with some edges removed, in \emph{link prediction} our goal is to predict the missing edges. 
More precisely, we keep the largest connected component (denoted by $\mathcal{G}'$) of the graph generated by removing $30\%$ of the edges from the original graph $\mathcal{G}$. 
The node embeddings are learned from $\mathcal{G}'$. 
Denote by $\mathbf{e}_{ij}$ the embedding of a pair of nodes $(v_i,v_j)$, which is computed in four ways \cite{node2vec} as listed in Table \ref{table-operator}. 
The edges in $\mathcal{G}'$ are treated as positive examples in the training set, and the removed edges whose two endpoints belong to $\mathcal{G}'$ are treated as positive examples in the test set. Denote by $\mathcal{G}''$ the induced subgraph of $\mathcal{G}$ which contains the nodes in the node set of $\mathcal{G}'$ and the edges between them. 
To generate negative examples, we randomly sample node pairs from $\mathcal{G}''$ which have no edge between them. 
The number of negative examples is identical to that of positive examples for both the training and test sets. 
We train a logistic regression classifier using the training set and evaluate the performance in terms of Area Under Curve (AUC) scores in the test set. 
The results (averaged over $10$ realizations) are shown in Table~\ref{table-lp}.

\subsection{Performance comparison}
\label{Ss:performance_comparison}

Fig.~\ref{fig:cluster} reveals that, for node clustering, the proposed GMF-FE performs better than all the baseline methods for both of the datasets considered.
For these two datasets, each node is associated with a single label and we use these labels as our ground-truth classes.
When contrasted with the second-best alternative (node2vec for CiteSeer and InfiniteWalk for Cora), our proposed method yields an increase of $0.036$ in clustering accuracy ($0.649$ compared to $0.613$) for CiteSeer and an increase of $0.021$ ($0.701$ compared to $0.680$) for Cora.
Similar improvements can be seen on the other three performance metrics considered and the improvements are more substantial when compared with the other baselines.

For node classification, we can see from Fig.~\ref{fig:nc} that GMF-FE markedly improves over all the baseline methods for three of the considered networks and, for CiteSeer, it is comparable with the top baselines. 
For example, when half of the nodes are labeled, GMF-FE yields a micro-F1 score of $0.851$, $0.220$ and $0.423$ for Cora, PPI and BlogCatalog, which improves upon the performance of the best competing baseline for each dataset, i.e., $0.838$, $0.205$ and $0.413$, respectively.
Moreover, notice that GMF-FE's superior performance is robust to the fraction of labeled nodes.

Table~\ref{table-lp} shows that, in link prediction, GMF-FE when used with the Hadamard operator achieves the best performance across all networks with marked differences over the second-best alternative for some of the datasets.
Importantly, for PPI, GMF-FE yields a reduction in the error rate of $25\%$ when compared with the best baseline by pushing the accuracy from $0.882$ for GraRep to $0.912$.
Briefly revisiting Fig.~\ref{fig:nc} it can be appreciated that also for node classification the largest relative performance improvements were attained for the PPI network.

In summary, we have established that the proposed method provides superior performance over all baseline methods for the three downstream tasks considered across various networks.  

\begin{figure}
\centering
\includegraphics[scale=0.27]{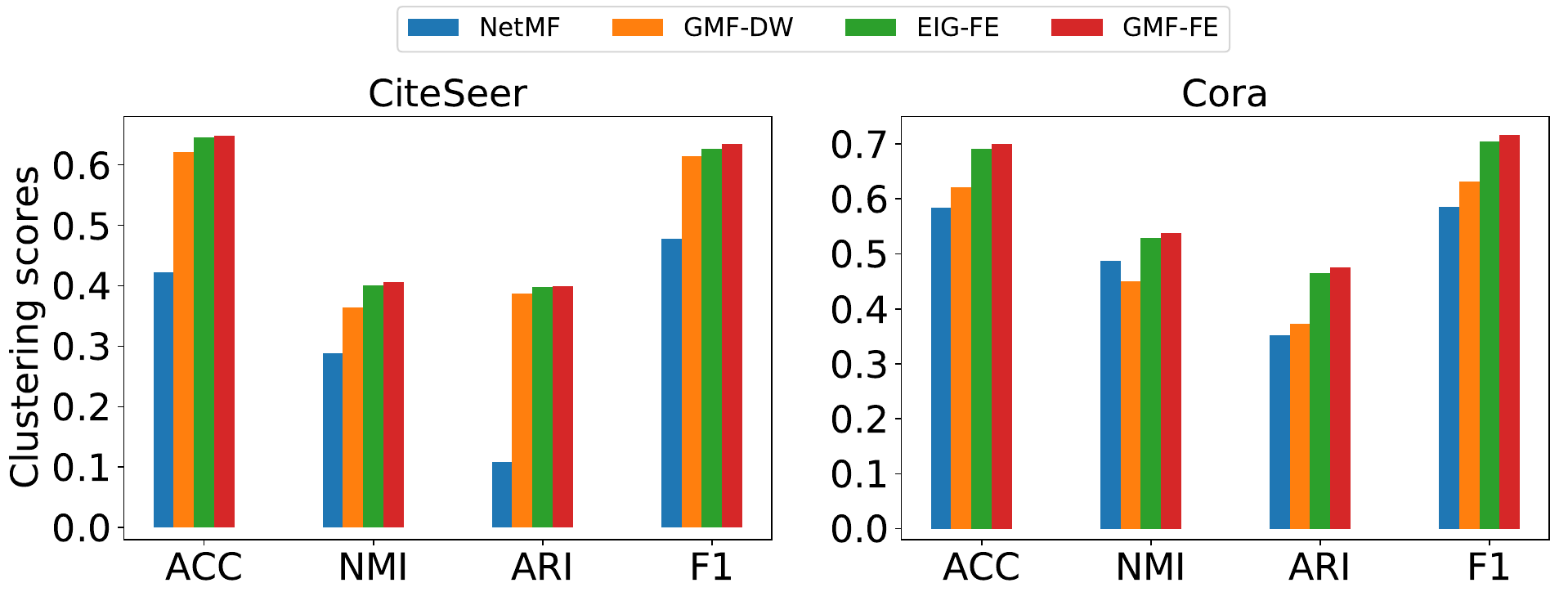}
\caption{\small Node clustering results for the ablation study.}
\label{fig:ablation_cluster}
\end{figure}

\begin{figure*}[!htbp]
\centering
\includegraphics[scale=0.42]{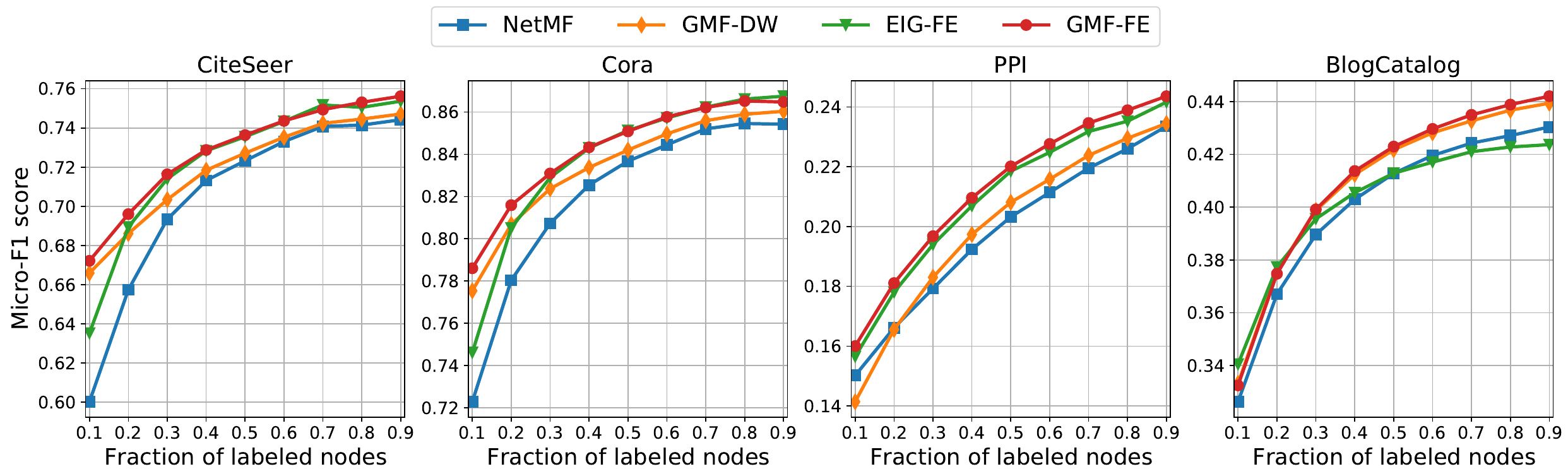}
\caption{\small Node classification results for the ablation study.}
\label{fig:ablation_nc}
\end{figure*}

\begin{figure*}
\centering
\begin{subfigure}[t]{0.24\textwidth}
\centering
\includegraphics[scale=0.2]{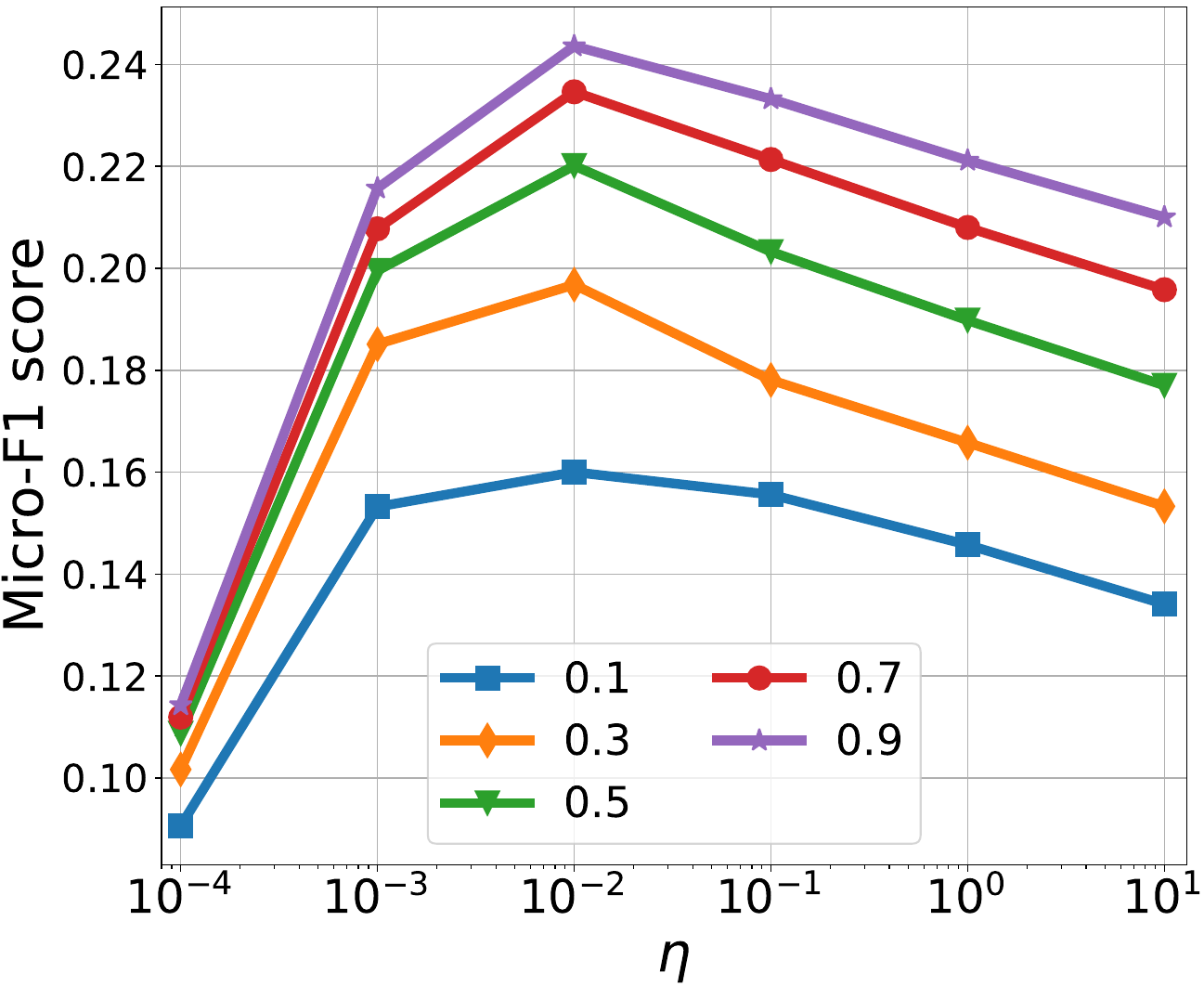}
\caption{}
\end{subfigure}
\begin{subfigure}[t]{0.24\textwidth}
\centering
\includegraphics[scale=0.2]{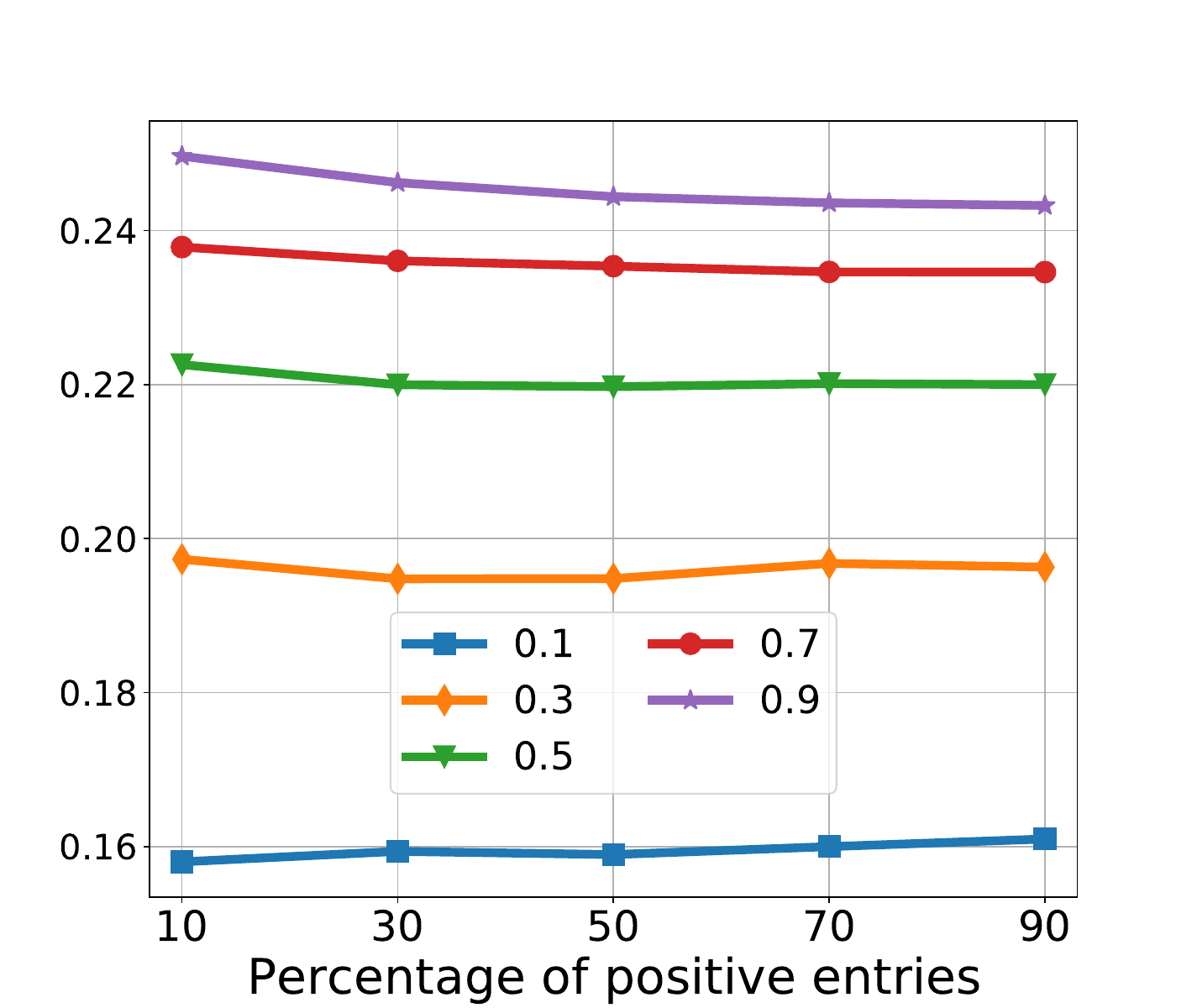}
\caption{}
\end{subfigure}
\begin{subfigure}[t]{0.24\textwidth}
\centering
\includegraphics[scale=0.2]{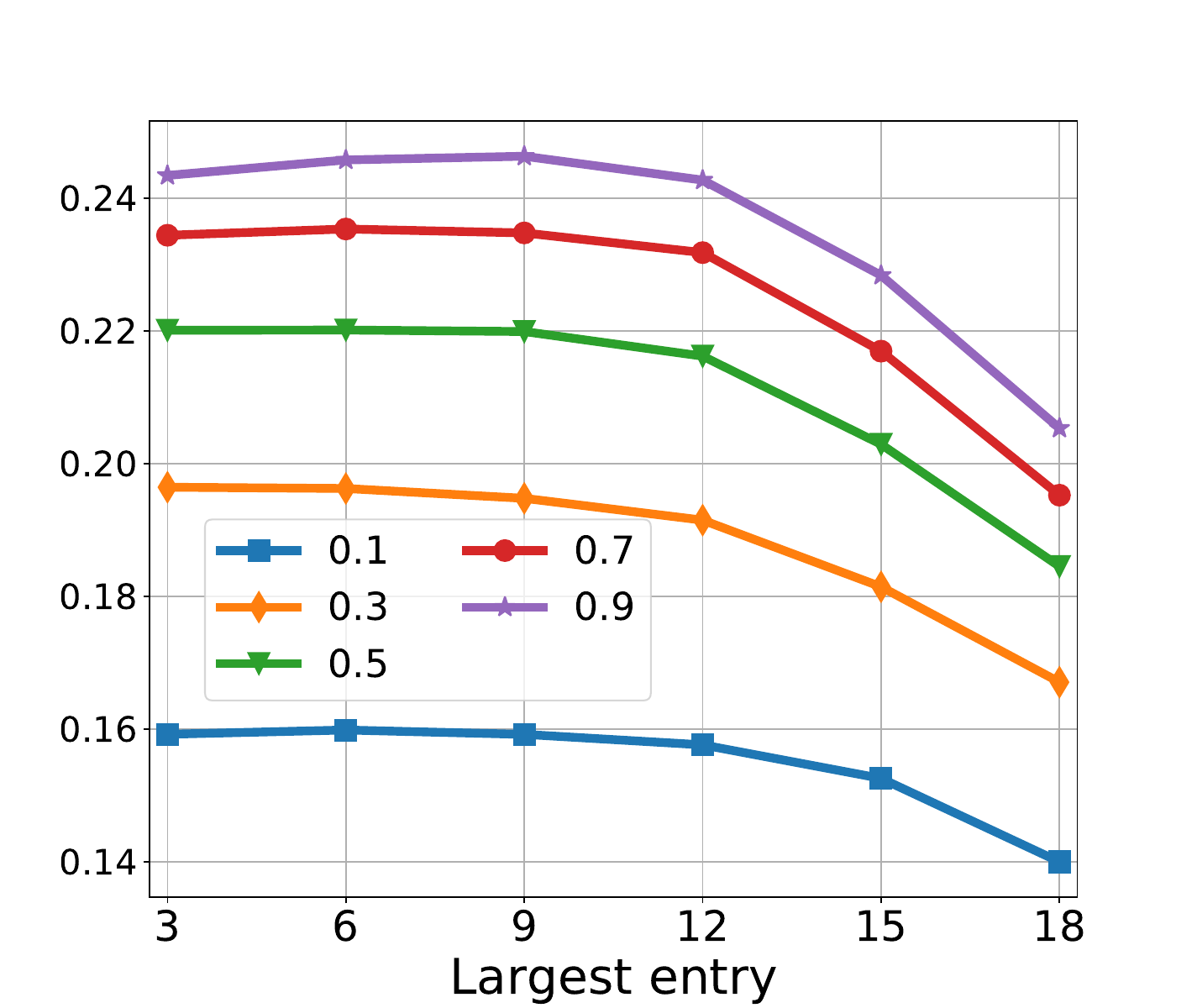}
\caption{}
\end{subfigure}
\begin{subfigure}[t]{0.24\textwidth}
\centering
\includegraphics[scale=0.2]{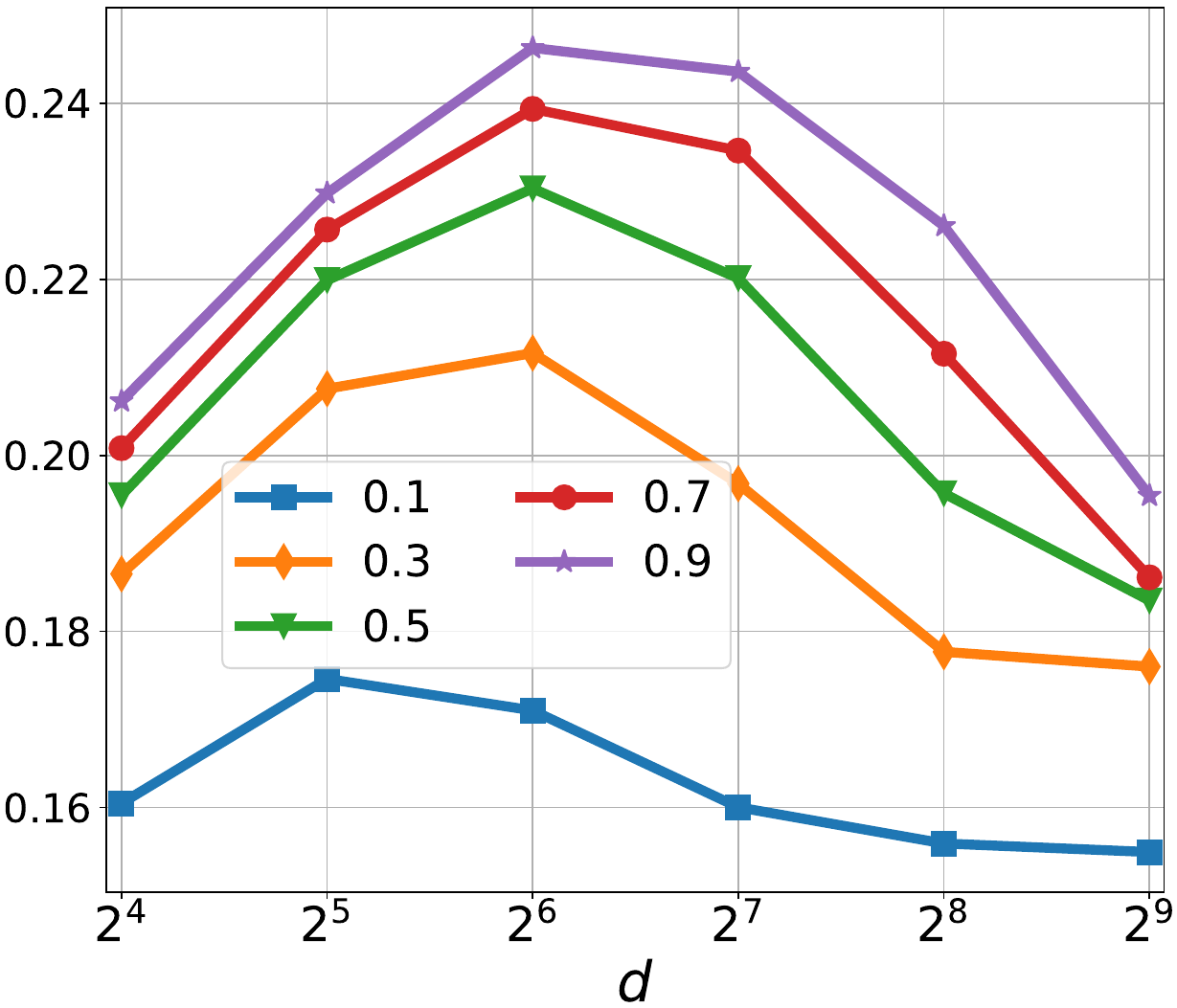}
\caption{}
\end{subfigure}
\begin{subfigure}[t]{0.24\textwidth}
\centering
\includegraphics[scale=0.2]{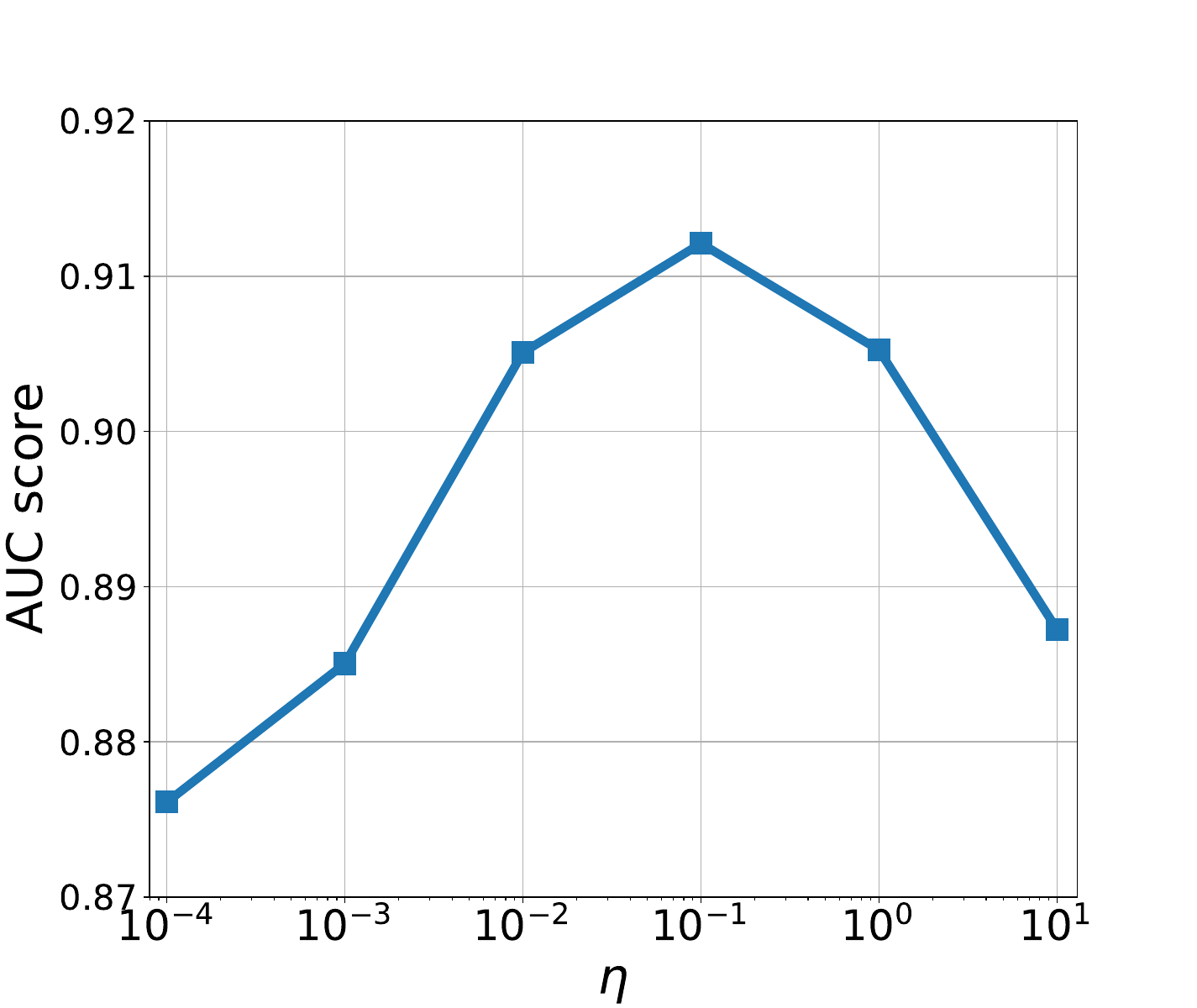}
\caption{}
\end{subfigure}
\begin{subfigure}[t]{0.24\textwidth}
\centering
\includegraphics[scale=0.2]{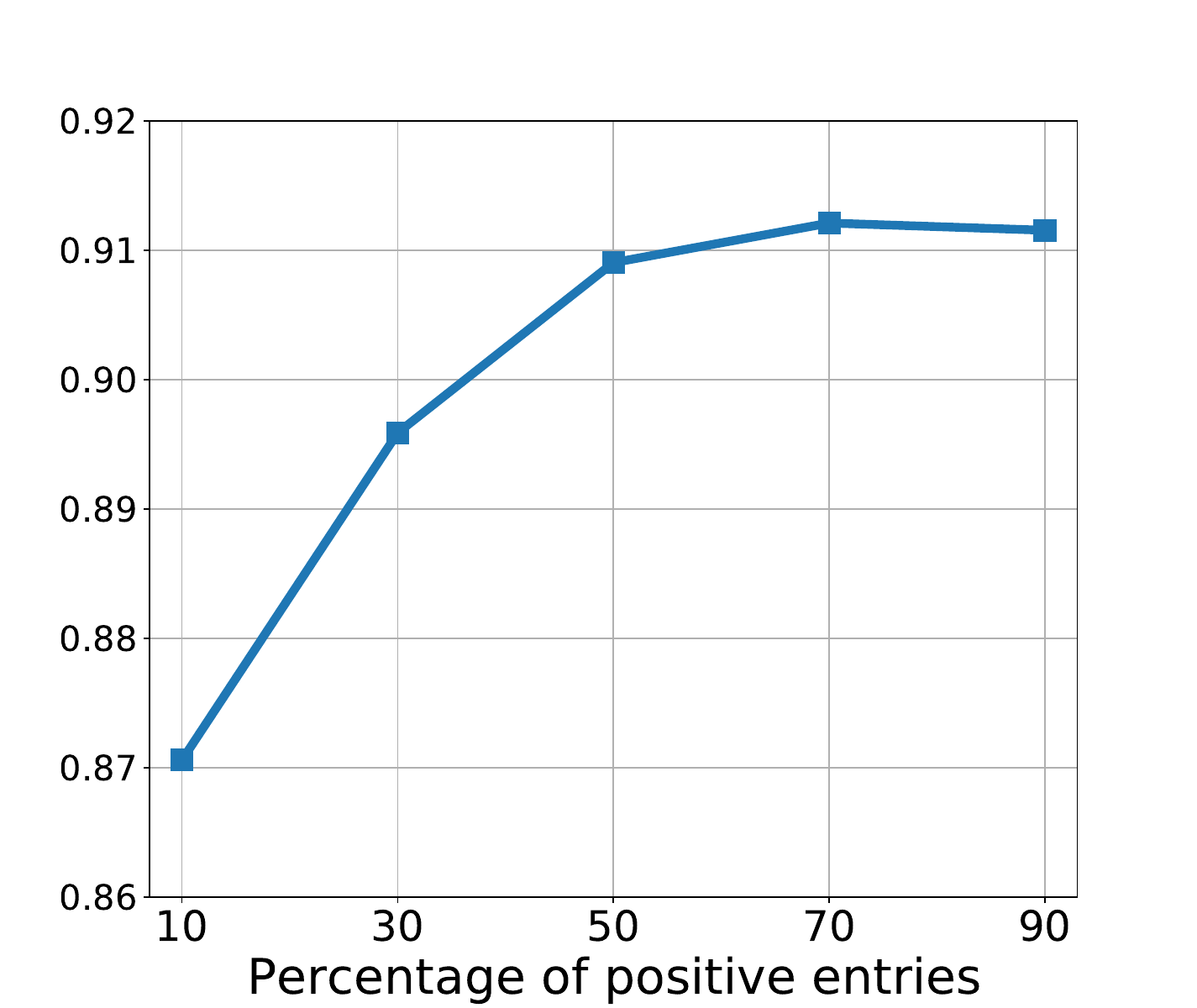}
\caption{}
\end{subfigure}
\begin{subfigure}[t]{0.24\textwidth}
\centering
\includegraphics[scale=0.2]{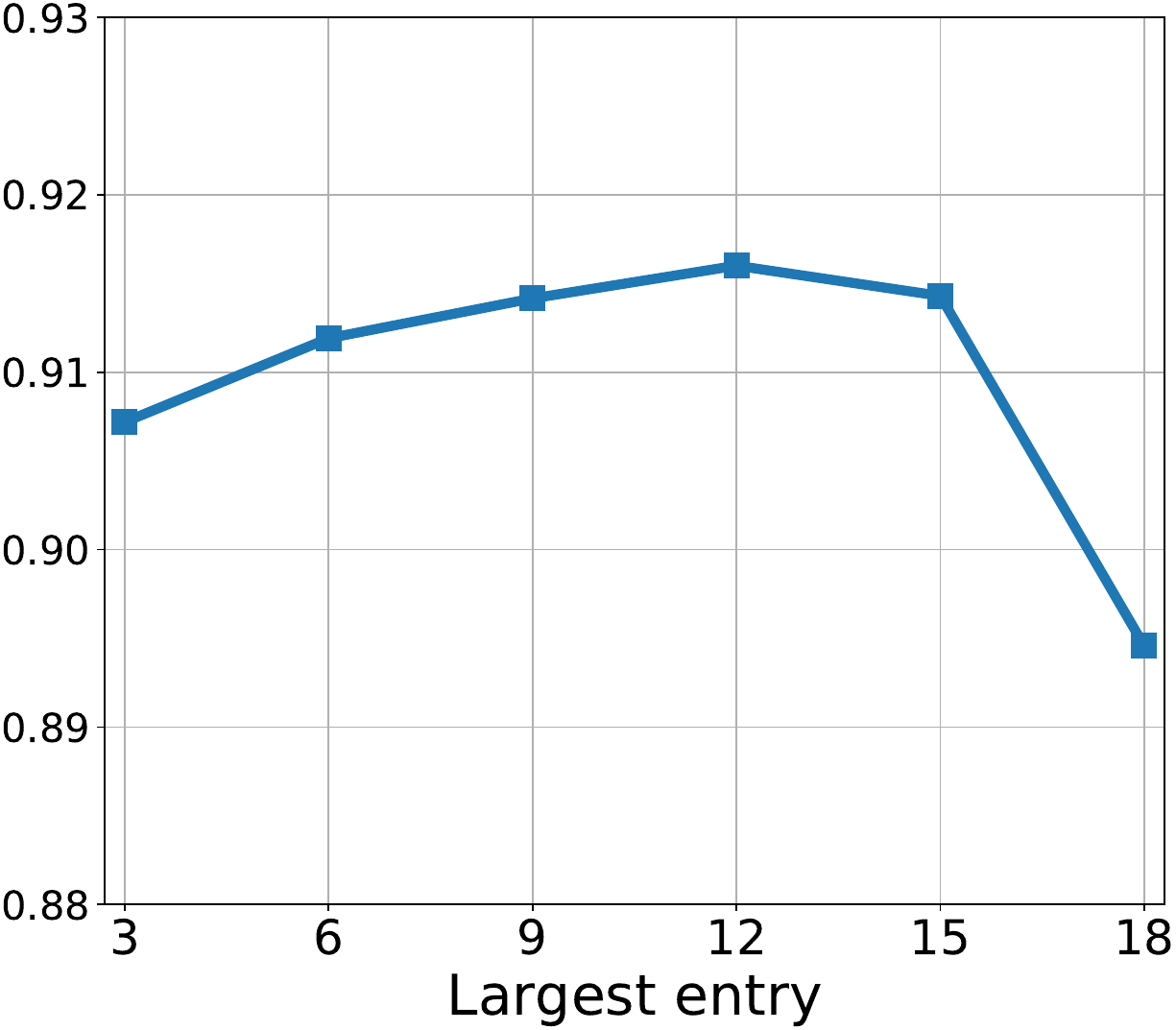}
\caption{}
\end{subfigure}
\begin{subfigure}[t]{0.24\textwidth}
\centering
\includegraphics[scale=0.2]{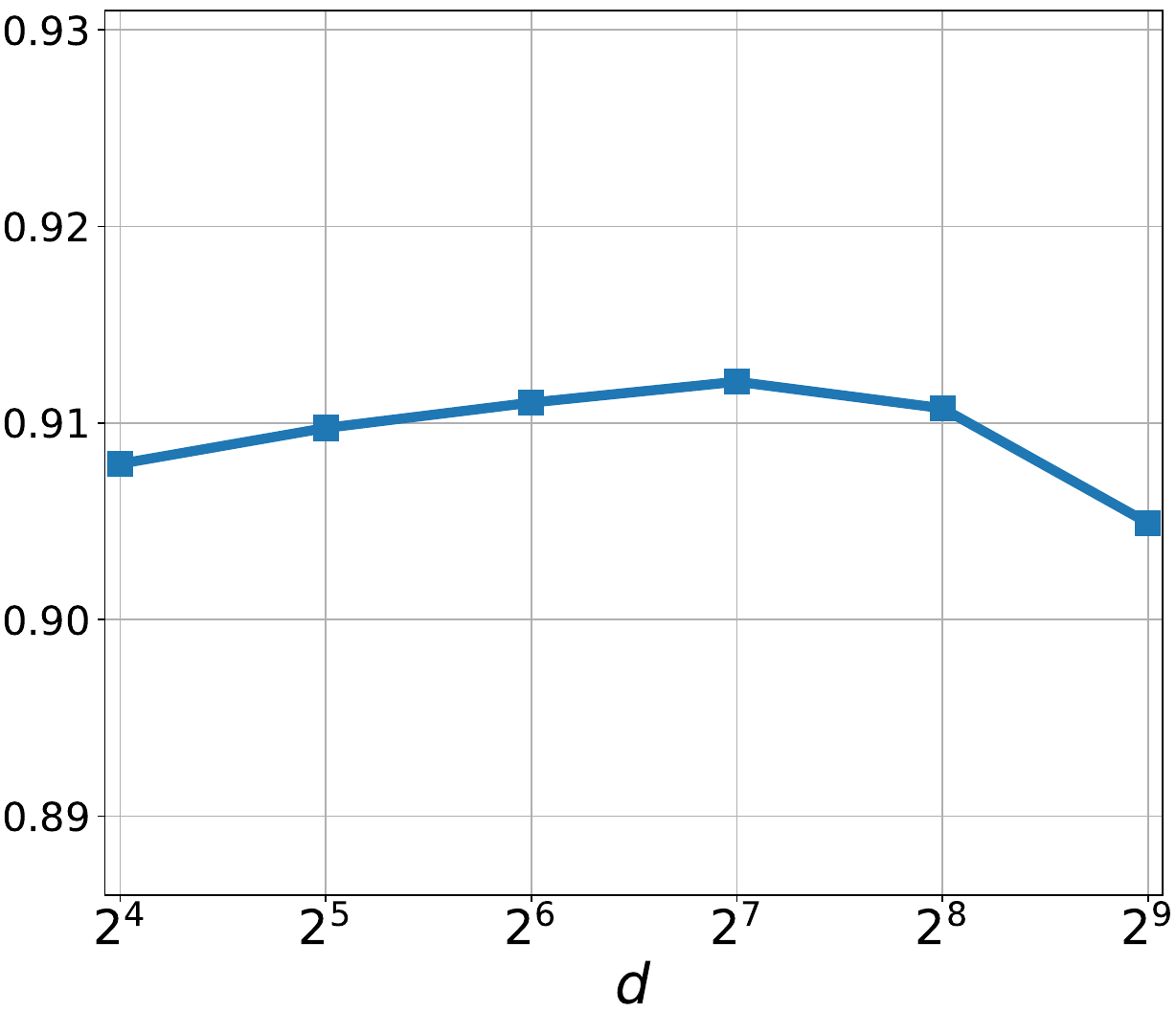}
\caption{}
\end{subfigure}
\caption{\small Parameter sensitivity test on the PPI network. (a)-(d) are for node classification where the legend denotes the fraction of labeled nodes. (e)-(h) are for link prediction where the Hadamard operator is applied.}
\label{fig:ppi-para}
\end{figure*}

\subsection{Ablation study}
\label{Ss:ablation_study}
We compare the following two approaches with NetMF and GMF-FE to determine the impact of each step in GMF-FE.
\begin{itemize}
\item \textbf{GMF-DW}.
	We adopt GMF to factorize the similarity matrix $\mathbf{S}^{\mathrm{DW}}$ derived in NetMF (see~\eqref{E:dw}).
	Here we consider a scaled version $\alpha\mathbf{S}^{\mathrm{DW}}$.
	Note that the scaling does not change the SVD results but influences the results of GMF.
	More precisely, in (4) we set $S^{+}_{ij}$ to the $\alpha$th power of $\exp(S^{\mathrm{DW}}_{ij})$ and set $S^{-}_{ij}=1$ for all pairs $(i,j)$.
	We perform a line search of $\alpha$ over the set $\{0.2, 0.4, \cdots, 1.8, 2\}$.
\item \textbf{EIG-FE}.
	We use the eigendecomposition to factorize the FE distance-based similarity matrix~\eqref{e:fe_sim}.
	The parameter setting is the same as that of GMF-FE.
	We use the top eigenvectors as node embeddings. 
\end{itemize}
The numerical results for node clustering, node classification, and link prediction are respectively given in Figure~\ref{fig:ablation_cluster}, Figure~\ref{fig:ablation_nc}, and Table~\ref{table-lp} (the bottom two lines).
It can be observed that, for these three downstream tasks and four datasets, the performance of GMF-DW and EIG-FE generally lie between the performance of NetMF and GMF-FE (except for node clustering NMI score on Cora and node classification on BlogCatalog). 
The improvement of GMF-DW over NetMF (as well as GMF-FE over EIG-FE) indicates the benefit of considering our matrix decomposition technique whereas the improvement of EIG-FE over NetMF (as well as GMF-FE over GMF-DW) indicates the benefit of considering the free energy distance. 
The superior performance of GMF-FE demonstrates the benefit of considering both modifications together. 
In addition, we can also see that the use of the FE distance brings a bigger performance increase than the proposed matrix factorization method in most cases (except for node classification on BlogCatalog).

\subsection{Parameter sensitivity}
\label{Ss:parameter_sensitivity}

We evaluate how different choices of parameters in the proposed GMF-FE method influence its performance.  
We present the results on the PPI network in Fig.~\ref{fig:ppi-para}, where the two rows respectively correspond to the tasks of node classification and link prediction.
Except for the parameter being examined, all other parameters are set to default values (as detailed when introducing the baseline methods). 
The default value for $\eta$ is selected as the optimal one over the set $\{10^{-4}, 10^{-3},\cdots,10^1\}$, i.e., $10^{-2}$ for node classification and $10^{-1}$ for link prediction.
We have also conducted examinations on the other three networks as well as the node clustering task, and similar observations are obtained; see Figs.~S2,~S3 and~S4 in the supplementary material.

We analyze the effect of $\eta$ in Figs.~\ref{fig:ppi-para}(a) and~\ref{fig:ppi-para}(e), where it becomes evident that the optimal performance is achieved for intermediate values of the parameter. 
This aligns well with our intuition that the optimal representation is not achieved at the extremal SP and CT distances but rather for an intermediate FE distance. 
It should also be noted that the optimal value of $\eta$ is robust to the fraction of labeled nodes [cf. Fig.~\ref{fig:ppi-para}(a)], which is a desirable feature in practice.

Figs.~\ref{fig:ppi-para}(b) and~\ref{fig:ppi-para}(f) show the effect of the shift parameter $b$, which controls the percentage of positive entries in the similarity matrix, whereas
Figs.~\ref{fig:ppi-para}(c) and~\ref{fig:ppi-para}(g) present the influence of the scale parameter $\gamma$ [cf. \eqref{e:fe_sim}].
Combining the results in Fig.~\ref{fig:ppi-para} with those for other networks shown in the supplementary material, we conclude that the proposed GMF-FE performs well and is stable when the percentage of positive entries and the largest entry are in the range $50\sim 90$ and $3\sim 9$, respectively.
Intuitively, larger values of $\gamma$ might be undesirable from a numerical standpoint due to the exponential function operator in \eqref{e:fe_embed}.

Figs.~\ref{fig:ppi-para}(d) and~\ref{fig:ppi-para}(h) show the effect of increasing the embedding dimension $d$. 
It can be observed that, in node classification, the optimal $d$ is dependent on the fraction of labeled nodes. 
This also aligns well with intuition, where a simpler model (smaller $d$) is preferred when limited data are available (smaller fraction of labeled nodes) but a more complex model can be afforded for larger fraction of labeled nodes.
Finally, it can be seen that the performance of link prediction (as well as node clustering shown in Figs.~S2(l) and S3(l) in the supplementary material) is only minimally affected by the choice of $d$ within the range considered. 

\begin{figure}
\centering
\begin{subfigure}[t]{0.24\textwidth}
\centering
\includegraphics[width=4.4cm, height=3.7cm]{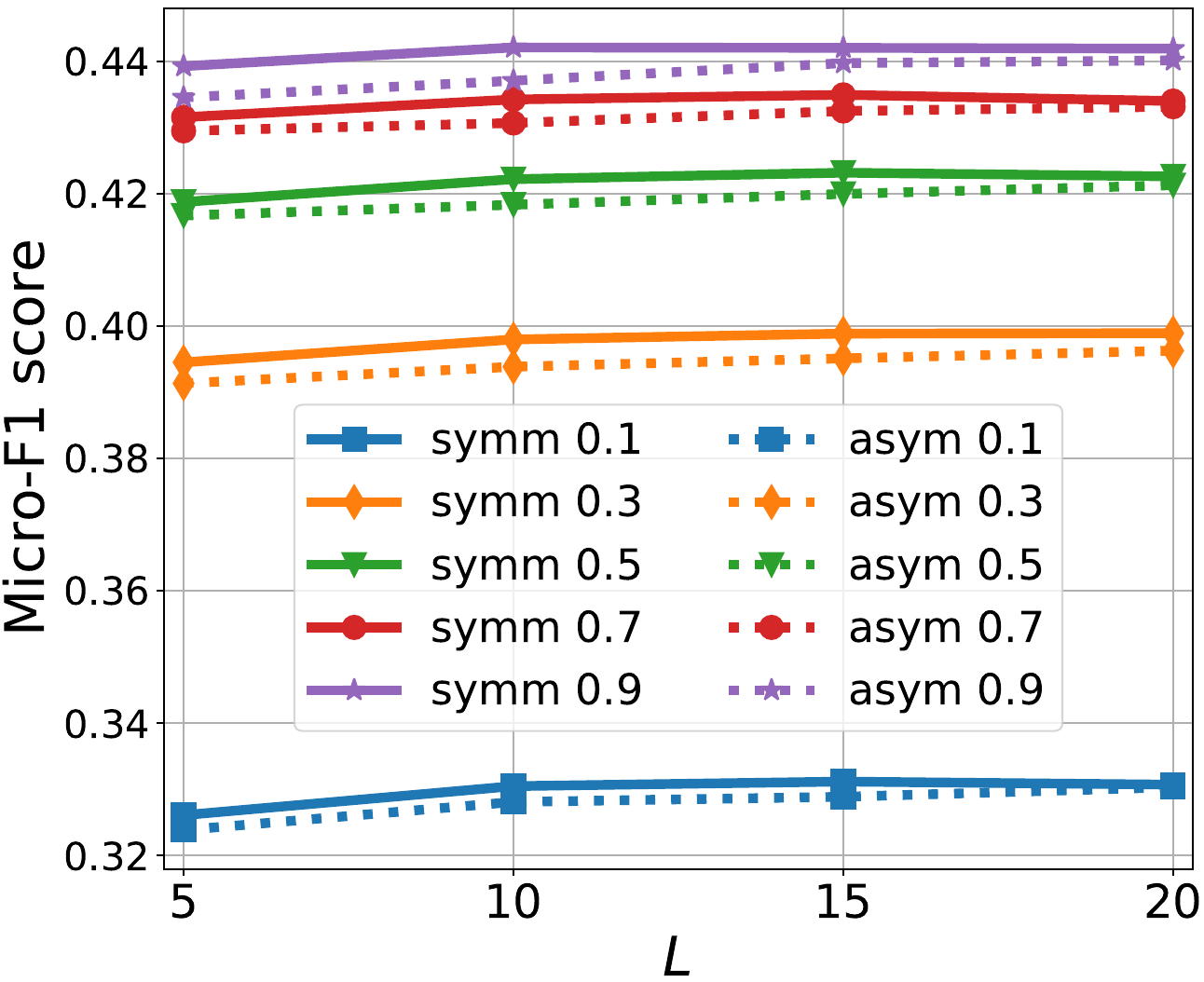}
\caption{}
\end{subfigure}%
\begin{subfigure}[t]{0.24\textwidth}
\centering
\includegraphics[width=4.2cm, height=3.7cm]{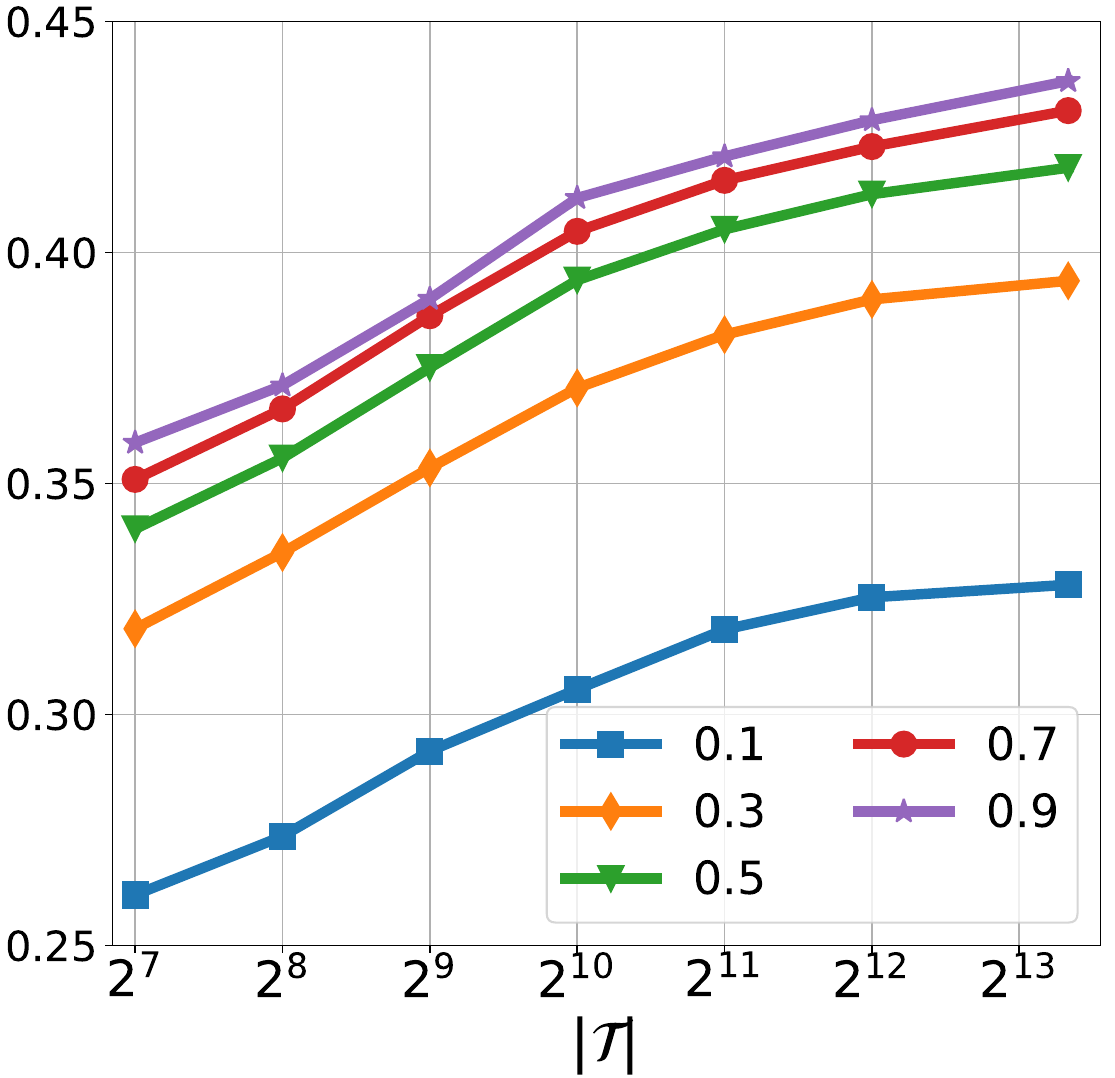}
\caption{}
\end{subfigure}
\caption{\small The influence of (a) the number of iterative steps $L$ and (b) the number of sampled target nodes $|\mathcal{T}|$ on the performance of node classification for the BlogCatalog network. The legend denotes the fraction of labeled nodes. Good performance can be achieved even for small values of $L$ while the sensitivity to $|\mathcal{T}|$ is larger.
}
\label{fig:fast-nc}
\end{figure}

\section{Scalability}\label{s:scalability}

In this section, we discuss the scalability of the proposed method.
More precisely, we state a scalable way of computing the FE distance and analyze its complexity in Section~\ref{ss:complexity}.
Then, we validate its effectiveness on the BlogCatalog dataset as well as another larger dataset called CoCit in Section~\ref{ss:large}.

\subsection{Computational complexity analysis}\label{ss:complexity}

As mentioned in Section~\ref{ss:fe}, the FE distance matrix $\pmb{\Delta}_{\eta}^{\mathrm{FE}}$ can be computed in closed form using the method proposed in~\cite{fe_dist}. 
This method involves a matrix inversion of one $n\times n$ matrix, thus its computation complexity is $\mathcal{O}(n^3)$ if we adopt a naive Gauss-Jordan elimination method.

A scalable way of computing the FE distance on large and sparse graphs is proposed in~\cite{fe_fast}. 
For notational simplicity, denote the \emph{directed free energy dissimilarity} from node $s$ to node $t$ by $\phi^{\mathrm{FE}}_{st}$ (which corresponds to $\phi(\mathbb{P}^{\mathrm{FE}}_{st})$ in Section~\ref{ss:fe}).
Then, its value can be computed via an iterative process as 
\begin{align}\label{E:iter}
& \phi^{\mathrm{FE}}_{st}(\tau+1) =  \\
& \quad \begin{cases}
-\frac{1}{\eta} \log \left[ \sum\limits_{i\in\mathcal{N}_s} P_{si} \exp \left[ -\eta\left(C_{si}+\phi^{\mathrm{FE}}_{it}(\tau)\right) \right] \right] & \text{if } s \neq t, \\
0 & \text{if } s=t,
\end{cases} \nonumber
\end{align}
where $\mathcal{N}_s$ is the set of (out-)neighbors of node $s$, and $\tau$ is the iteration index.  
It can be understood that $\phi^{\mathrm{FE}}_{st}(\tau)$ contains the directed FE dissimilarity from $s$ to $t$ when considering paths up to length $\tau$.
Starting from the initialization at $\tau=0$
\begin{align}
\phi^{\mathrm{FE}}_{st}(0) = 
\begin{cases}
\infty & \text{if } s \neq t,\\
0 & \text{if } s=t,
\end{cases}
\end{align}
$\phi^{\mathrm{FE}}_{st}(\tau)$ will converge to $\phi^{\mathrm{FE}}_{st}$ as $\tau$ increases.
We have observed from experiments that the iterative process converges faster for a larger $\eta$.
This is consistent with the fact that the FE distance converges to the SP distance as $\eta\to\infty$, thus shorter paths (corresponding to smaller $\tau$) are sufficient for larger $\eta$.
Denote the matrix collecting the directed FE dissimilarities between all pairs of nodes by $\mathbf{\Phi}^{\mathrm{FE}}_{\eta}$, and its relation with the FE distance matrix is $\pmb{\Delta}^{\mathrm{FE}}_{\eta} = (\mathbf{\Phi}^{\mathrm{FE}}_{\eta} + (\mathbf{\Phi}^{\mathrm{FE}}_{\eta})^{\top}) / 2$.

Inspired by~\cite{fe_fast}, we adopt two techniques to scale the computation of~\eqref{E:iter}.
First, for a fixed pair $(s,t)$, we set $x_i=C_{si}+\phi^{\mathrm{FE}}_{it}(\tau)$ for simplicity and pre-compute $x^{\ast}=\min_{i\in\mathcal{N}_s} x_i$.
Then, the expression when $s\neq t$ in~\eqref{E:iter} can be rewritten as
\begin{equation}\label{E:lse}
x^{\ast} - \frac{1}{\eta} \log \left[ \sum\limits_{i\in\mathcal{N}_s} P_{si} \exp \left[ -\eta\left( x_i - x^{\ast} \right) \right] \right]
\end{equation}
according to the log-sum-exp trick, which can help avoid numerical underflow problems.
Moreover, the terms in the summation for which $\eta(x_i - x^{\ast})$ exceeds a certain threshold value (which we set to $7$ in our experiments) can be ignored to reduce the number of terms to be computed. 
The second adopted technique is to consider only paths bounded by a given length $L \ll n$. 
Putting it differently, we terminate the iterative process in~\eqref{E:iter} when $\tau$ reaches $L$ instead of repeating the iteration until convergence.
We denote the approximated FE distance and dissimilarity matrices by $\pmb{\Delta}^{\mathrm{FE}}_{\eta, L}$ and $\mathbf{\Phi}^{\mathrm{FE}}_{\eta, L}$, respectively.


\begin{table}
\caption{\small The influence of (a) the number of iterative steps $L$ and (b) the number of sampled target nodes $|\mathcal{T}|$ on the performance of link prediction for the BlogCatalog network. Good performance can be achieved even for small values of $L$ and small sensitivity to $|\mathcal{T}|$ can be observed for the range considered.}
\label{table:fast-lp}
\begin{subtable}{0.22\textwidth}
\caption{} 
\begin{tabular}{l c|p{0.4cm}p{0.4cm}p{0.4cm}p{0.4cm}} 
 \Xhline{0.8pt}
	& $L$ & Avg & Hada & L1 & L2  \\
 \hline
  \multirow{3}{0.01cm}{\rotatebox{90}{symm}}  & 5   &  0.952 & \underline{0.954} & 0.822 & 0.827   \\
    & 10  &  0.953 & \underline{0.958} & 0.809 & 0.811   \\
    & 15  &  0.953 & \underline{0.960} & 0.810 & 0.809   \\
 \hline
   \multirow{3}{0.01cm}{\rotatebox{90}{asym}} & 5   &  \underline{0.924} & 0.793 & 0.799 & 0.801    \\
    & 10  &  \underline{0.912} & 0.785 & 0.811 & 0.814    \\
    & 15  &  \underline{0.913} & 0.775 & 0.828 & 0.830     \\
 \Xhline{0.8pt}
\end{tabular}
\end{subtable}
\hspace{2em}
\begin{subtable}{0.22\textwidth}
\caption{} 
\begin{tabular}{p{0.4cm}|p{0.4cm}p{0.4cm}p{0.4cm}p{0.4cm}} 
 \Xhline{0.8pt}
  $|\mathcal{T}|$	& Avg & Hada & L1 & L2 \\
 \hline
$2^7$ &    \underline{0.895} & 0.734 & 0.617 & 0.616  \\
$2^8$ &    \underline{0.897} & 0.746 & 0.629 & 0.629  \\
$2^9$ &    \underline{0.908} & 0.748 & 0.650 & 0.651  \\
$2^{10}$ & \underline{0.908} & 0.773 & 0.672 & 0.673 \\
$2^{11}$ & \underline{0.909} & 0.785 & 0.702 & 0.702 \\
$2^{12}$ & \underline{0.913} & 0.797 & 0.743 & 0.744 \\
 \Xhline{0.8pt}
\end{tabular}
\end{subtable}
\end{table}

Considering a single iteration, it can be observed from~\eqref{E:lse} that the computational complexity of updating $\phi^{\mathrm{FE}}_{st}(\tau)$ for one pair of nodes $(s,t)$ is in the order of $|\mathcal{N}_s|$.
This implies that for a fixed $t$ and all possible $s$, the computational complexity is $\sum_{s\in\mathcal{V}}|\mathcal{N}_s|$, thus for all possible pairs $(s,t)$ we get a complexity of $n\times \sum_{s\in\mathcal{V}}|\mathcal{N}_s|=2n|\mathcal{E}|$.
The total computational complexity for $L$ iterations is $2Ln|\mathcal{E}|$.
Moreover, from~\eqref{E:iter} we can see that the computation of $\phi^{\mathrm{FE}}_{st}(\tau+1)$ only requires $\phi^{\mathrm{FE}}_{it}(\tau)$ for $i\in\mathcal{N}_s$, hence $\phi^{\mathrm{FE}}_{st}(\tau)$ for every $t\in\mathcal{V}$ (namely columns in $\mathbf{\Phi}^{\mathrm{FE}}_{\eta, L}$) can be computed in parallel.

To further reduce the computational complexity, instead of considering the directed free energy dissimilarities from every node to all other nodes, we propose to consider a subset of nodes $\mathcal{T}\subseteq \mathcal{V}$ as the target nodes.   
Denote by $\mathbf{\Phi}^{\mathrm{FE}}_{\eta, L, \mathcal{T}}$ the matrix composed of columns of $\mathbf{\Phi}^{\mathrm{FE}}_{\eta, L}$ indexed by $\mathcal{T}$.
Computing $\mathbf{\Phi}^{\mathrm{FE}}_{\eta, L, \mathcal{T}}$ instead of the whole matrix reduces the computation complexity by a factor of $n/|\mathcal{T}|$.
We can sample $\mathcal{T}$ from $\mathcal{V}$ randomly or select the nodes with higher degrees.

We can convert any of these dissimilarity matrices $\pmb{\Delta}^{\mathrm{FE}}_{\eta, L}$, $\mathbf{\Phi}^{\mathrm{FE}}_{\eta, L}$ and $\mathbf{\Phi}^{\mathrm{FE}}_{\eta, L, \mathcal{T}}$ to a similarity matrix using \eqref{e:fe_sim}, then adopt \eqref{e:fe_embed} or \eqref{E:obj} to factorize the similarity matrix and obtain the node embeddings. 
The similarity matrix computed from $\pmb{\Delta}^{\mathrm{FE}}_{\eta, L}$ is symmetric, thus we can directly adopt \eqref{e:fe_embed}.
The similarity matrices computed from $\mathbf{\Phi}^{\mathrm{FE}}_{\eta, L}$ and $\mathbf{\Phi}^{\mathrm{FE}}_{\eta, L, \mathcal{T}}$ are, respectively, asymmetric and non-square. 
However, we can still use \eqref{E:obj} or a modified version\footnote{Note that Proposition~\ref{P:main_result} is also valid for a non-square matrix $\mathbf{S}$, in which case, we only need to adjust the dimensions of $\mathbf{U}$ and $\mathbf{V}$.} to factorize them and use the computed $\mathbf{U}$ as the node embeddings.

\begin{figure}
\centering
\includegraphics[scale=0.5]{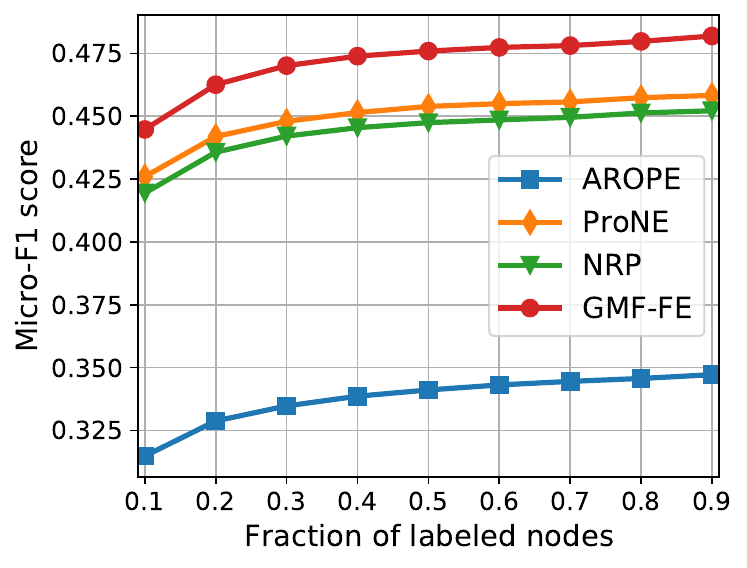}
\caption{\small Node classification results on the CoCit dataset. For GMF-FE, $\mathbf{\Phi}^{\mathrm{FE}}_{\eta, L, \mathcal{T}}$ is used to compute node embeddings where $\mathcal{T}$ consists of nodes with degree no smaller than $20$ and $L=10$.}
\label{fig:nc_cocit}
\end{figure}


\subsection{Numerical results}\label{ss:large}

We first consider the BlogCatalog dataset (and set $\eta$ to $10^{-1}$).
Fig.~\ref{fig:fast-nc}(a) and Table~\ref{table:fast-lp}(a) show the influence of the number of iterative steps $L$ on performance.
The schemes adopting $\pmb{\Delta}^{\mathrm{FE}}_{\eta, L}$ and $\mathbf{\Phi}^{\mathrm{FE}}_{\eta, L}$ are denoted by `symm' and `asym', respectively.
It can be observed from Fig.~\ref{fig:fast-nc}(a) that, in node classification, adopting $\pmb{\Delta}^{\mathrm{FE}}_{\eta, L}$ yields a slightly better performance than adopting $\mathbf{\Phi}^{\mathrm{FE}}_{\eta, L}$, while both of them achieve comparable performance to adopting the exact FE distance matrix $\pmb{\Delta}^{\mathrm{FE}}_{\eta}$ [cf. Fig.~\ref{fig:nc}] for a very small $L$.
For example, when half of the nodes are labeled, adopting $\pmb{\Delta}^{\mathrm{FE}}_{\eta, L}$ and $\mathbf{\Phi}^{\mathrm{FE}}_{\eta, L}$ with $L=10$ yield a micro-F1 score of $0.422$ and $0.418$ respectively, which are comparable with the performance of adopting $\pmb{\Delta}^{\mathrm{FE}}_{\eta}$, i.e., $0.423$.
We can see from Table~\ref{table:fast-lp}(a) that, in link prediction, using $\pmb{\Delta}^{\mathrm{FE}}_{\eta, L}$ yields about the same performance as adopting $\pmb{\Delta}^{\mathrm{FE}}_{\eta}$ [cf. Table~\ref{table-lp}]. 
For instance, adopting $\pmb{\Delta}^{\mathrm{FE}}_{\eta, L}$ ($L=10$) and $\pmb{\Delta}^{\mathrm{FE}}_{\eta}$, both with the Hadamard operator, respectively achieve a AUC score of $0.958$ and $0.957$.
When adopting the asymmetric $\mathbf{\Phi}^{\mathrm{FE}}_{\eta, L}$, the Average operator becomes the optimal choice.
Although performance degrades, it is still comparable with some baselines like DeepWalk.
Overall, for both tasks, even small values of $L$ can result in good performance, significantly reducing the computation time.
 
Fig.~\ref{fig:fast-nc}(b) and Table~\ref{table:fast-lp}(b) show the effect of the number of sampled target nodes $|\mathcal{T}|$ on performance when basing our node embeddings on the matrix $\mathbf{\Phi}^{\mathrm{FE}}_{\eta, L, \mathcal{T}}$, where we set $L=10$.
Here we adopt a random sampling strategy.
For each curve in Fig.~\ref{fig:fast-nc}(b), the last point corresponds to the case $\mathcal{T}=\mathcal{V}$, namely we use the whole matrix $\mathbf{\Phi}^{\mathrm{FE}}_{\eta, L}$.
We can see that, in node classification, accuracy increases as $|\mathcal{T}|$ increases and tends to saturate once $|\mathcal{T}|$ reaches $2^{12}$, especially for low fractions of labeled nodes. 
This implies that we can achieve performance similar to that using $\mathbf{\Phi}^{\mathrm{FE}}_{\eta, L}$ while reducing the computation time by a half. 
Moreover, Fig.~\ref{fig:fast-nc}(b) reveals an interesting trade-off between performance and computation. 
For example, for a fraction of labeled nodes of $0.3$, we can reduce the total computation by a factor of around $10$ by selecting $|\mathcal{T}| = 2^{10}$ and paying a reduction in micro-F1 score of $0.023$.
Finally, it can be observed in Table~\ref{table:fast-lp}(b) that, for link prediction, although the performance slightly improves as $|\mathcal{T}|$ increases, it remains relatively stable.
Thus, if an accuracy of around $0.9$ is acceptable, it can be achieved with a largely reduced computational effort.

\begin{table}
\caption{\small Link prediction results on the CoCit dataset.}\label{table:lp_cocit}
\begin{center}
\begin{tabular}{l|p{0.57cm}p{0.57cm}p{0.57cm}p{0.63cm}} 
 \Xhline{0.8pt}
	Method & Avg & Hada & L1 & L2 \\ [0.2em]
 	\hline
 	AROPE & 0.769 & \underline{0.874} & 0.695 & 0.679  \\ [0.2em]
 	ProNE & 0.710 & \underline{0.950} & 0.939 & 0.938 \\ [0.2em]
 	NRP   & 0.727 & \underline{0.955} & 0.951 & 0.952  \\ [0.2em]
 	GMF-FE & 0.726 &  \textbf{\underline{0.967}} &  0.961 & 0.961  \\
  \Xhline{0.8pt}
\end{tabular}
\end{center}
\end{table}

Finally, we consider a larger dataset and a different sampling strategy.
We adopt the CoCit dataset~\cite{zhou2021direction,tsitsulin2021frede}, which is a citation network.
Each node (publication) is associated with a single label indicating the conference where it is published.
The dataset statistics can be found in Table~\ref{table-datasets}.
Its degree distribution follows a power law and most nodes have a relatively small degree.
We select nodes of degree greater or equal to $20$ as the target nodes in $\mathcal{T}$, accounting for around $10\%$ of the whole node set.
We compute the embeddings based on approximate directed free energy dissimilarities $\mathbf{\Phi}^{\mathrm{FE}}_{\eta, L, \mathcal{T}}$ and set $L=10$.
Similar to the previous section, the optimal $\eta$ is obtained via a line search.
We compare our performance with three highly scalable baselines.
Numerical results for node classification and link prediction are given in Fig.~\ref{fig:nc_cocit} and Table~\ref{table:lp_cocit}, respectively.
It can be observed that GMF-FE outperforms all the baseline methods on both tasks, even with incomplete and approximate free energy dissimilarities.

\section{Conclusions}\label{s:conclusion}
We proposed a node embedding method based on the FE distance and a generalized skip-gram matrix factorization.
The parametric FE distance equips the method with the flexibility needed to adapt to different graph types and sizes. 
Moreover, the matrix factorization proposed can be applied to arbitrary similarity matrices and focuses on preserving node pairs with high similarity scores, thus making it well-suited for the task of node embedding. 
Experimental results validated the effectiveness of the proposed scheme. 

Many existing methods use (different variants of) random walks as building blocks, while this paper draws attention to similarities based on (hitting) paths and opens new possibilities for future research. For example, in real-world applications, one can consider non-uniform a priori probabilities of choosing the start and terminal nodes or extra penalization of specific undesirable paths~\cite{franccoisse2017bag}. Other directions for future work include:
i) Theoretical analysis of GMF, especially the optimal choice of $\mathbf{S}^{+}$ and $\mathbf{S}^{-}$ given a specific $\mathbf{S}$; 
ii) Generalization of the proposed node embedding method to digraphs (by leveraging the directed FE dissimilarity), multi-view graphs~\cite{crsp}, heterogeneous graphs~\cite{pte}, hypergraphs~\cite{zhu2021co, schaub2021signal_hg}, and simplicial complexes~\cite{schaub2021signal_sc};
iii) Designing better strategies for sampling the set of target nodes $\mathcal{T}$.

\ifCLASSOPTIONcompsoc
  \section*{Acknowledgments}
\else
  \section*{Acknowledgment}
\fi

Research was sponsored by the Army Research Office and was accomplished under Cooperative Agreement Number W911NF-19-2-0269.
The views and conclusions contained in this document are those of the authors and should not be interpreted as representing the official policies, either expressed or implied, of the Army Research Office or the U.S. Government. 
The U.S. Government is authorized to reproduce and distribute reprints for Government purposes notwithstanding any copy-right notation herein.

\ifCLASSOPTIONcaptionsoff
  \newpage
\fi



\bibliography{embed}

\begin{thebibliography}{10}
\providecommand{\url}[1]{#1}
\csname url@samestyle\endcsname
\providecommand{\newblock}{\relax}
\providecommand{\bibinfo}[2]{#2}
\providecommand{\BIBentrySTDinterwordspacing}{\spaceskip=0pt\relax}
\providecommand{\BIBentryALTinterwordstretchfactor}{4}
\providecommand{\BIBentryALTinterwordspacing}{\spaceskip=\fontdimen2\font plus
\BIBentryALTinterwordstretchfactor\fontdimen3\font minus
  \fontdimen4\font\relax}
\providecommand{\BIBforeignlanguage}[2]{{%
\expandafter\ifx\csname l@#1\endcsname\relax
\typeout{** WARNING: IEEEtran.bst: No hyphenation pattern has been}%
\typeout{** loaded for the language `#1'. Using the pattern for}%
\typeout{** the default language instead.}%
\else
\language=\csname l@#1\endcsname
\fi
#2}}
\providecommand{\BIBdecl}{\relax}
\BIBdecl

\bibitem{survey1}
W.~L. Hamilton, R.~Ying, and J.~Leskovec, ``Representation learning on graphs:
  Methods and applications,'' in \emph{IEEE Data Engineering Bulletin}, 2017.

\bibitem{chowdhury2020unfolding}
A.~Chowdhury, G.~Verma, C.~Rao, A.~Swami, and S.~Segarra, ``Unfolding {WMMSE}
  using graph neural networks for efficient power allocation,'' \emph{arXiv
  preprint arXiv:2009.10812}, 2020.

\bibitem{zhao2020distributed}
Z.~Zhao, G.~Verma, C.~Rao, A.~Swami, and S.~Segarra, ``Distributed scheduling
  using graph neural networks,'' \emph{arXiv preprint arXiv:2011.09430}, 2020.

\bibitem{jia2019graph}
J.~Jia, M.~T. Schaub, S.~Segarra, and A.~R. Benson, ``Graph-based
  semi-supervised \& active learning for edge flows,'' in \emph{KDD}, 2019, pp.
  761--–771.

\bibitem{roddenberry2019hodgenet}
T.~M. {Roddenberry} and S.~{Segarra}, ``{HodgeNet}: Graph neural networks for
  edge data,'' in \emph{Asilomar Conference on Signals, Systems, and
  Computers}, 2019, pp. 220--224.

\bibitem{ma2017multi}
G.~{Ma}, C.~{Lu}, L.~{He}, P.~S. {Yu}, and A.~B. {Ragin}, ``Multi-view graph
  embedding with hub detection for brain network analysis,'' in \emph{ICDM},
  2017, pp. 967--972.

\bibitem{yue2019graph}
X.~Yue, Z.~Wang, J.~Huang, S.~Parthasarathy, S.~Moosavinasab, Y.~Huang, S.~M.
  Lin, W.~Zhang, P.~Zhang, and H.~Sun, ``Graph embedding on biomedical
  networks: Methods, applications and evaluations,'' \emph{Bioinformatics},
  vol.~36, no.~4, pp. 1241--1251, 2019.

\bibitem{deepwalk}
B.~Perozzi, R.~Al-Rfou, and S.~Skiena, ``Deep{W}alk: Online learning of social
  representations,'' in \emph{KDD}, Aug. 2014, pp. 701--710.

\bibitem{node2vec}
A.~Grover and J.~Leskovec, ``node2vec: Scalable feature learning for
  networks,'' in \emph{KDD}, Aug. 2016, pp. 855--864.

\bibitem{survey2}
P.~Goyal and E.~Ferrara, ``Graph embedding techniques, applications, and
  performance: A survey,'' \emph{Knowledge-Based Systems}, vol. 151, pp.
  78--94, Jul. 2018.

\bibitem{survey3}
M.~Khosla, V.~Setty, and A.~Anand, ``A comparative study for unsupervised
  network representation learning,'' \emph{IEEE Transactions on Knowledge and
  Data Engineering}, Nov. 2019.

\bibitem{survey4}
D.~Zhang, J.~Yin, X.~Zhu, and C.~Zhang, ``Network representation learning: A
  survey,'' \emph{IEEE Transactions on Big Data}, vol.~6, no.~1, pp. 3--28,
  Mar. 2020.

\bibitem{hope}
M.~Ou, P.~Cui, J.~Pei, Z.~Zhang, and W.~Zhu, ``Asymmetric transitivity
  preserving graph embedding,'' in \emph{KDD}, Aug. 2016, pp. 1105--1114.

\bibitem{grarep}
S.~Cao, W.~Lu, and Q.~Xu, ``Gra{R}ep: Learning graph representations with
  global structural information,'' in \emph{CIKM}, Oct. 2015, pp. 891--900.

\bibitem{netmf}
J.~Qiu, Y.~Dong, H.~Ma, J.~Li, K.~Wang, and J.~Tang, ``Network embedding as
  matrix factorization: Unifying {D}eep{W}alk, {LINE}, {PTE}, and node2vec,''
  in \emph{WSDM}, Feb. 2018, pp. 459--467.

\bibitem{AROPE}
Z.~Zhang, P.~Cui, X.~Wang, J.~Pei, X.~Yao, and W.~Zhu, ``Arbitrary-order
  proximity preserved network embedding,'' in \emph{KDD}, 2018, pp. 2778--2786.

\bibitem{prone}
J.~Zhang, Y.~Dong, Y.~Wang, J.~Tang, and M.~Ding, ``{ProNE: F}ast and scalable
  network representation learning.'' in \emph{IJCAI}, vol.~19, 2019, pp.
  4278--4284.

\bibitem{word2vec_1}
T.~Mikolov, K.~Chen, G.~Corrado, and J.~Dean, ``Efficient estimation of word
  representations in vector space,'' in \emph{ICLR Workshop}, 2013.

\bibitem{word2vec_2}
T.~Mikolov, I.~Sutskever, K.~Chen, G.~Corrado, and J.~Dean, ``Distributed
  representations of words and phrases and their compositionality,'' in
  \emph{NIPS}, Dec. 2013, pp. 3111--3119.

\bibitem{nce1}
M.~U. Gutmann and A.~Hyv\"{a}rinen, ``Noise-contrastive estimation of
  unnormalized statistical models, with applications to natural image
  statistics,'' \emph{JMLR}, vol.~13, no.~1, pp. 307--361, Feb. 2012.

\bibitem{nce2}
A.~Mnih and Y.~W. Teh, ``A fast and simple algorithm for training neural
  probabilistic language models,'' in \emph{ICML}, Jun. 2012, pp. 419--426.

\bibitem{fe_dist}
I.~Kivim\"{a}ki, M.~Shimbo, and M.~Saerens, ``Developments in the theory of
  randomized shortest paths with a comparison of graph node distances,''
  \emph{Physica A: Statistical Mechanics and its Applications}, vol. 393, pp.
  600--616, Jan. 2014.

\bibitem{simnet}
M.~Khajehnejad, ``{S}im{N}et: Similarity-based network embeddings with mean
  commute time,'' \emph{Plos one}, vol.~14, no.~8, Aug. 2019.

\bibitem{ct_problem}
M.~Brand, ``A random walks perspective on maximizing satisfaction and profit,''
  in \emph{SDM}, 2005, pp. 12--19.

\bibitem{ct_problem2}
A.~Radl, U.~von Luxburg, and M.~Hein, ``The resistance distance is meaningless
  for large random geometric graphs,'' in \emph{NIPS Workshop}, 2009.

\bibitem{ct_problem3}
U.~von Luxburg, A.~Radl, and M.~Hein, ``Getting lost in space: Large sample
  analysis of the commute distance,'' in \emph{NIPS}, 2010, pp. 2622--2630.

\bibitem{line}
J.~Tang, M.~Qu, M.~Wang, M.~Zhang, J.~Yan, and Q.~Mei, ``{LINE}: Large-scale
  information network embedding,'' in \emph{WWW}, May 2015, pp. 1067--1077.

\bibitem{walklets}
B.~Perozzi, V.~Kulkarni, H.~Chen, and S.~Skiena, ``Don't walk, skip! {O}nline
  learning of multi-scale network embeddings,'' in \emph{ASONAM}, Jul. 2017,
  pp. 258--265.

\bibitem{app}
C.~Zhou, Y.~Liu, X.~Liu, Z.~Liu, and J.~Gao, ``Scalable graph embedding for
  asymmetric proximity,'' in \emph{AAAI}, Feb. 2017, pp. 2942--2948.

\bibitem{metapath2vec}
Y.~Dong, N.~V. Chawla, and A.~Swami, ``metapath2vec: Scalable representation
  learning for heterogeneous networks,'' in \emph{KDD}, 2017, pp. 135--144.

\bibitem{mines}
Y.~Ma, Z.~Ren, Z.~Jiang, J.~Tang, and D.~Yin, ``Multi-dimensional network
  embedding with hierarchical structure,'' in \emph{WSDM}, 2018, pp. 387--395.

\bibitem{verse}
A.~Tsitsulin, D.~Mottin, P.~Karras, and E.~Müller, ``{VERSE}: Versatile graph
  embeddings from similarity measures,'' in \emph{WWW}, Apr. 2018, pp.
  539--548.

\bibitem{tf}
M.~Abadi \emph{et~al.}, ``Tensorflow: A system for large-scale machine
  learning,'' in \emph{OSDI}, 2016, pp. 265--283.

\bibitem{gra}
X.~Liu, T.~Murata, K.-S. Kim, C.~Kotarasu, and C.~Zhuang, ``A general view for
  network embedding as matrix factorization,'' in \emph{WSDM}, Jan. 2019, pp.
  375--383.

\bibitem{netsmf}
J.~Qiu, Y.~Dong, H.~Ma, J.~Li, C.~Wang, K.~Wang, and J.~Tang, ``{N}et{SMF}:
  Large-scale network embedding as sparse matrix factorization,'' in
  \emph{WWW}, May 2019, pp. 1509--1520.

\bibitem{SampleEndPoint}
D.~Fogaras, B.~R\'{a}cz, K.~Csalog\'{a}ny, and T.~Sarl\'{o}s, ``Towards scaling
  fully personalized page{R}ank: Algorithms, lower bounds, and experiments,''
  \emph{Internet Mathematics}, vol.~2, no.~3, pp. 333--358, 2005.

\bibitem{svd1}
M.~W. Berry, ``Large-scale sparse singular value computations,'' \emph{The
  International Journal of Supercomputing Applications}, vol.~6, no.~1, pp.
  13--49, 1992.

\bibitem{svd2}
R.~B. Lehoucq, D.~C. Sorensen, and C.~Yang, \emph{{ARPACK} users' guide:
  solution of large-scale eigenvalue problems with implicitly restarted
  {A}rnoldi methods}.\hskip 1em plus 0.5em minus 0.4em\relax SIAM, 1998.

\bibitem{svd4}
N.~Halko, P.-G. Martinsson, and J.~A. Tropp, ``Finding structure with
  randomness: Probabilistic algorithms for constructing approximate matrix
  decompositions,'' \emph{SIAM Review}, vol.~53, no.~2, pp. 217--288, 2011.

\bibitem{svd3}
C.~C. Paige and M.~A. Saunders, ``Towards a generalized singular value
  decomposition,'' \emph{SIAM Journal on Numerical Analysis}, vol.~18, no.~3,
  pp. 398--405, 1981.

\bibitem{svd5}
M.~Hochstenbach, ``A {Jacobi–Davidson} type method for the generalized
  singular value problem,'' \emph{Linear Algebra and its Applications}, vol.
  431, no. 3-4, pp. 471--487, Jul. 2009.

\bibitem{word_mf}
O.~Levy and Y.~Goldberg, ``Neural word embedding as implicit matrix
  factorization,'' in \emph{NIPS}, Dec. 2014, pp. 2177--2185.

\bibitem{franccoisse2017bag}
K.~Fran{\c{c}}oisse, I.~Kivim{\"a}ki, A.~Mantrach, F.~Rossi, and M.~Saerens,
  ``A bag-of-paths framework for network data analysis,'' \emph{Neural
  Networks}, vol.~90, pp. 90--111, 2017.

\bibitem{devooght2014random}
R.~Devooght, A.~Mantrach, I.~Kivim{\"a}ki, H.~Bersini, A.~Jaimes, and
  M.~Saerens, ``Random walks based modularity: application to semi-supervised
  learning,'' in \emph{WWW}, 2014, pp. 213--224.

\bibitem{fe_fast}
S.~Courtain, B.~Lebichot, I.~Kivim\"{a}ki, and M.~Saerens, ``Graph-based fraud
  detection with the free energy distance,'' in \emph{COMPLEX NETWORKS}, 2019,
  pp. 40--52.

\bibitem{peliti2011statistical}
L.~Peliti, \emph{Statistical mechanics in a nutshell}.\hskip 1em plus 0.5em
  minus 0.4em\relax Princeton University Press, 2011.

\bibitem{chebotarev2011class}
P.~Chebotarev, ``A class of graph-geodetic distances generalizing the
  shortest-path and the resistance distances,'' \emph{Discrete Applied
  Mathematics}, vol. 159, no.~5, pp. 295--302, 2011.

\bibitem{alamgir2011phase}
M.~Alamgir and U.~Luxburg, ``Phase transition in the family of p-resistances,''
  \emph{NIPS}, vol.~24, pp. 379--387, 2011.

\bibitem{yen2008family}
L.~Yen, M.~Saerens, A.~Mantrach, and M.~Shimbo, ``A family of dissimilarity
  measures between nodes generalizing both the shortest-path and the
  commute-time distances,'' in \emph{KDD}, 2008, pp. 785--793.

\bibitem{saerens2009randomized}
M.~Saerens, Y.~Achbany, F.~Fouss, and L.~Yen, ``Randomized shortest-path
  problems: Two related models,'' \emph{Neural Computation}, vol.~21, no.~8,
  pp. 2363--2404, 2009.

\bibitem{Hilbert_MLE}
E.~Newell, K.~Kenyon-Dean, and J.~C.~K. Cheung, ``Deconstructing and
  reconstructing word embedding algorithms,'' in \emph{arXiv preprint
  arXiv:1911.13280}, 2019.

\bibitem{pytorch}
A.~Paszke \emph{et~al.}, ``Pytorch: An imperative style, high-performance deep
  learning library,'' in \emph{NIPS}, 2019, pp. 8026--8037.

\bibitem{hamilton2020graph}
W.~L. Hamilton, ``Graph representation learning,'' \emph{Synthesis Lectures on
  Artifical Intelligence and Machine Learning}, vol.~14, no.~3, pp. 1--159,
  2020.

\bibitem{citeseer}
C.~L. Giles, K.~D. Bollacker, and S.~Lawrence, ``Cite{S}eer: An automatic
  citation indexing system,'' in \emph{DL}, 1998, pp. 89--98.

\bibitem{cora}
A.~K. McCallum, K.~Nigam, J.~Rennie, and K.~Seymore, ``Automating the
  construction of internet portals with machine learning,'' \emph{Information
  Retrieval}, vol.~3, no.~2, pp. 127--163, 2000.

\bibitem{BlogCatalog}
L.~Tang and H.~Liu, ``Relational learning via latent social dimensions,'' in
  \emph{KDD}, Jun. 2009, pp. 817--826.

\bibitem{chanpuriya2020infinitewalk}
S.~Chanpuriya and C.~Musco, ``Infinitewalk: {D}eep network embeddings as
  {L}aplacian embeddings with a nonlinearity,'' in \emph{KDD}, 2020, pp.
  1325--1333.

\bibitem{yang2020homogeneous}
R.~Yang, J.~Shi, X.~Xiao, Y.~Yang, and S.~S. Bhowmick, ``Homogeneous network
  embedding for massive graphs via reweighted personalized pagerank,''
  \emph{Proceedings of the VLDB Endowment}, vol.~13, no.~5, pp. 670--683, 2020.

\bibitem{lloyd1982least}
S.~Lloyd, ``Least squares quantization in {PCM},'' \emph{IEEE transactions on
  information theory}, vol.~28, no.~2, pp. 129--137, 1982.

\bibitem{kuhn1955hungarian}
H.~W. Kuhn, ``The {H}ungarian method for the assignment problem,'' \emph{Naval
  research logistics quarterly}, vol.~2, no. 1-2, pp. 83--97, 1955.

\bibitem{emmons2016analysis}
S.~Emmons, S.~Kobourov, M.~Gallant, and K.~B{\"o}rner, ``Analysis of network
  clustering algorithms and cluster quality metrics at scale,'' \emph{PloS
  one}, vol.~11, no.~7, p. e0159161, 2016.

\bibitem{von2007tutorial}
U.~Von~Luxburg, ``A tutorial on spectral clustering,'' \emph{Statistics and
  computing}, vol.~17, no.~4, pp. 395--416, 2007.

\bibitem{zhou2021direction}
S.~Zhou, X.~Wang, M.~Ester, B.~Li, C.~Ye, Z.~Zhang, C.~Wang, and J.~Bu,
  ``Direction-aware user recommendation based on asymmetric network
  embedding,'' \emph{ACM Transactions on Information Systems}, vol.~40, no.~2,
  pp. 1--23, 2021.

\bibitem{tsitsulin2021frede}
A.~Tsitsulin, M.~Munkhoeva, D.~Mottin, P.~Karras, I.~Oseledets, and
  E.~M{\"u}ller, ``{FREDE}: anytime graph embeddings,'' \emph{Proceedings of
  the VLDB Endowment}, vol.~14, no.~6, pp. 1102--1110, 2021.

\bibitem{crsp}
A.~Gamage, B.~Rappaport, S.~Aeron, and X.~Hu, ``Common randomized shortest
  paths ({C-RSP}): A simple yet effective framework for multi-view graph
  embedding,'' in \emph{ICASSP}, May 2019, pp. 3542--3546.

\bibitem{pte}
J.~Tang, M.~Qu, and Q.~Mei, ``{PTE}: Predictive text embedding through
  large-scale heterogeneous text networks,'' in \emph{KDD}, Aug. 2015, pp.
  1165--1174.

\bibitem{zhu2021co}
Y.~Zhu, B.~Li, and S.~Segarra, ``Co-clustering vertices and hyperedges via
  spectral hypergraph partitioning,'' in \emph{European Signal Processing
  Conference (EUSIPCO)}, 2021, pp. 1416--1420.

\bibitem{schaub2021signal_hg}
M.~T. Schaub, Y.~Zhu, J.-B. Seby, T.~M. Roddenberry, and S.~Segarra, ``Signal
  processing on higher-order networks: Livin’on the edge... and beyond,''
  \emph{Signal Processing}, vol. 187, p. 108149, 2021.

\bibitem{schaub2021signal_sc}
M.~T. Schaub, J.-B. Seby, F.~Frantzen, T.~M. Roddenberry, Y.~Zhu, and
  S.~Segarra, ``Signal processing on simplicial complexes,'' \emph{arXiv
  preprint arXiv:2106.07471}, 2021.

\end{thebibliography}
\bibliographystyle{IEEEtran}

%
%
%

%






\begin{titlepage}
\centering{\Large\bfseries Supplementary material}

\renewcommand\thefigure{S\arabic{figure}}    
\setcounter{figure}{0}  

\begin{figure*}[!htbp]
\centering
\includegraphics[scale=0.42]{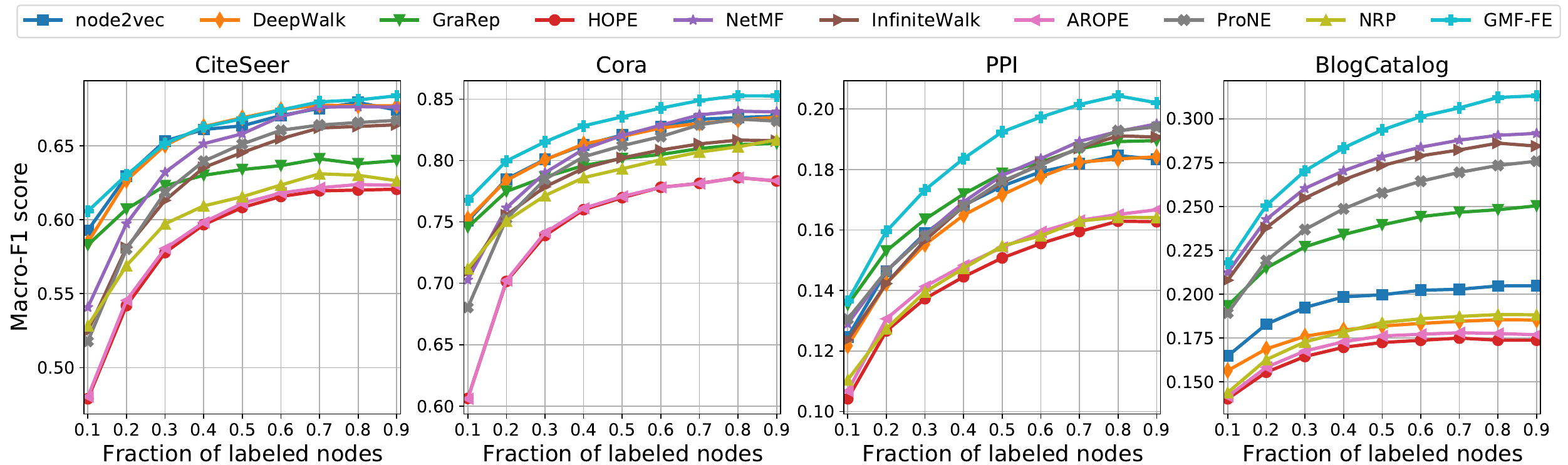}
\caption{\small Node classification results (macro-F1 scores) for different methods and datasets as a function of the fraction of labeled nodes.}
\label{fig:nc_macro}
\end{figure*}

\begin{figure*}[!htbp]
\centering
\begin{subfigure}[t]{0.24\textwidth}
\centering
\includegraphics[scale=0.2]{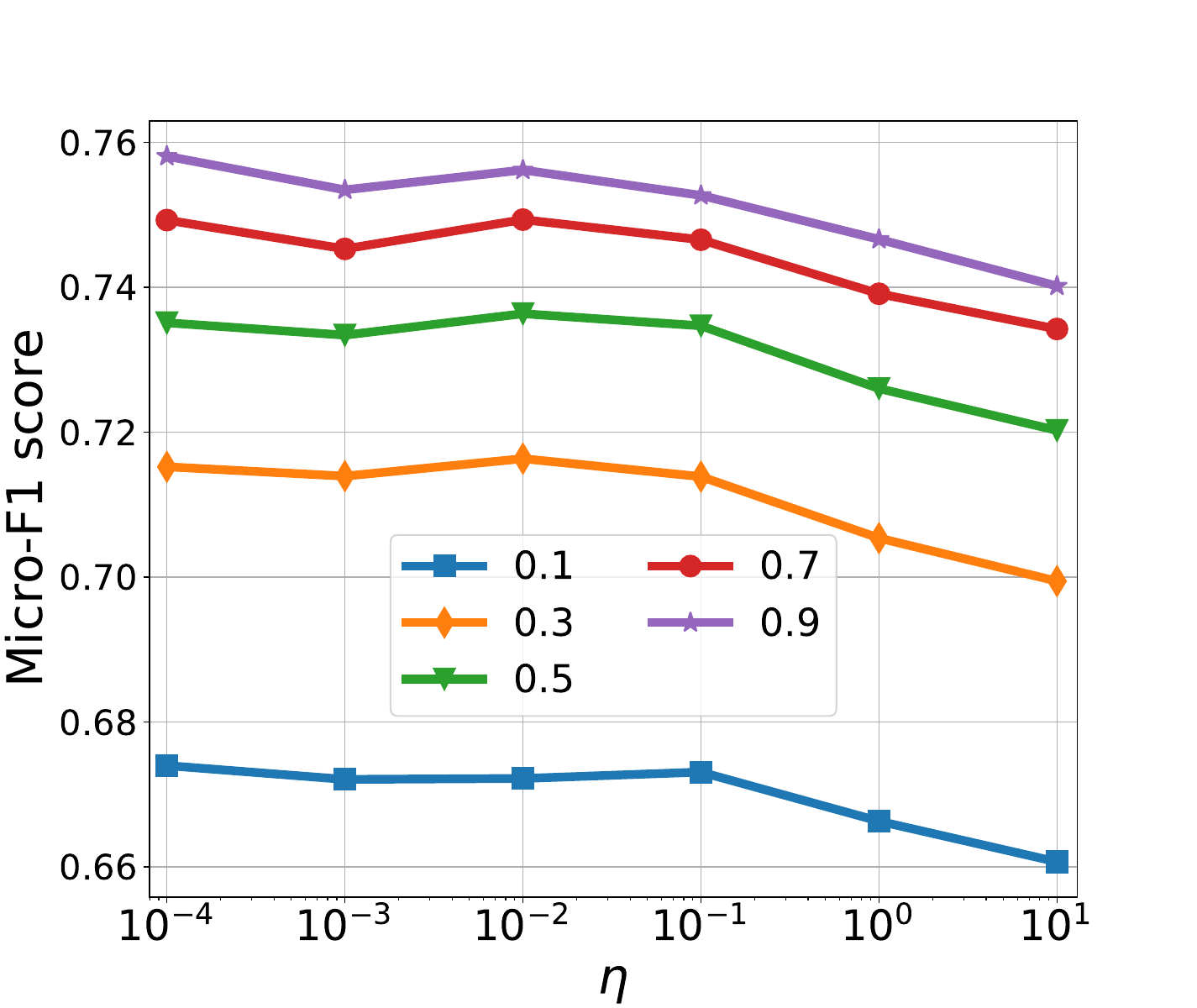}
\caption{}
\end{subfigure}
\begin{subfigure}[t]{0.24\textwidth}
\centering
\includegraphics[scale=0.2]{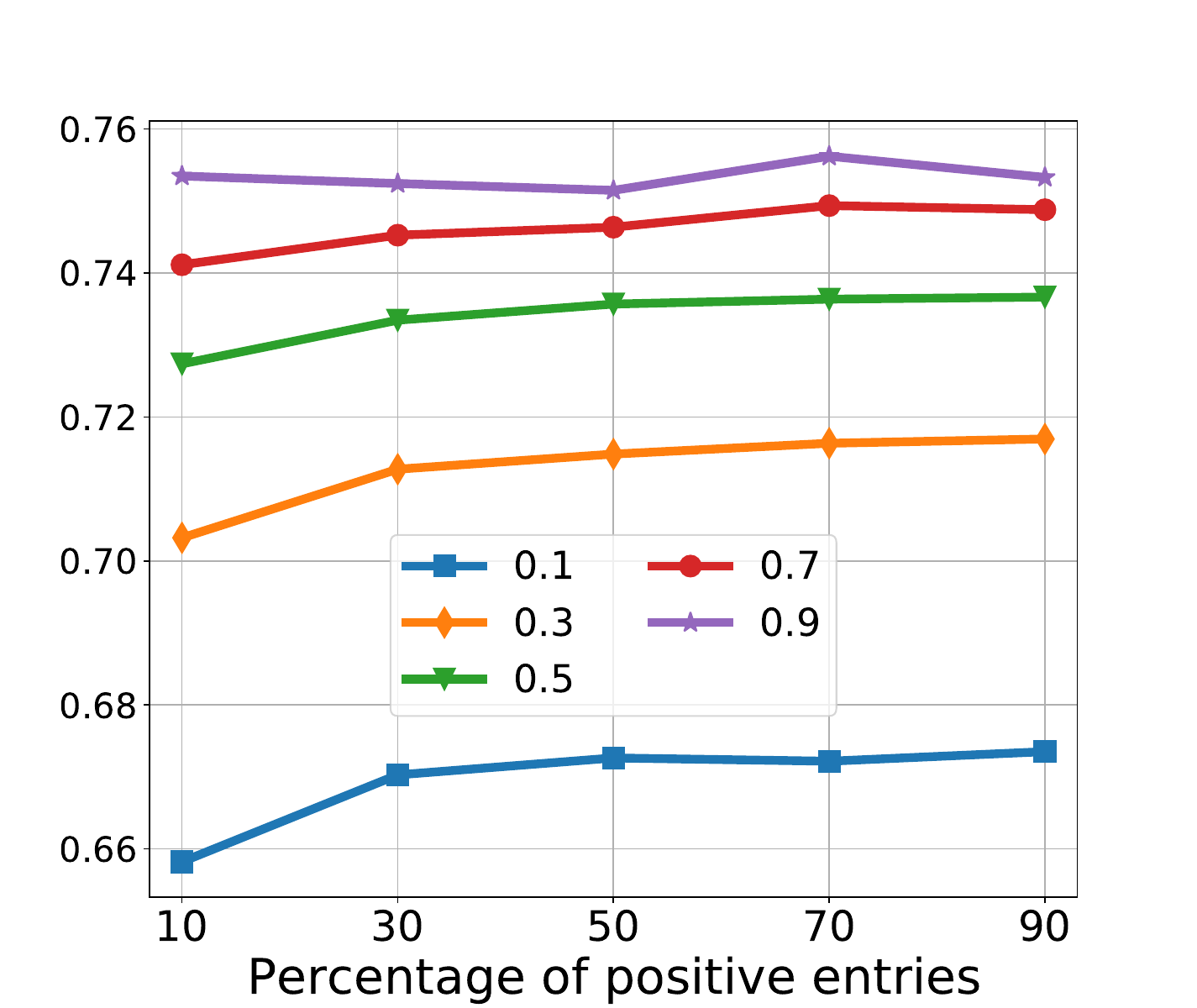}
\caption{}
\end{subfigure}
\begin{subfigure}[t]{0.24\textwidth}
\centering
\includegraphics[scale=0.2]{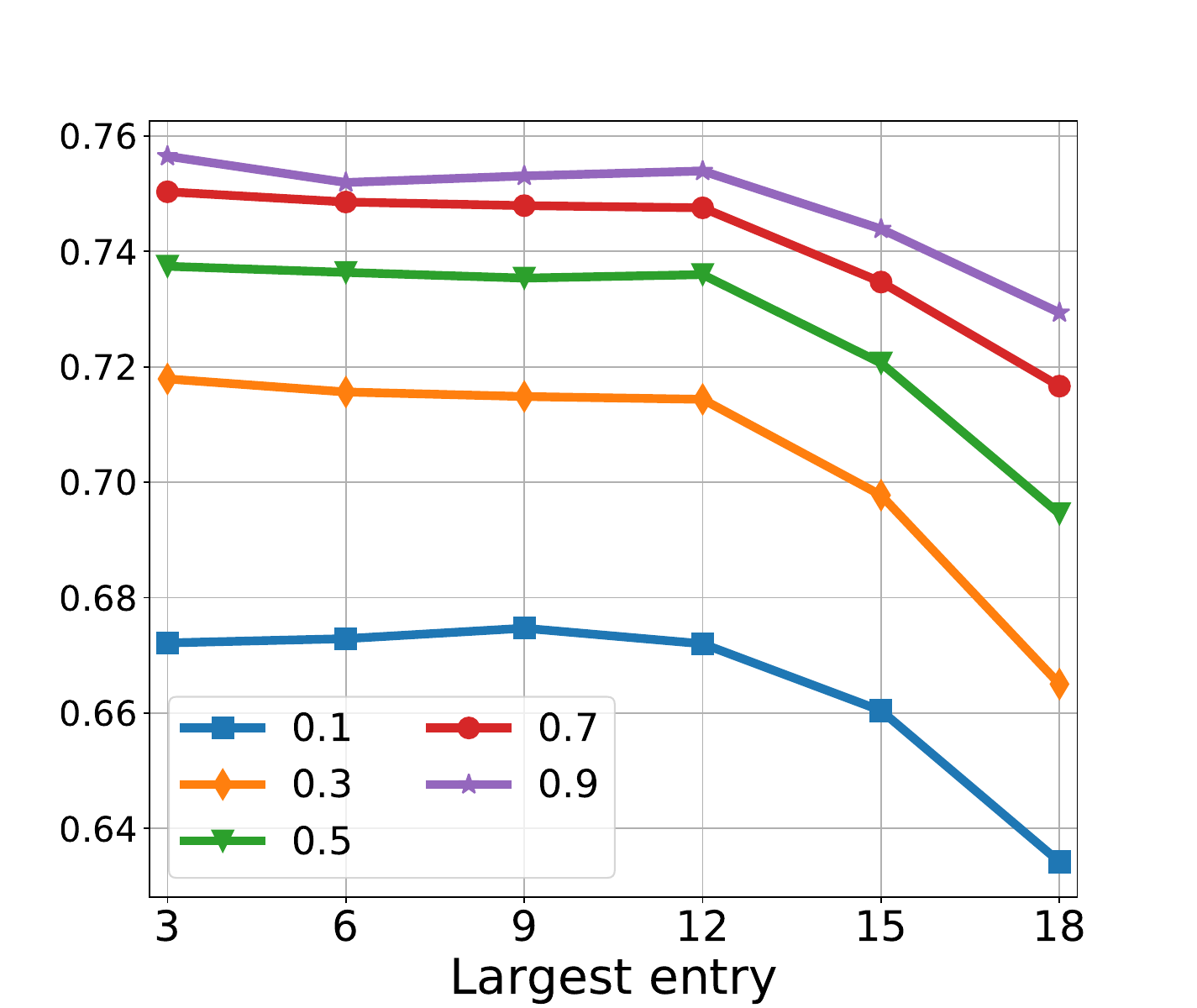}
\caption{}
\end{subfigure}
\begin{subfigure}[t]{0.24\textwidth}
\centering
\includegraphics[scale=0.2]{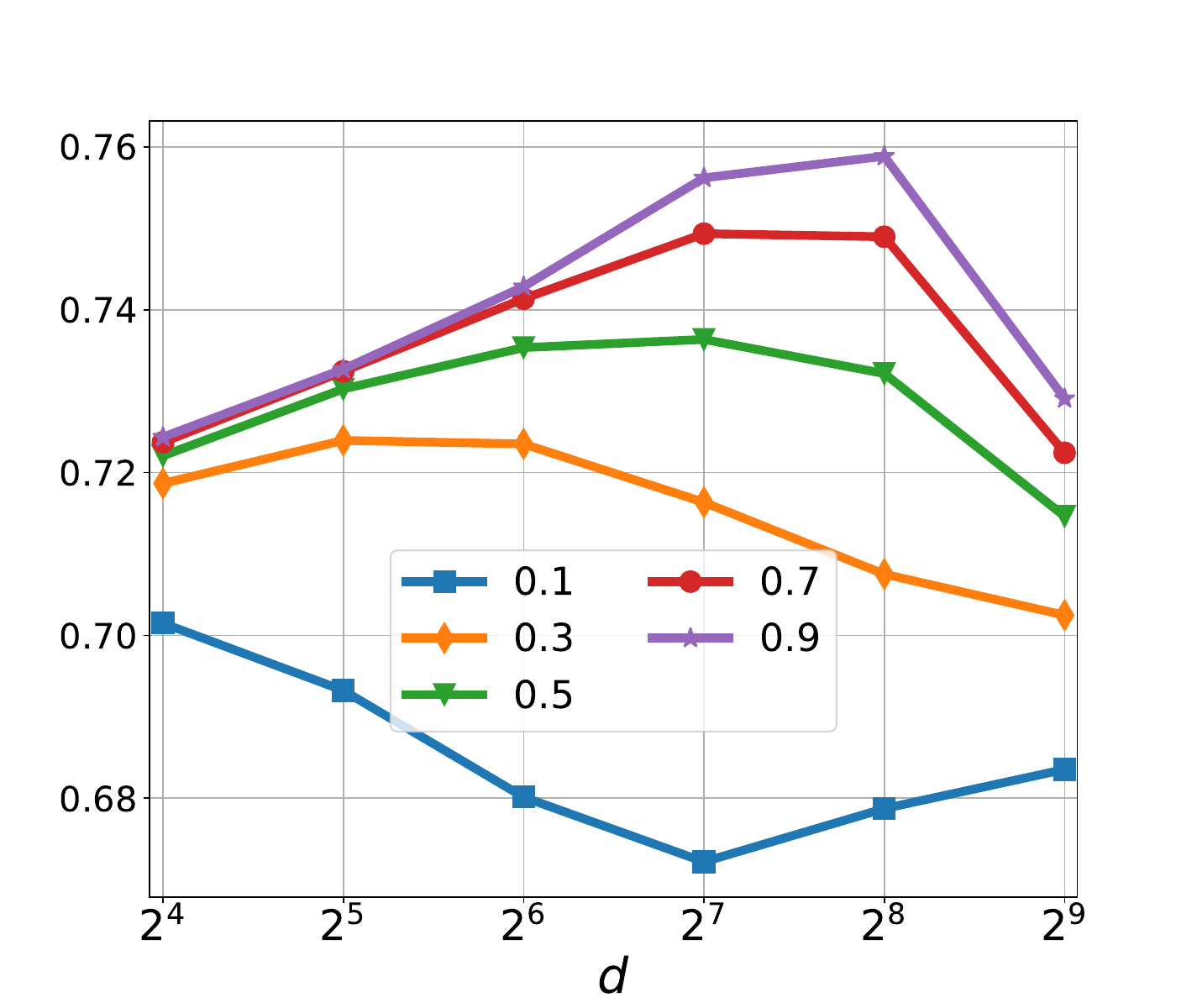}
\caption{}
\end{subfigure}
\begin{subfigure}[t]{0.24\textwidth}
\centering
\includegraphics[scale=0.2]{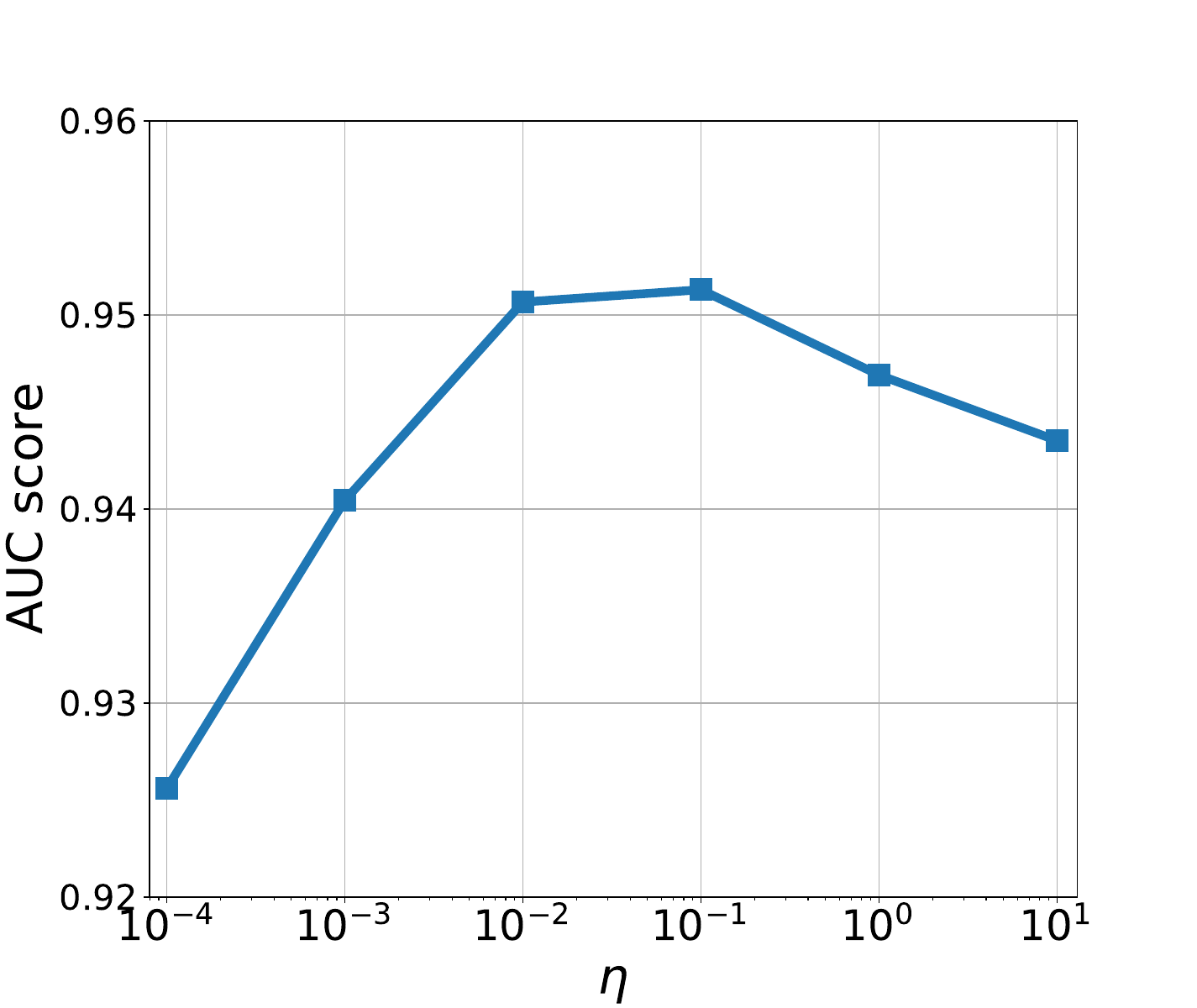}
\caption{}
\end{subfigure}
\begin{subfigure}[t]{0.24\textwidth}
\centering
\includegraphics[scale=0.2]{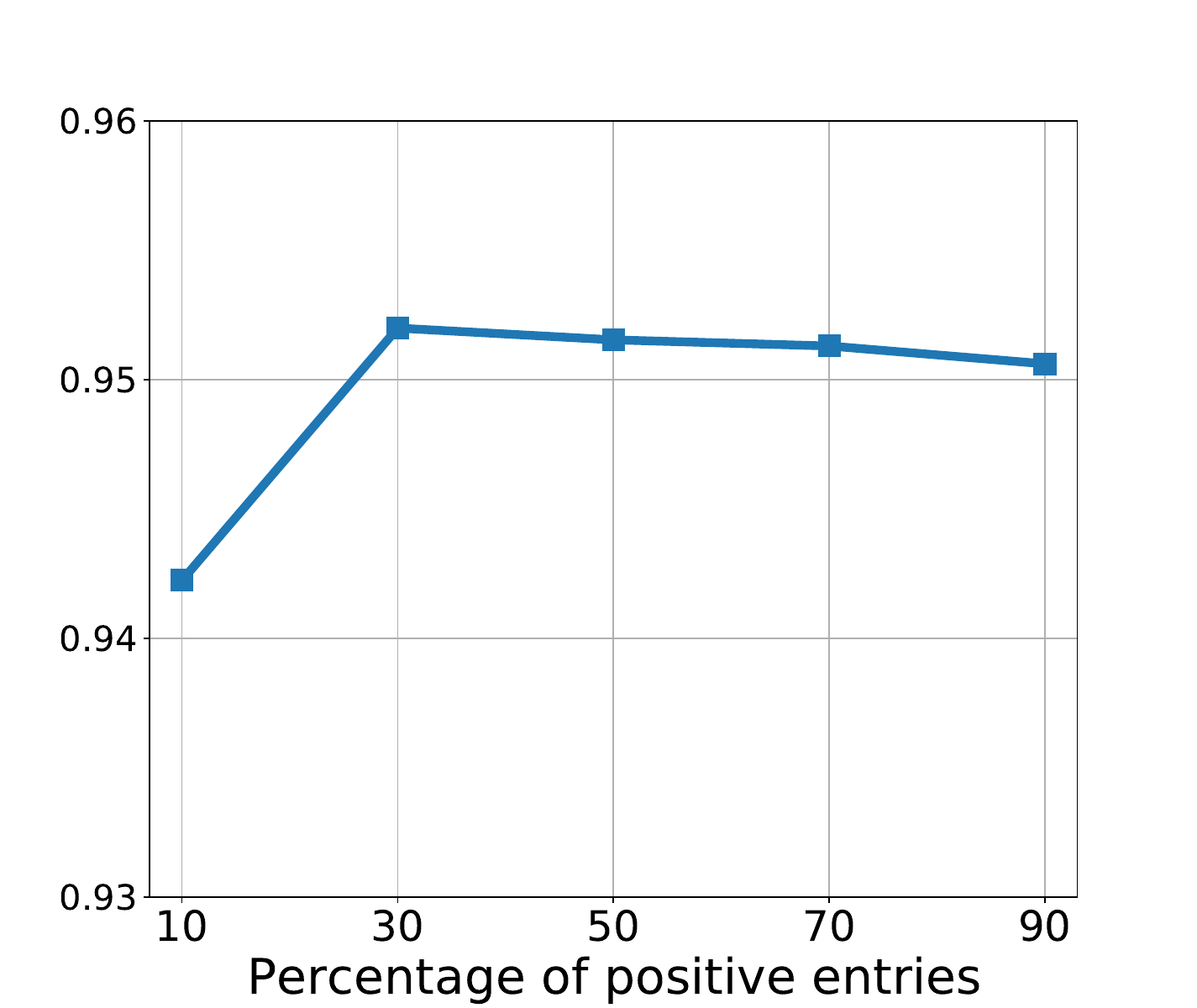}
\caption{}
\end{subfigure}
\begin{subfigure}[t]{0.24\textwidth}
\centering
\includegraphics[scale=0.2]{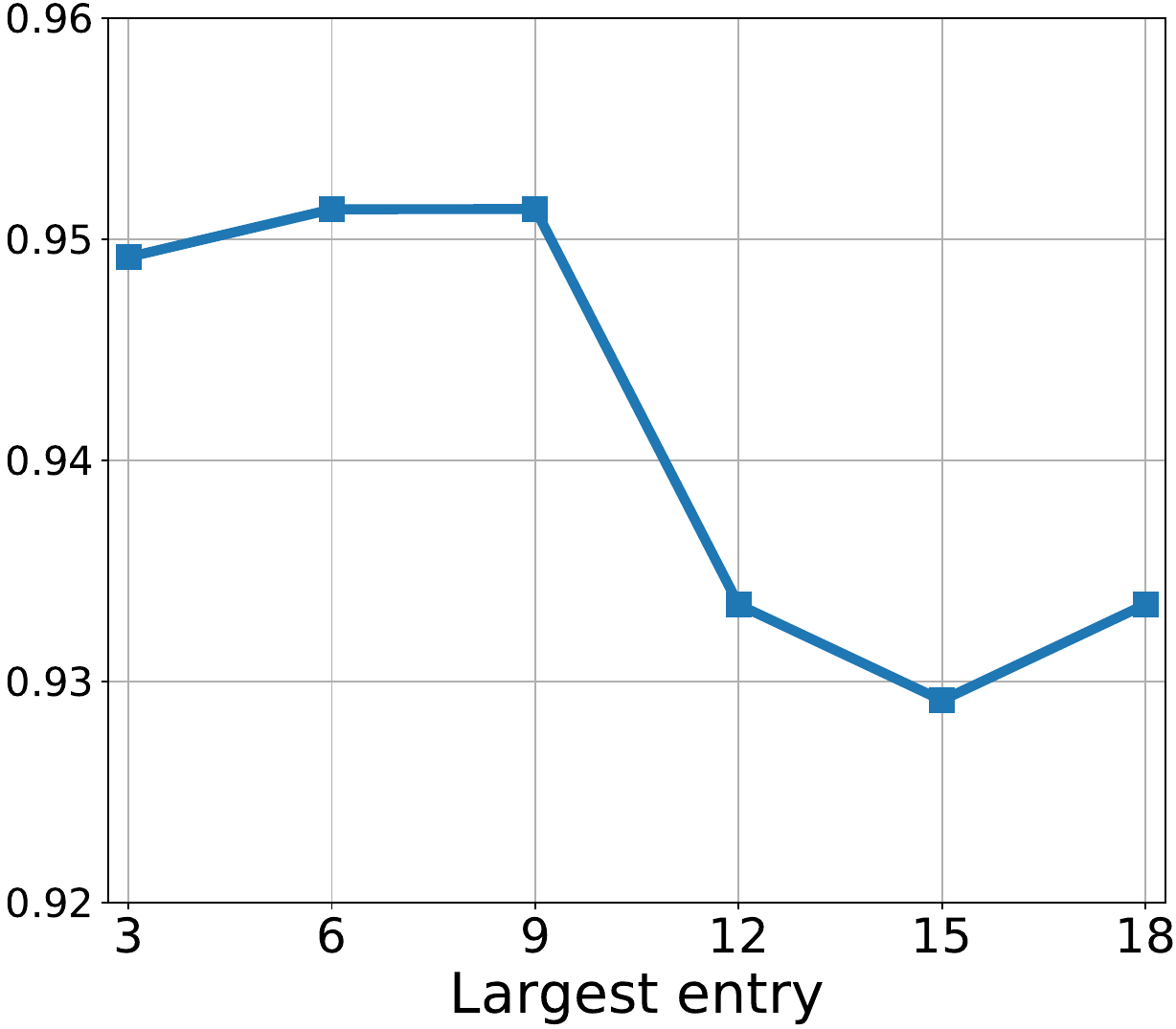}
\caption{}
\end{subfigure}
\begin{subfigure}[t]{0.24\textwidth}
\centering
\includegraphics[scale=0.2]{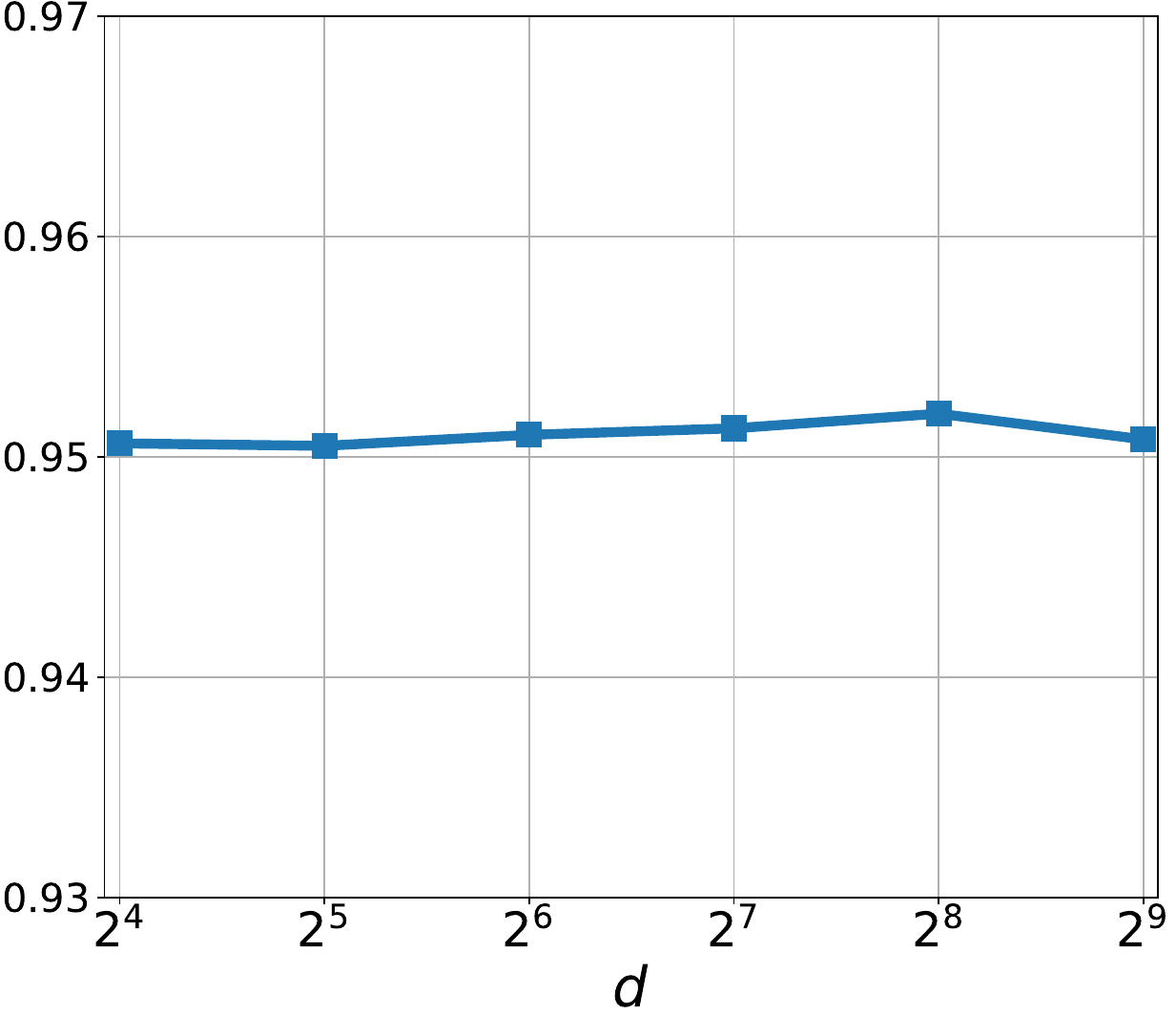}
\caption{}
\end{subfigure}
\begin{subfigure}[t]{0.24\textwidth}
\centering
\includegraphics[scale=0.2]{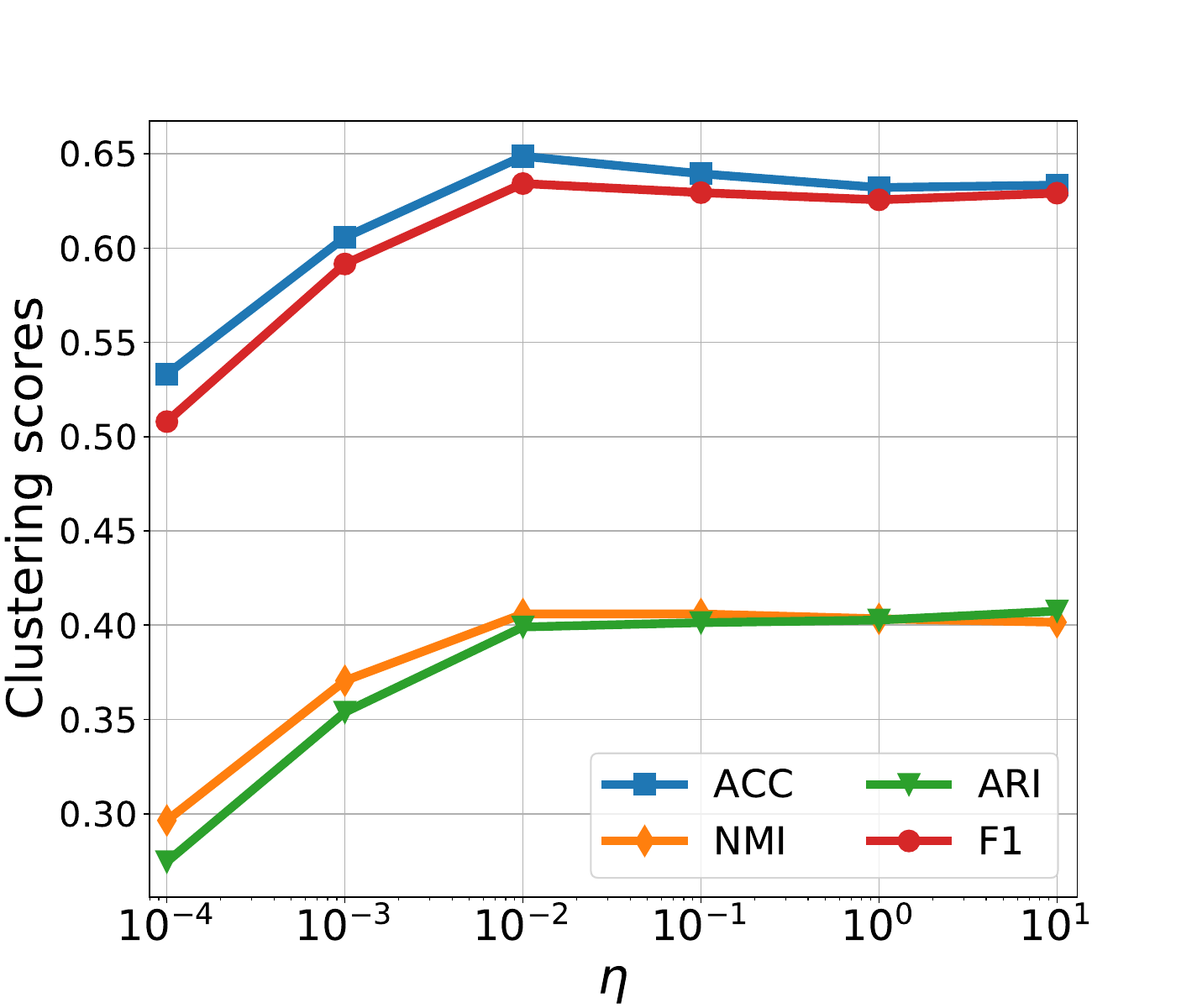}
\caption{}
\end{subfigure}
\begin{subfigure}[t]{0.24\textwidth}
\centering
\includegraphics[scale=0.2]{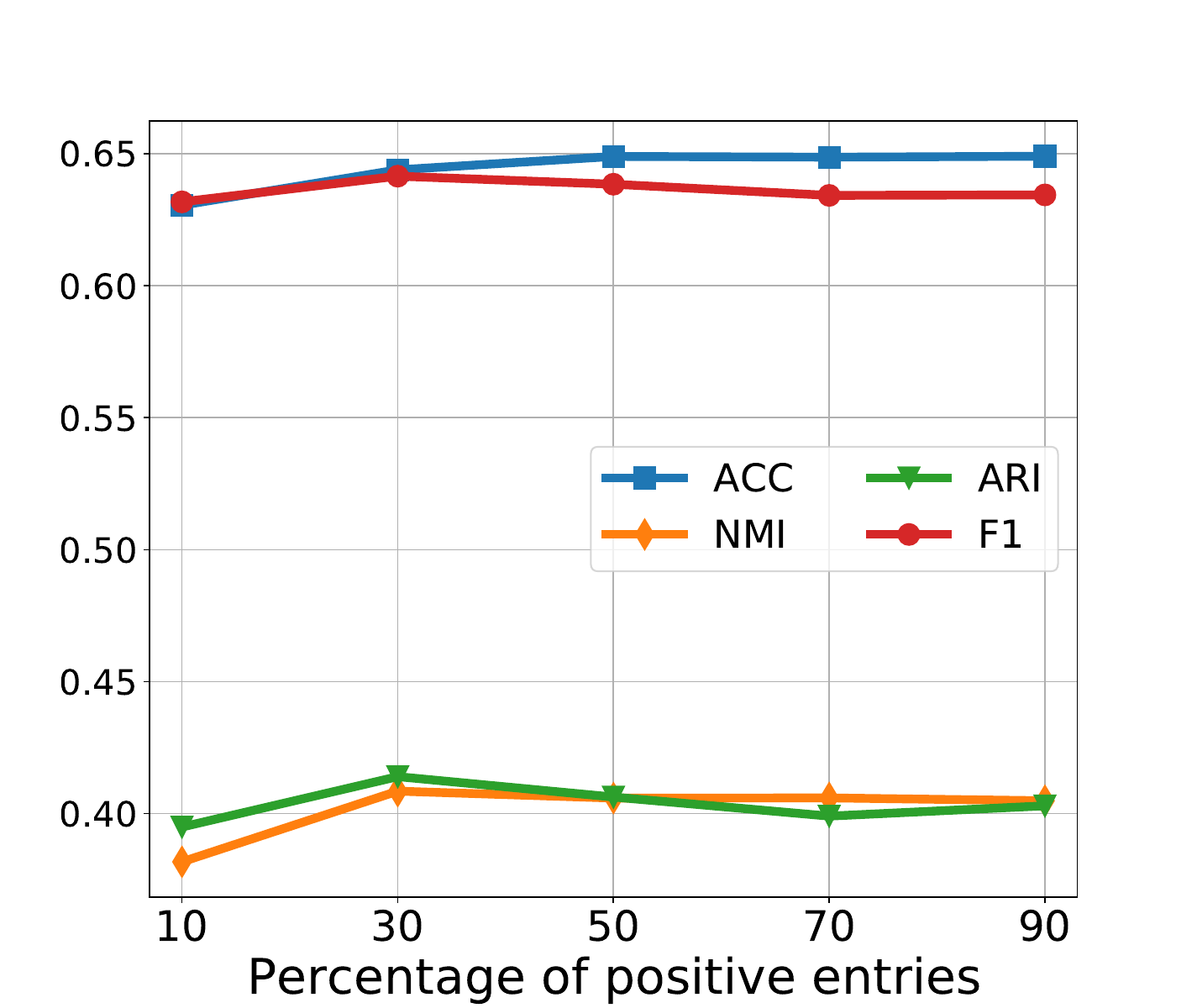}
\caption{}
\end{subfigure}
\begin{subfigure}[t]{0.24\textwidth}
\centering
\includegraphics[scale=0.2]{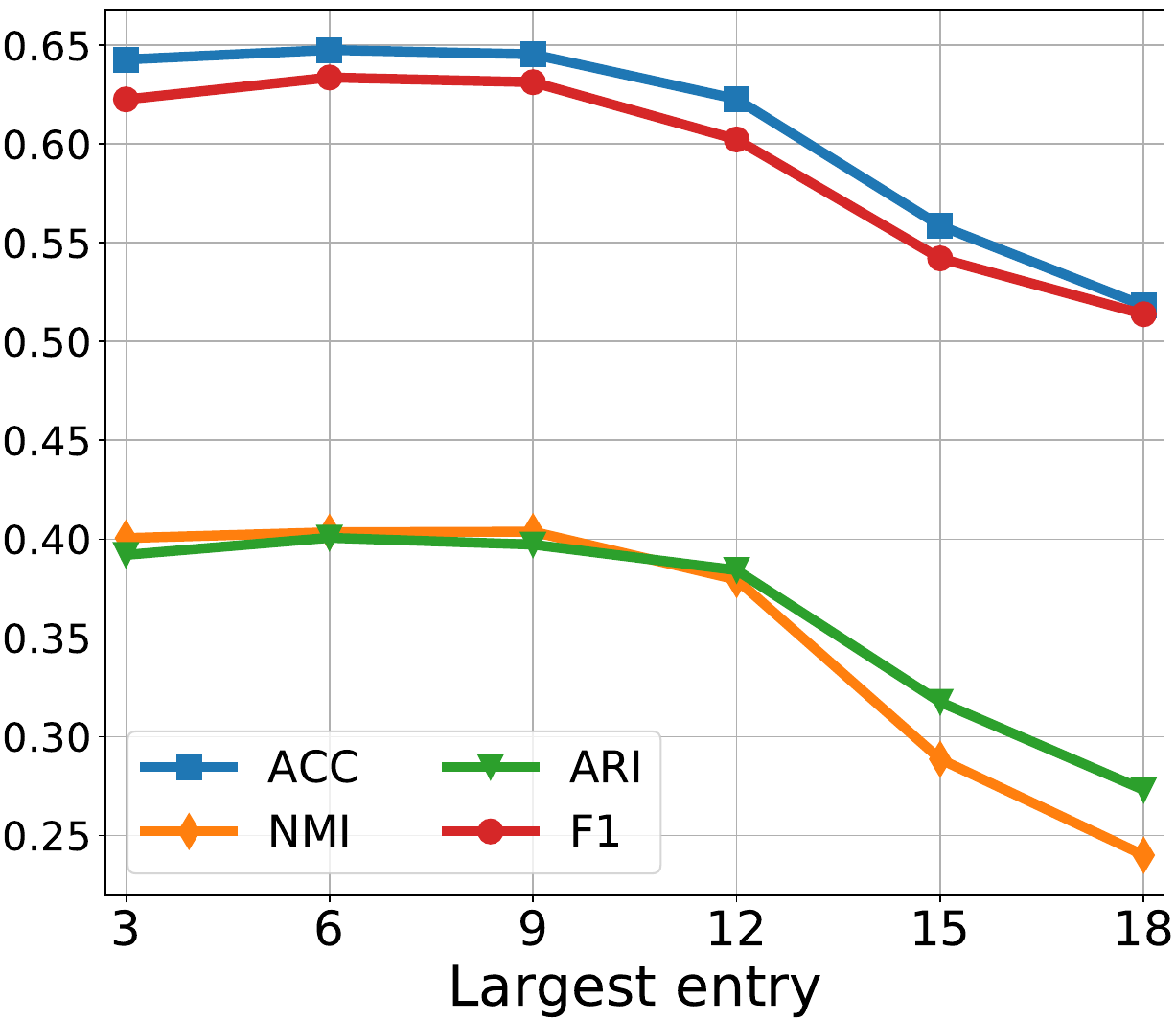}
\caption{}
\end{subfigure}
\begin{subfigure}[t]{0.24\textwidth}
\centering
\includegraphics[scale=0.2]{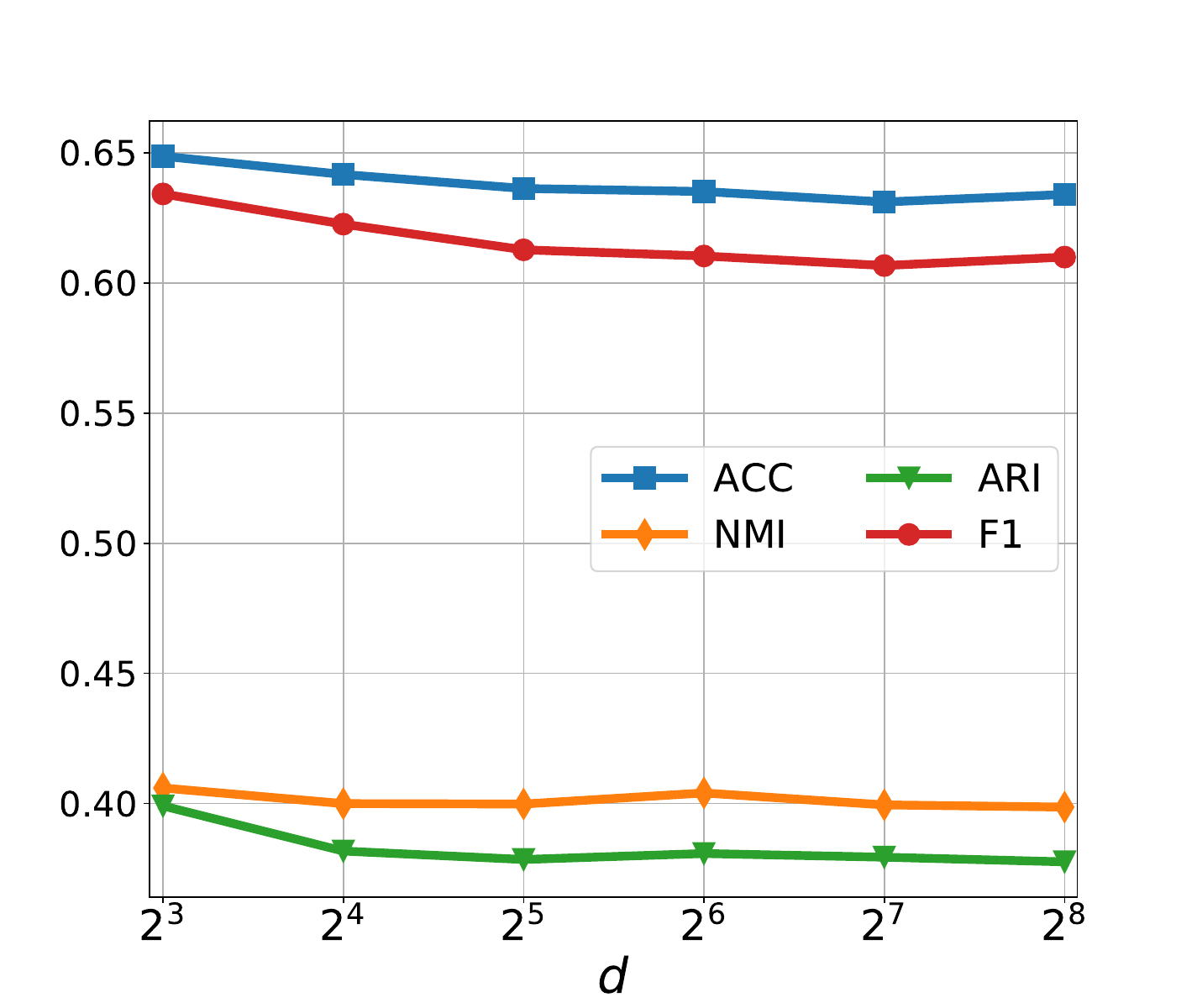}
\caption{}
\end{subfigure}
\caption{\small Parameter sensitivity test on the CiteSeer network. (a)-(d) are for node classification where the legend denotes the fraction of labeled nodes. (e)-(h) are for link prediction where the Hadamard operator is applied. (i)-(l) are for node clustering.}
\label{fig:citeseer_para}
\end{figure*}

\begin{figure*}[!htbp]
\centering
\begin{subfigure}[t]{0.24\textwidth}
\centering
\includegraphics[scale=0.2]{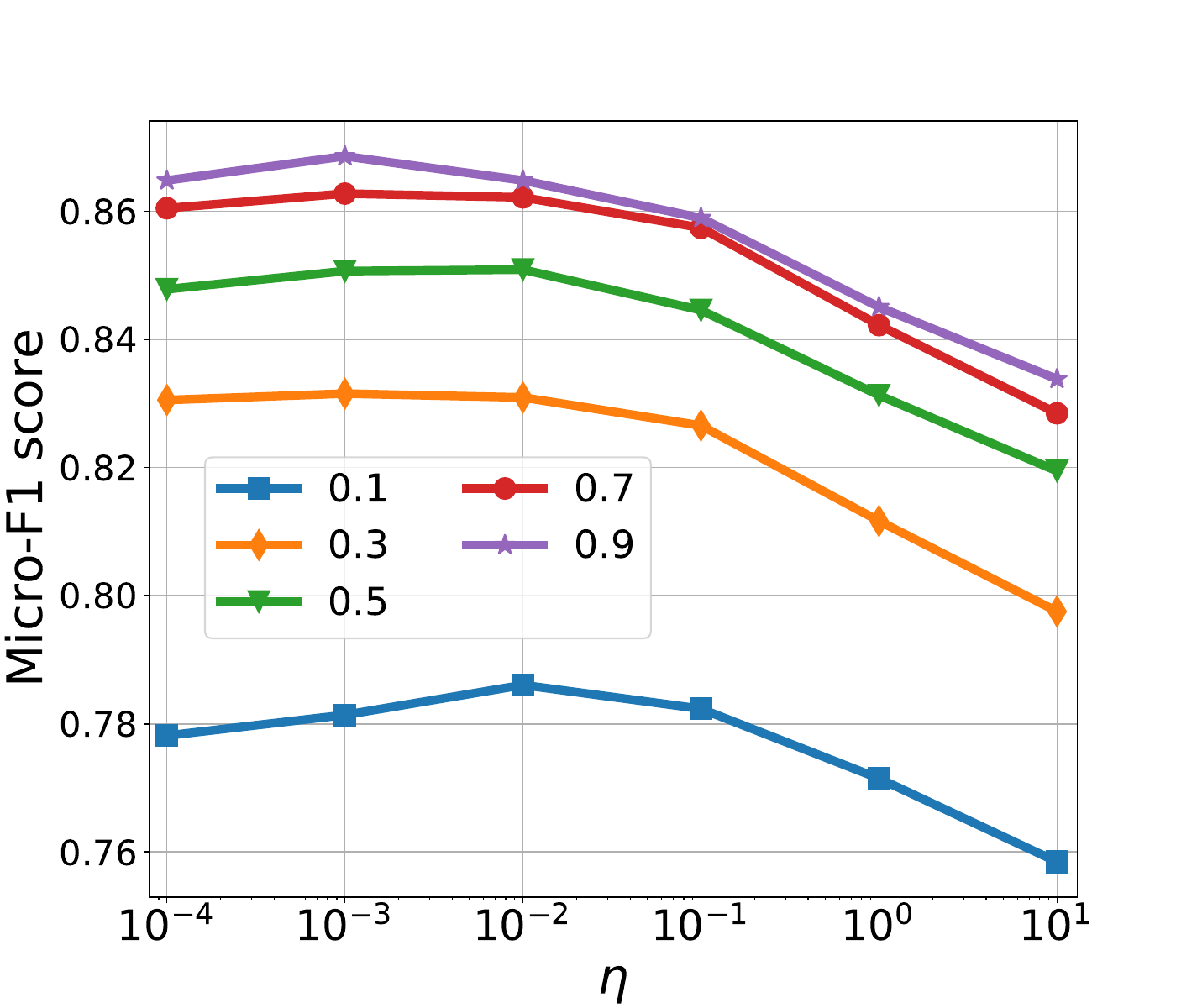}
\caption{}
\end{subfigure}
\begin{subfigure}[t]{0.24\textwidth}
\centering
\includegraphics[scale=0.2]{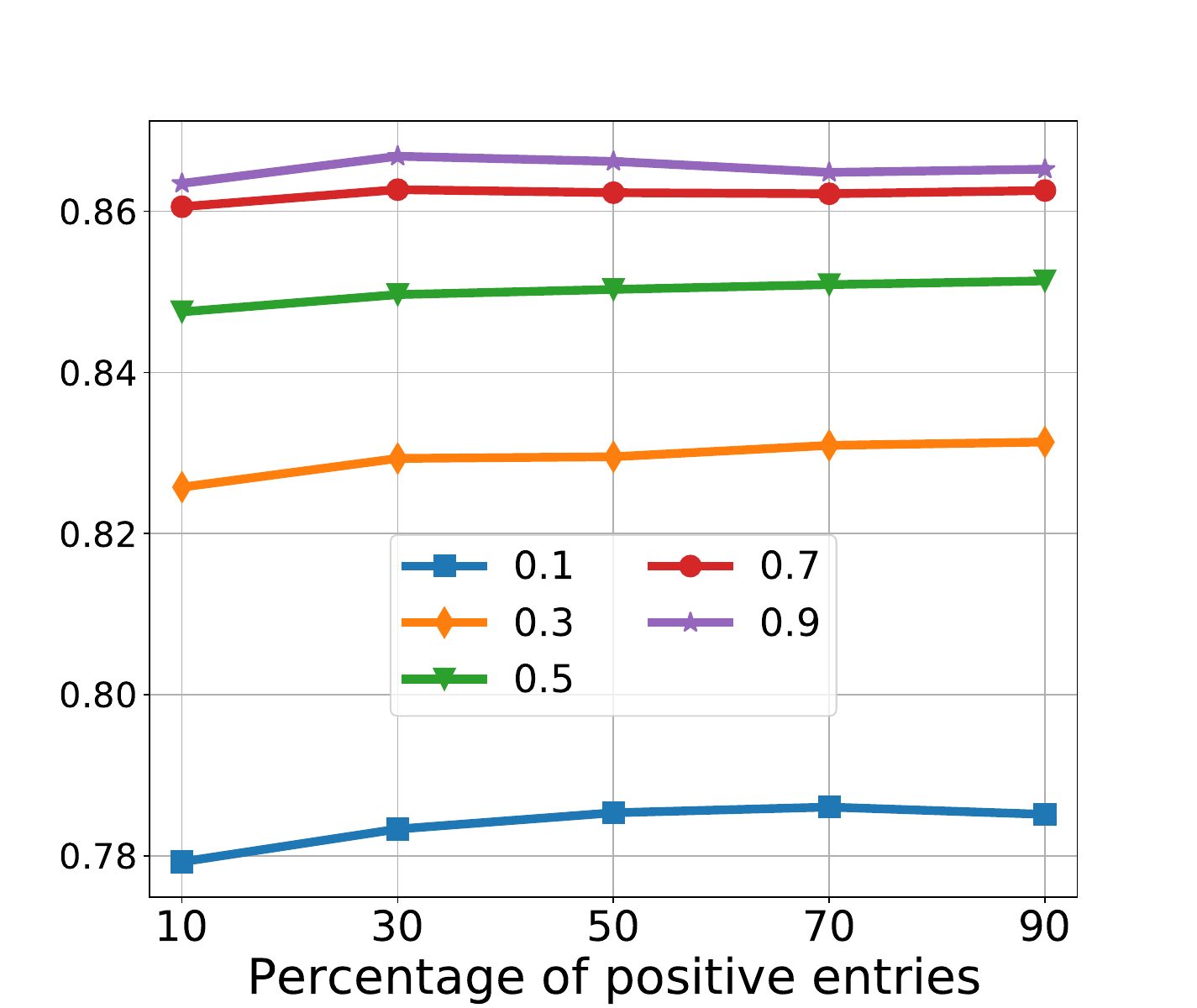}
\caption{}
\end{subfigure}
\begin{subfigure}[t]{0.24\textwidth}
\centering
\includegraphics[scale=0.2]{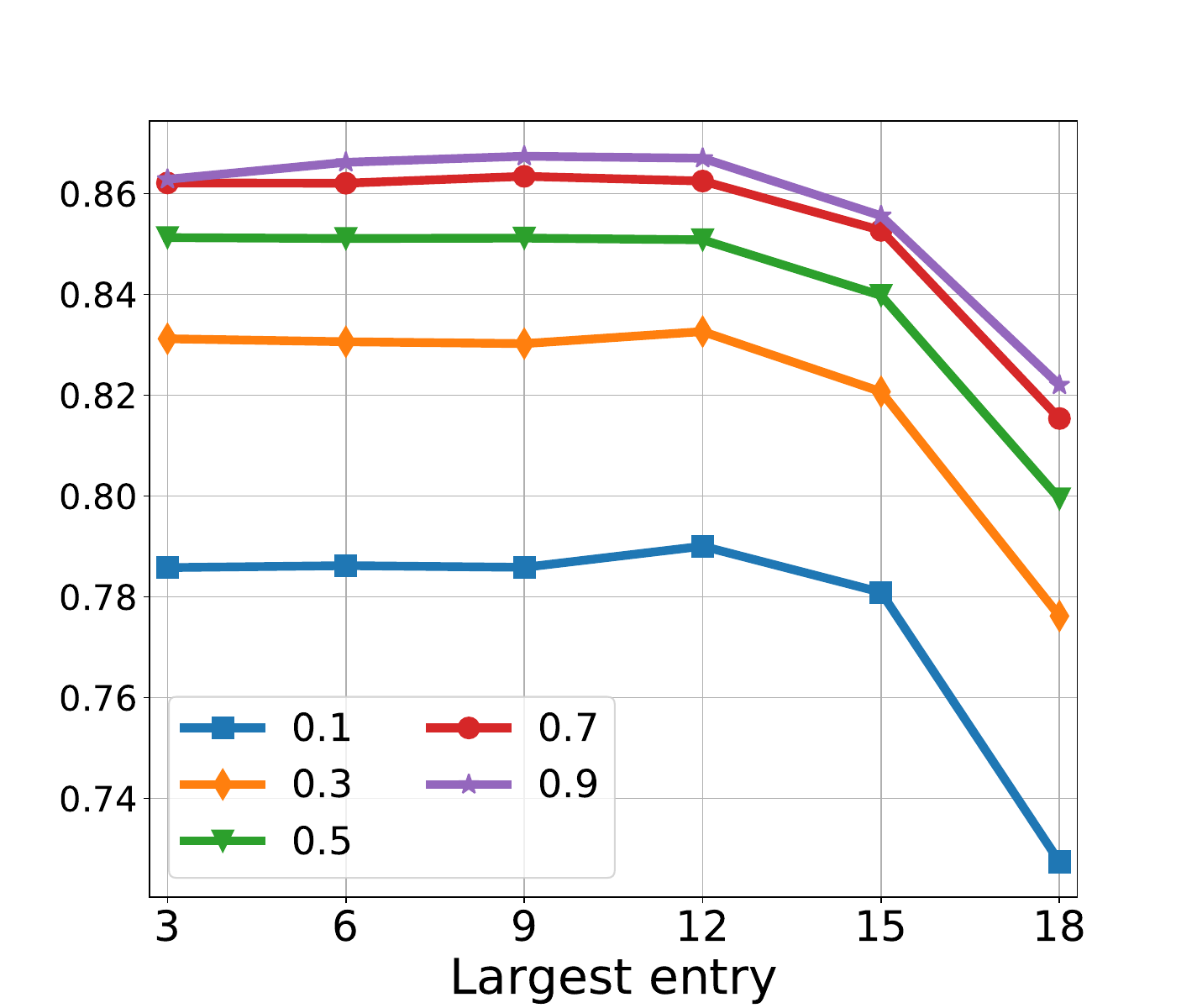}
\caption{}
\end{subfigure}
\begin{subfigure}[t]{0.24\textwidth}
\centering
\includegraphics[scale=0.2]{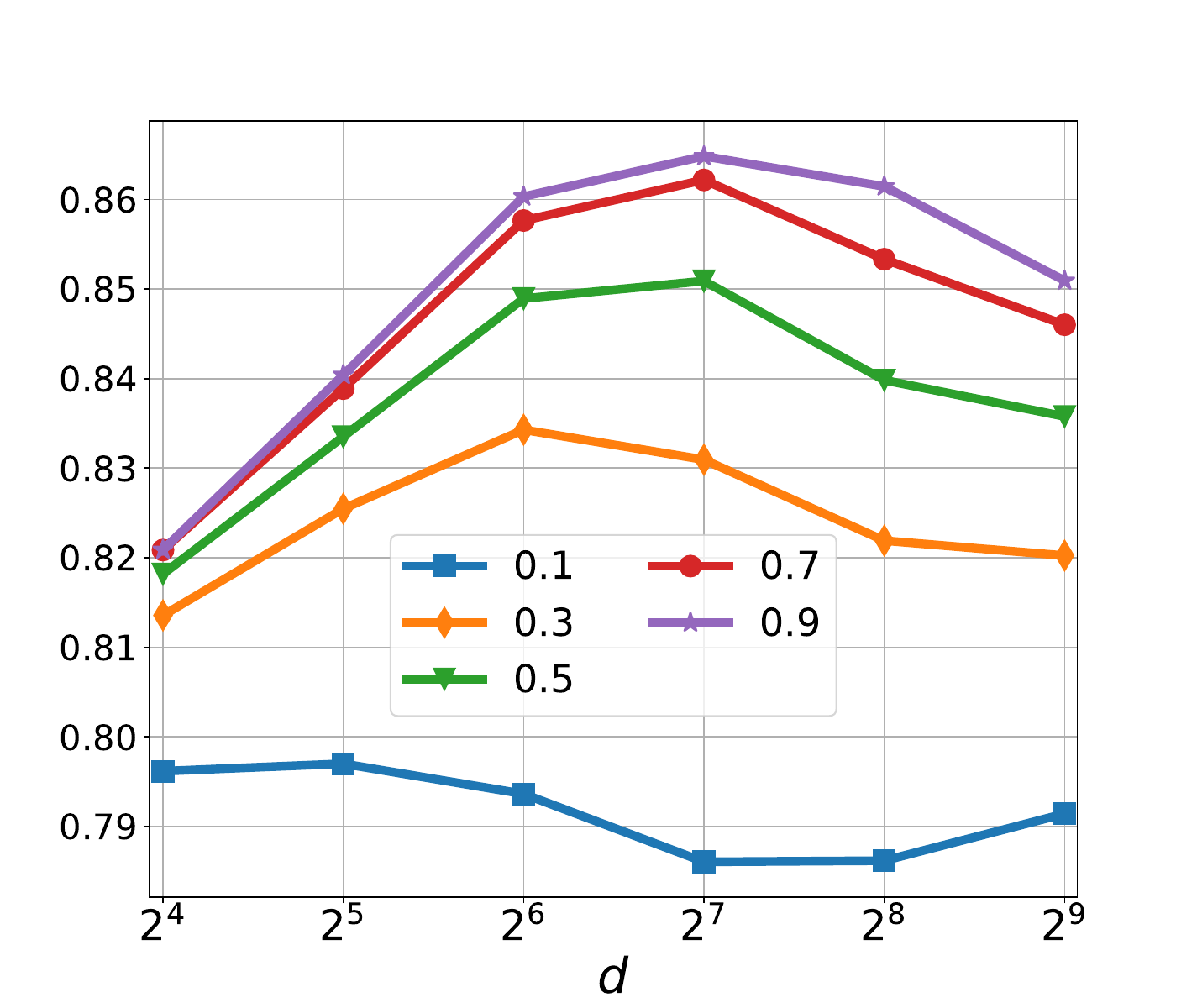}
\caption{}
\end{subfigure}
\begin{subfigure}[t]{0.24\textwidth}
\centering
\includegraphics[scale=0.2]{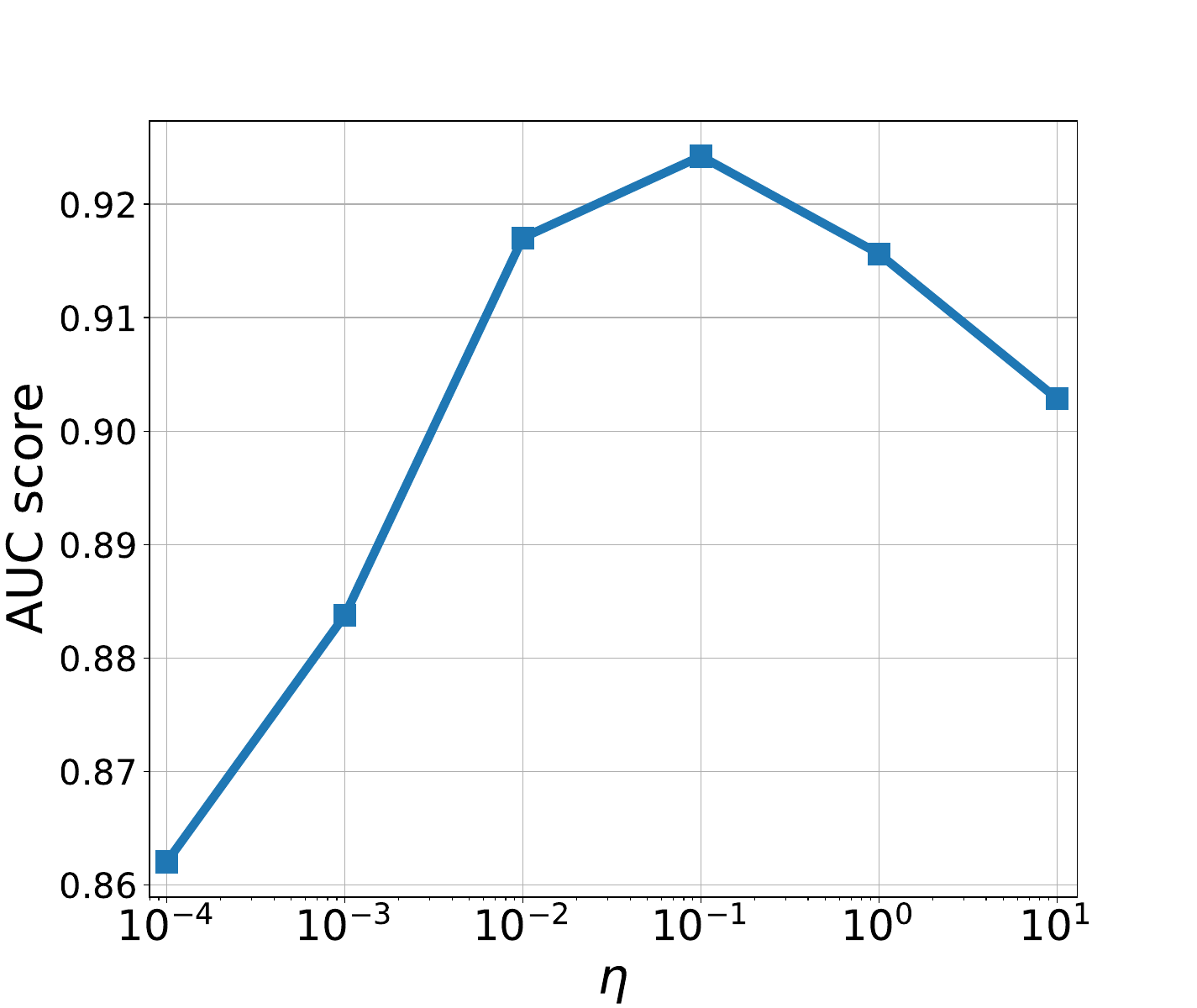}
\caption{}
\end{subfigure}
\begin{subfigure}[t]{0.24\textwidth}
\centering
\includegraphics[scale=0.2]{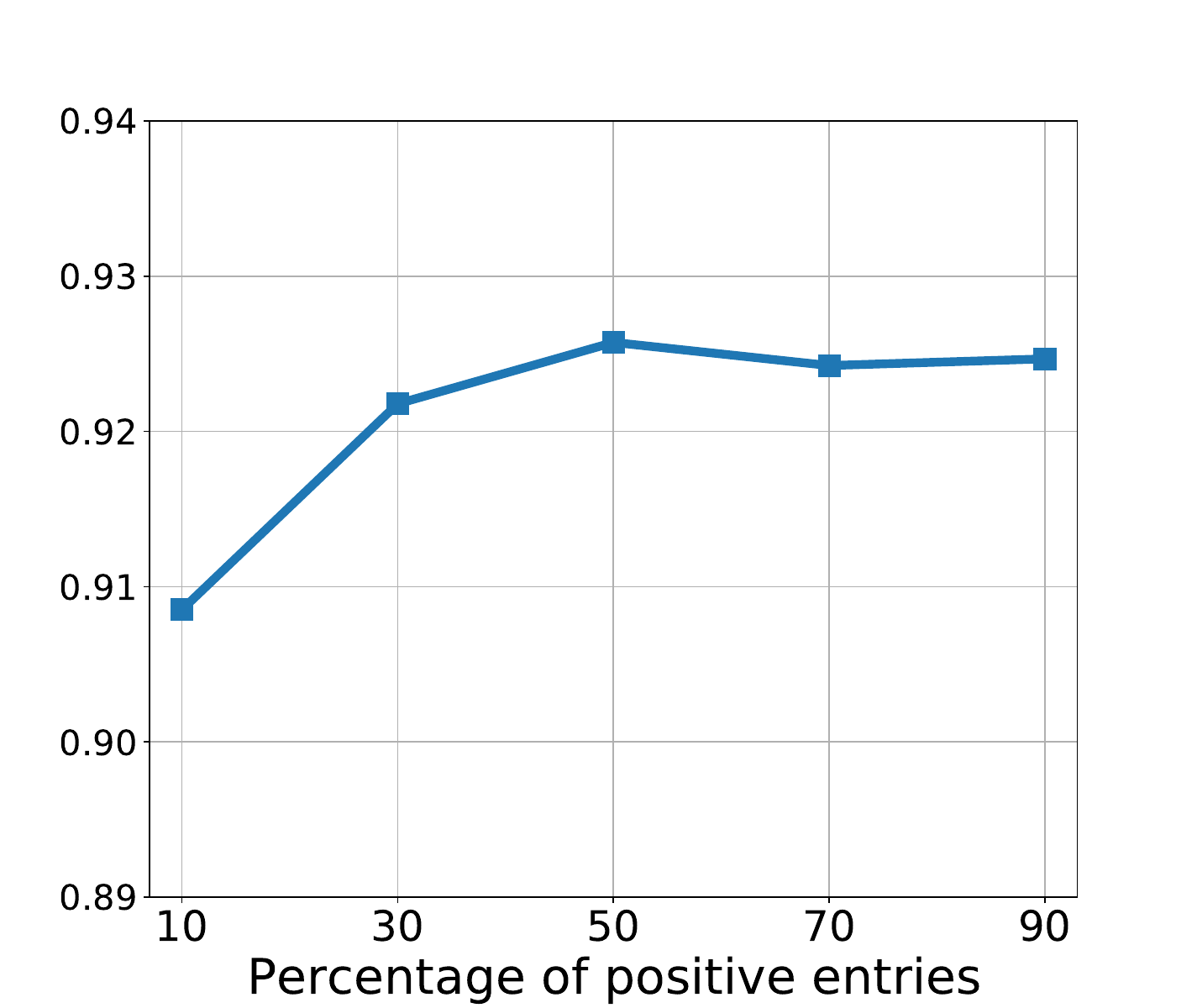}
\caption{}
\end{subfigure}
\begin{subfigure}[t]{0.24\textwidth}
\centering
\includegraphics[scale=0.2]{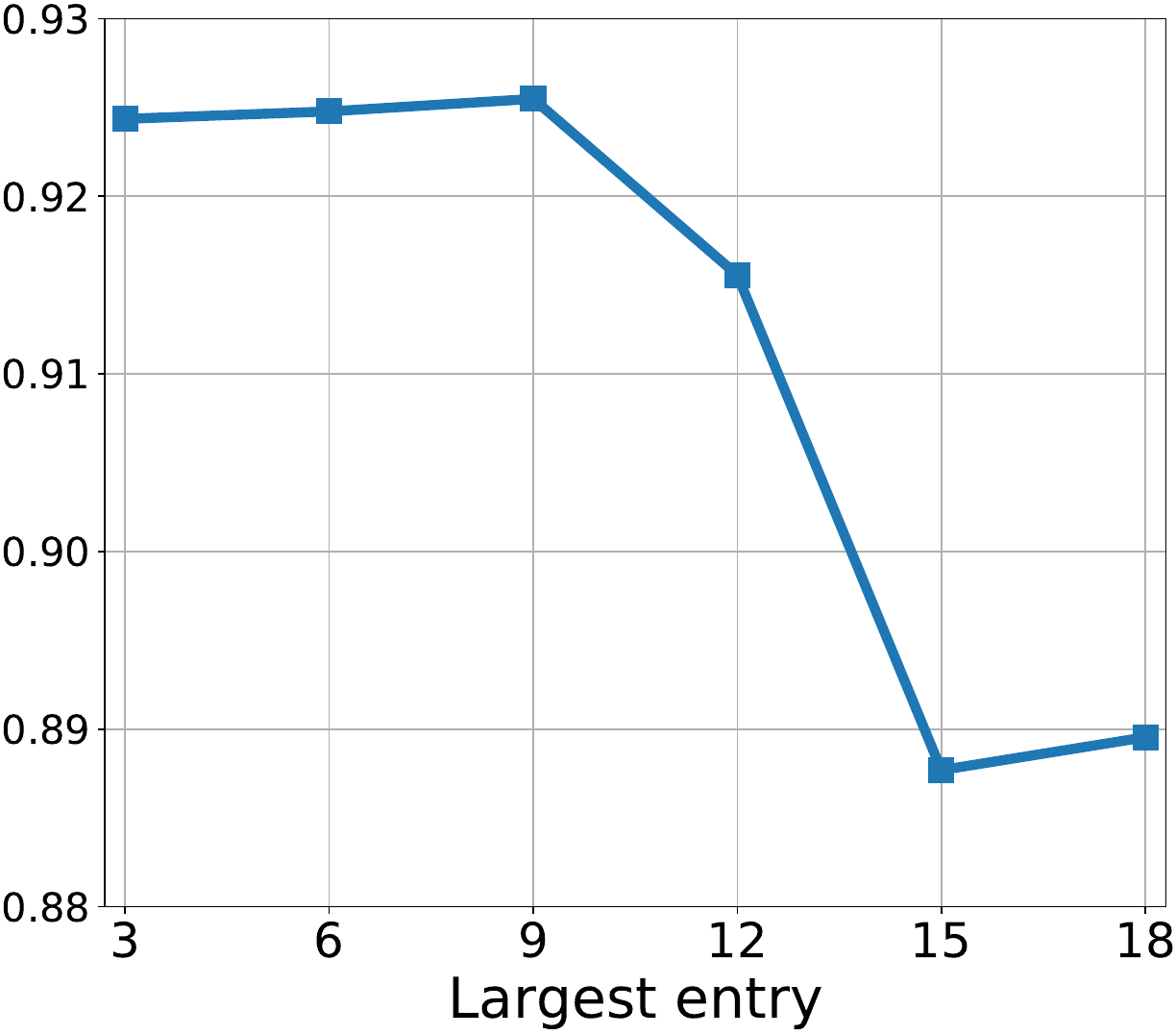}
\caption{}
\end{subfigure}
\begin{subfigure}[t]{0.24\textwidth}
\centering
\includegraphics[scale=0.2]{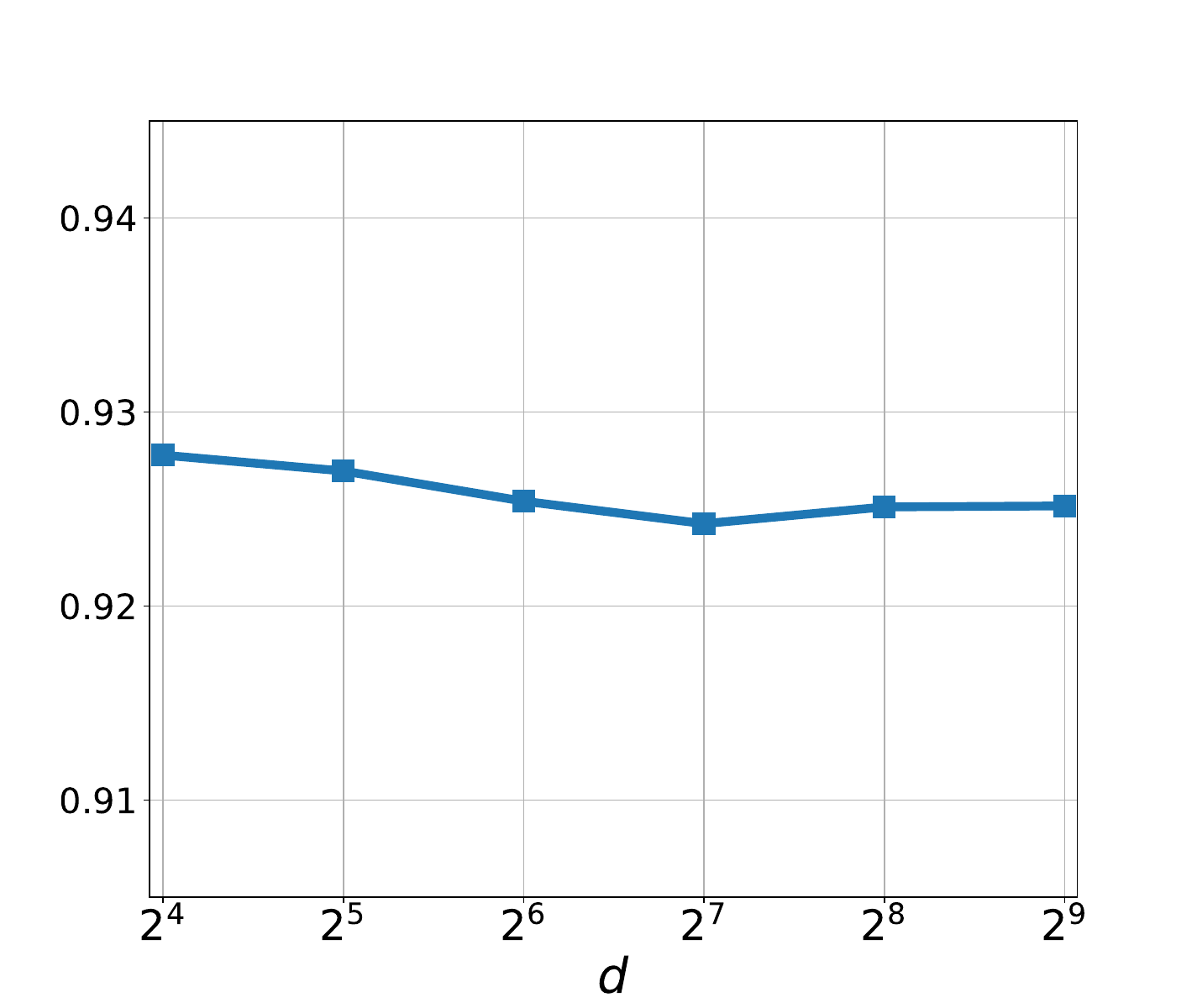}
\caption{}
\end{subfigure}
\begin{subfigure}[t]{0.24\textwidth}
\centering
\includegraphics[scale=0.2]{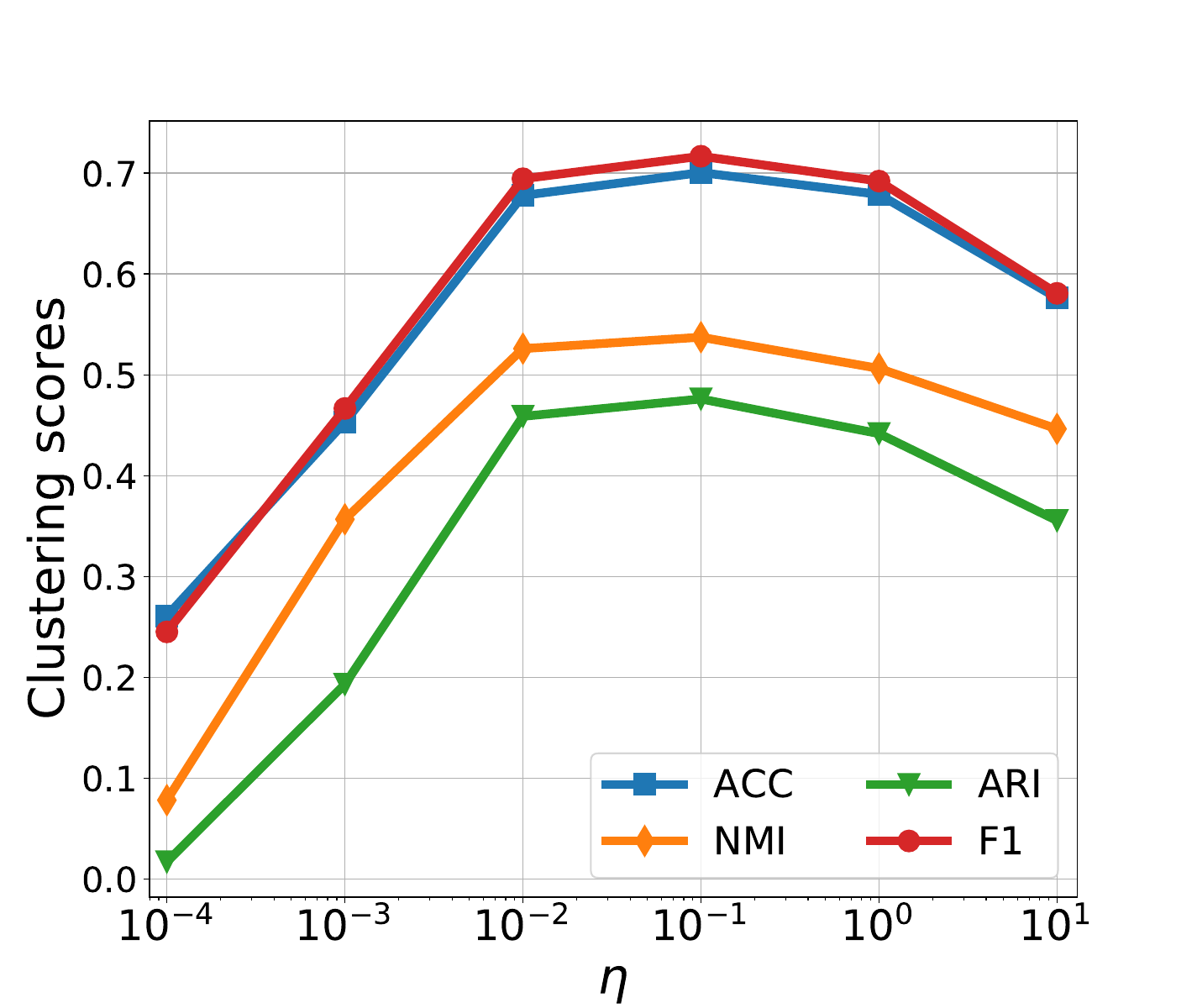}
\caption{}
\end{subfigure}
\begin{subfigure}[t]{0.24\textwidth}
\centering
\includegraphics[scale=0.2]{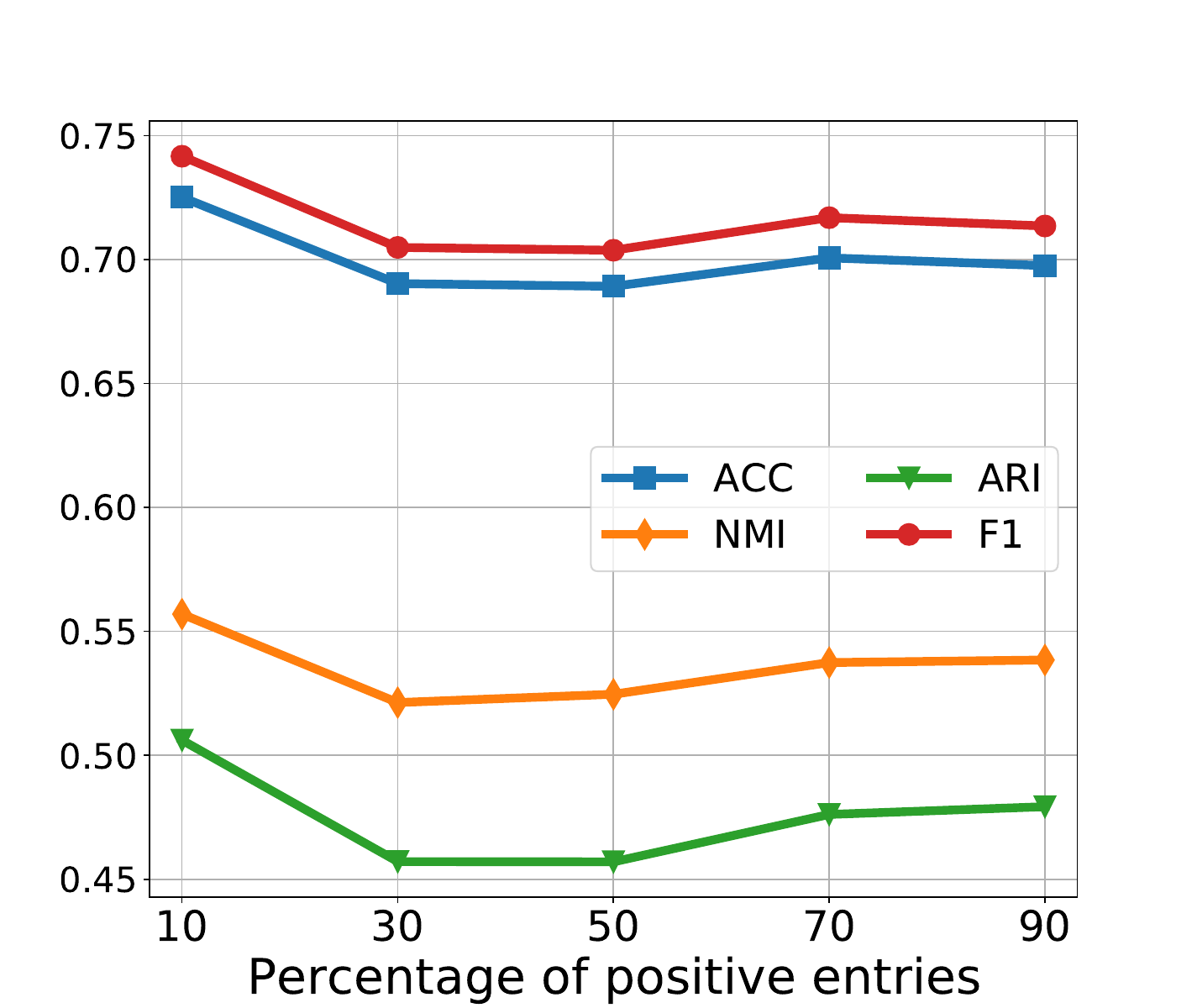}
\caption{}
\end{subfigure}
\begin{subfigure}[t]{0.24\textwidth}
\centering
\includegraphics[scale=0.2]{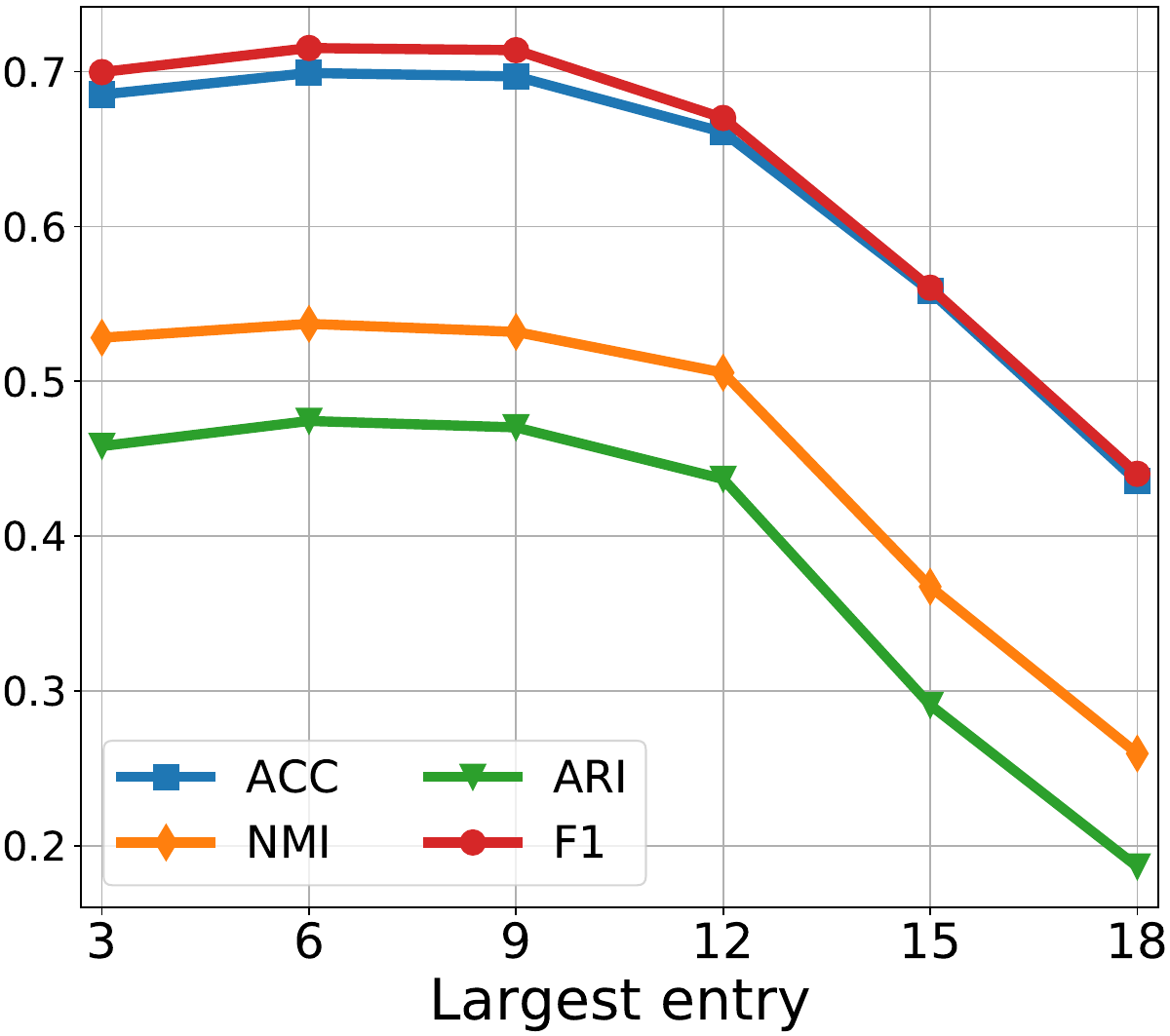}
\caption{}
\end{subfigure}
\begin{subfigure}[t]{0.24\textwidth}
\centering
\includegraphics[scale=0.2]{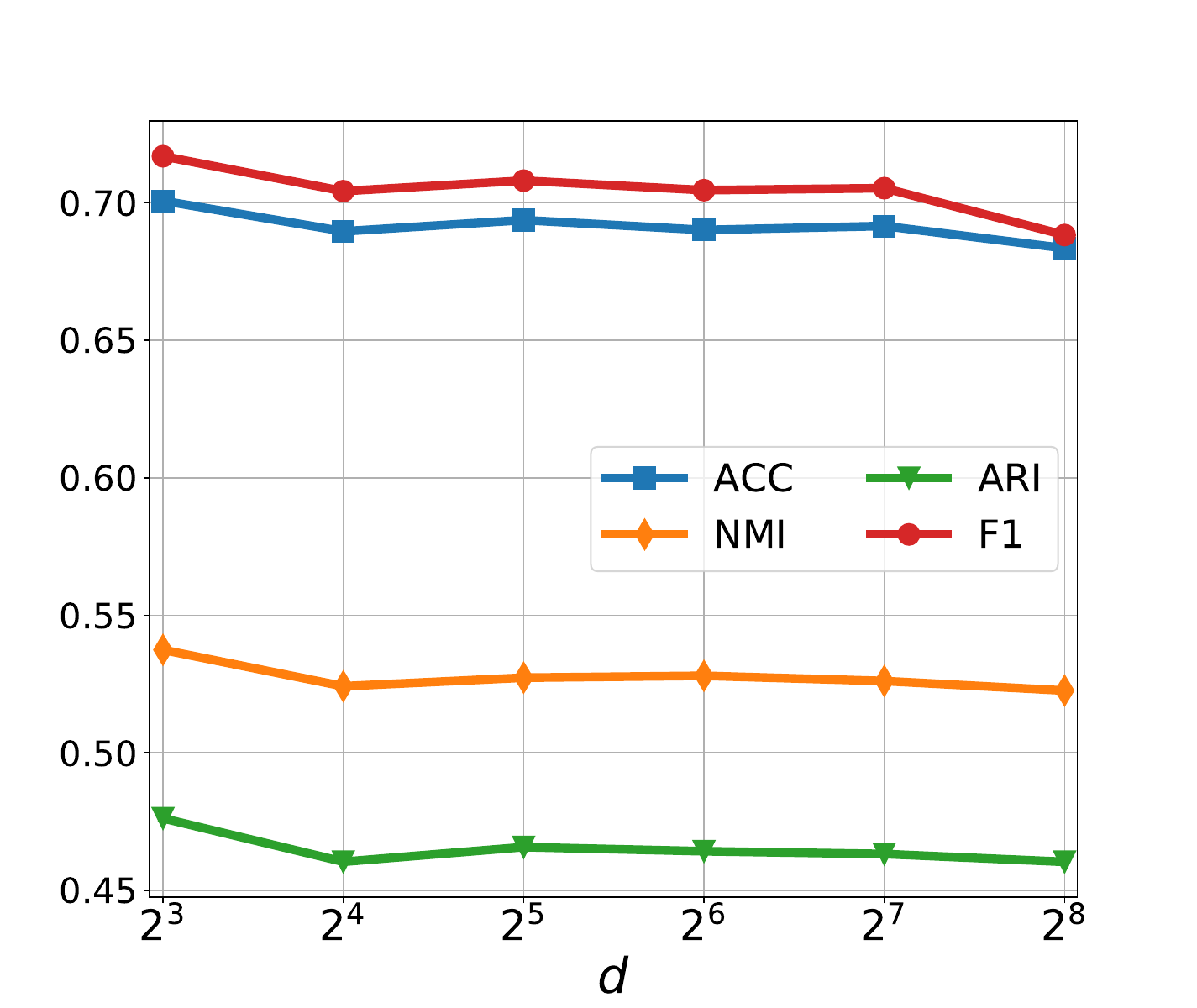}
\caption{}
\end{subfigure}
\caption{\small Parameter sensitivity test on the Cora network. (a)-(d) are for node classification where the legend denotes the fraction of labeled nodes. (e)-(h) are for link prediction where the Hadamard operator is applied. (i)-(l) are for node clustering.}
\label{fig:cora_para}
\end{figure*}

\begin{figure*}[!htbp]
\centering
\begin{subfigure}[t]{0.24\textwidth}
\centering
\includegraphics[scale=0.2]{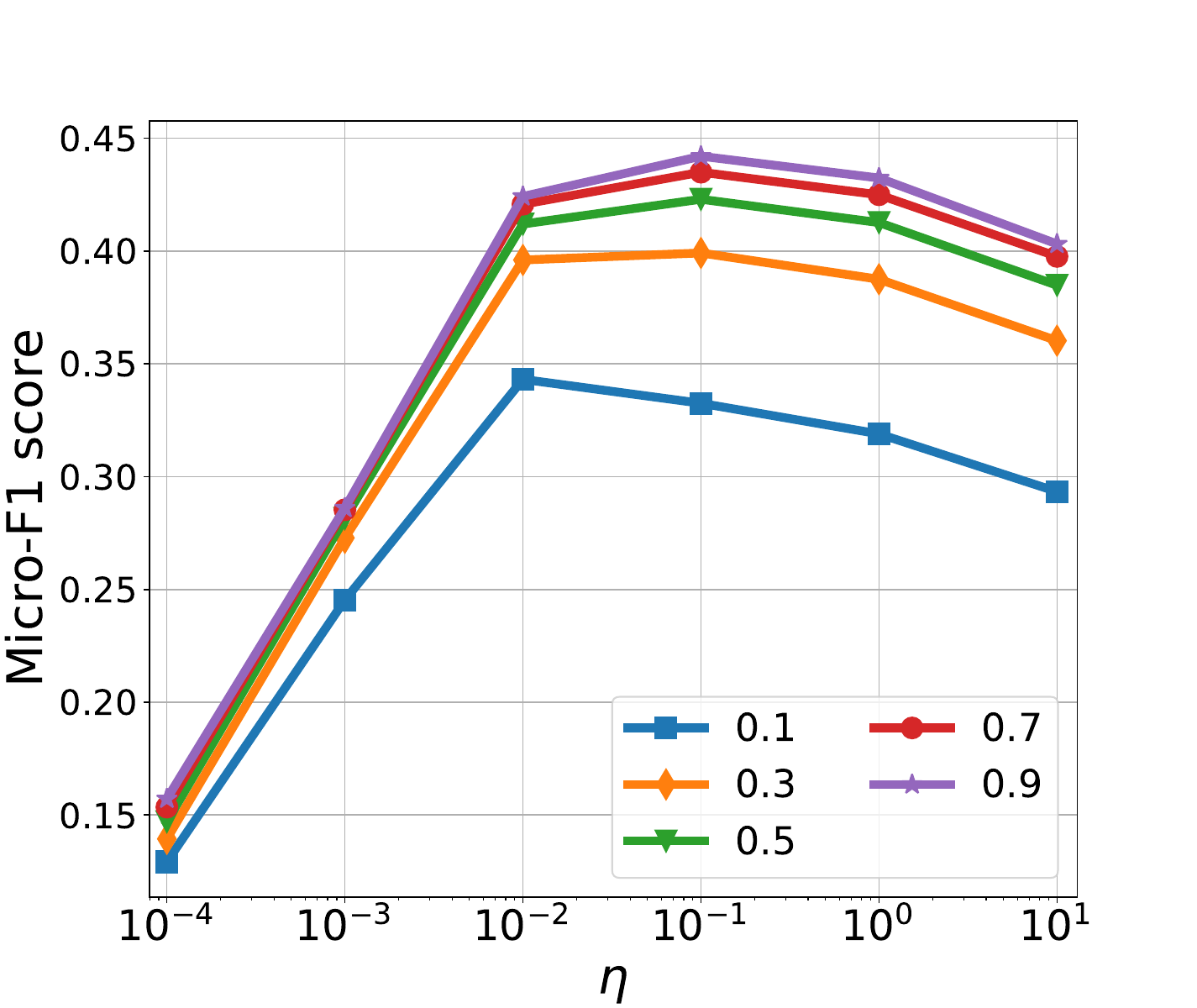}
\caption{}
\end{subfigure}
\begin{subfigure}[t]{0.24\textwidth}
\centering
\includegraphics[scale=0.2]{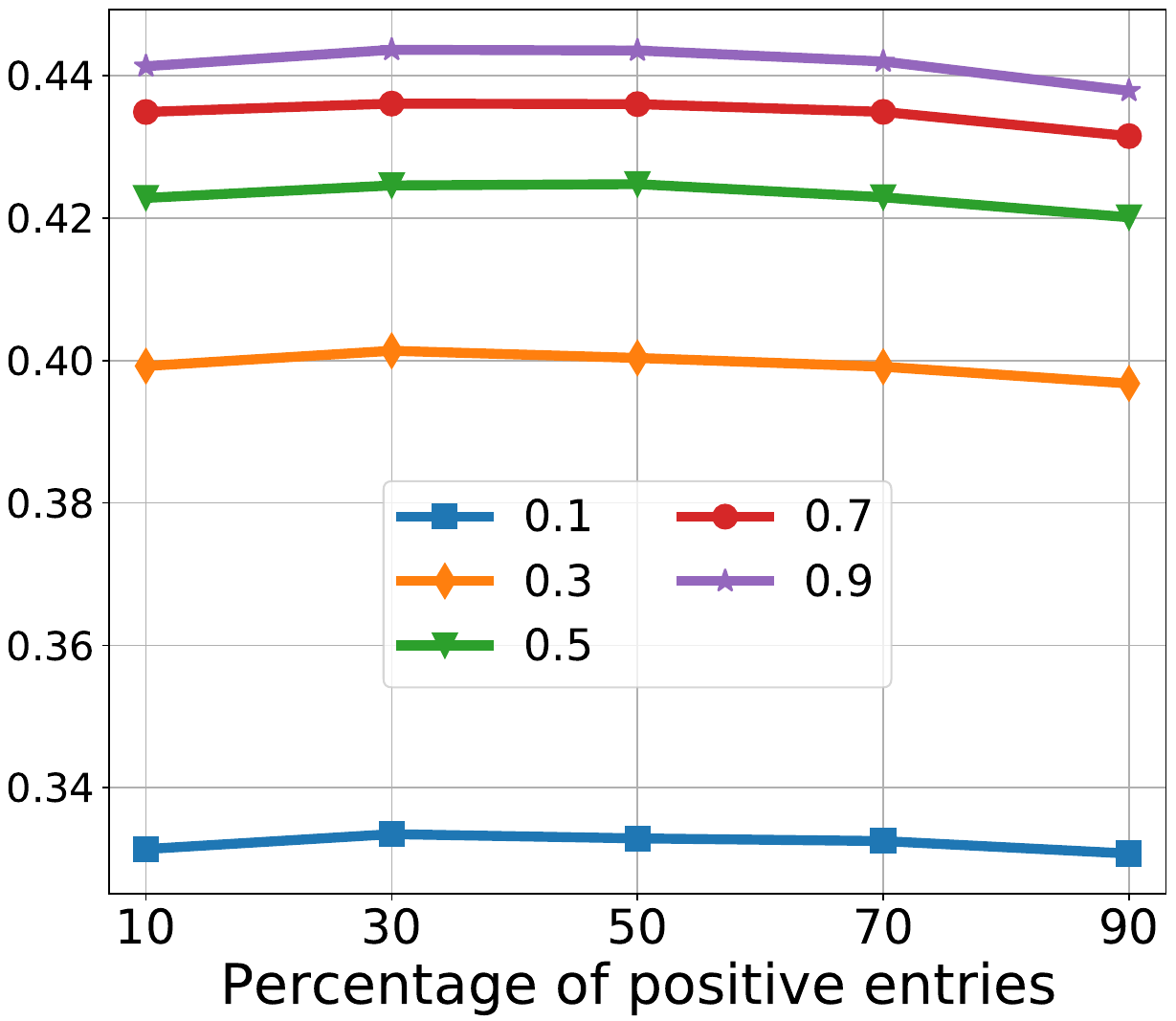}
\caption{}
\end{subfigure}
\begin{subfigure}[t]{0.24\textwidth}
\centering
\includegraphics[scale=0.2]{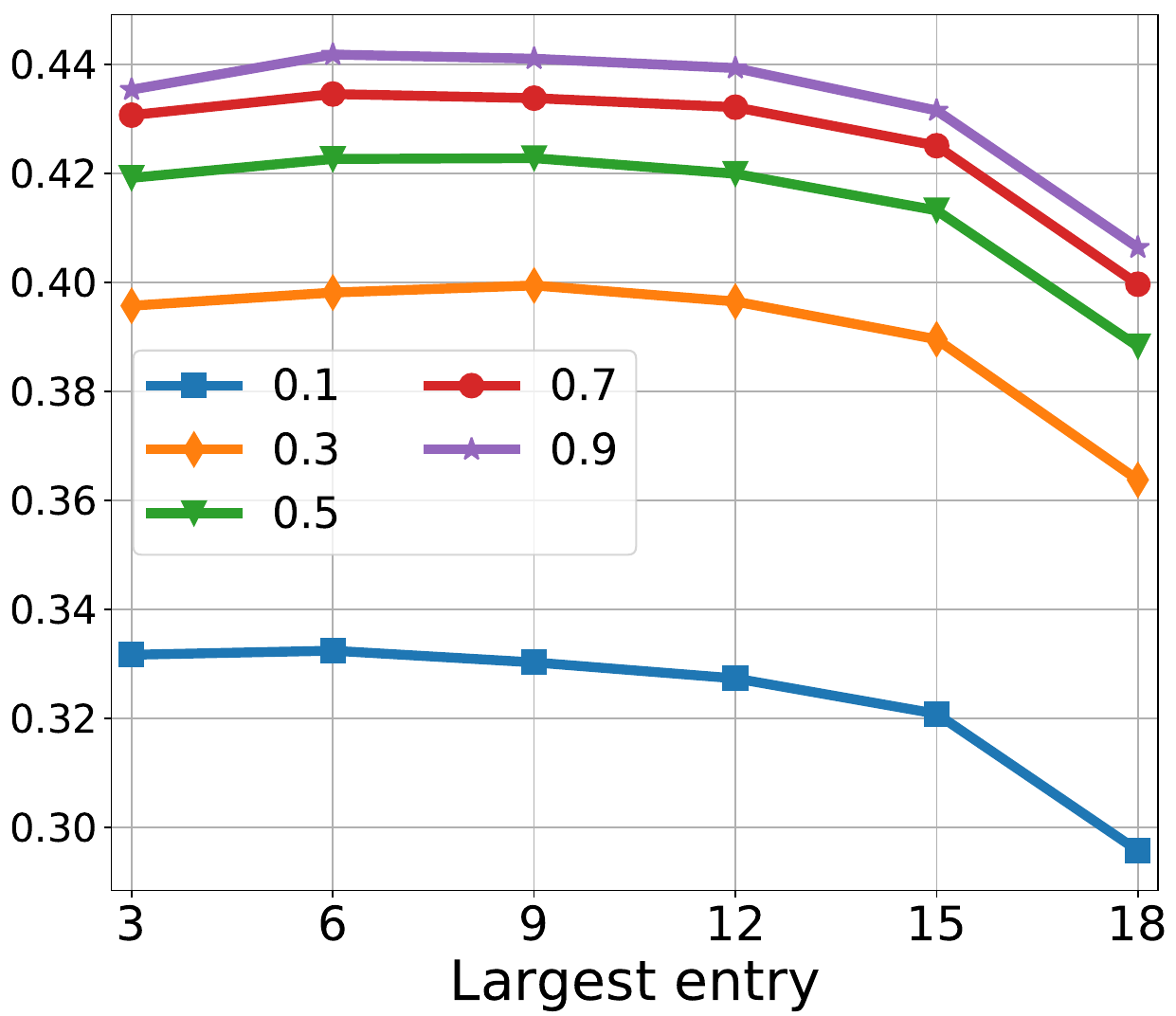}
\caption{}
\end{subfigure}
\begin{subfigure}[t]{0.24\textwidth}
\centering
\includegraphics[scale=0.2]{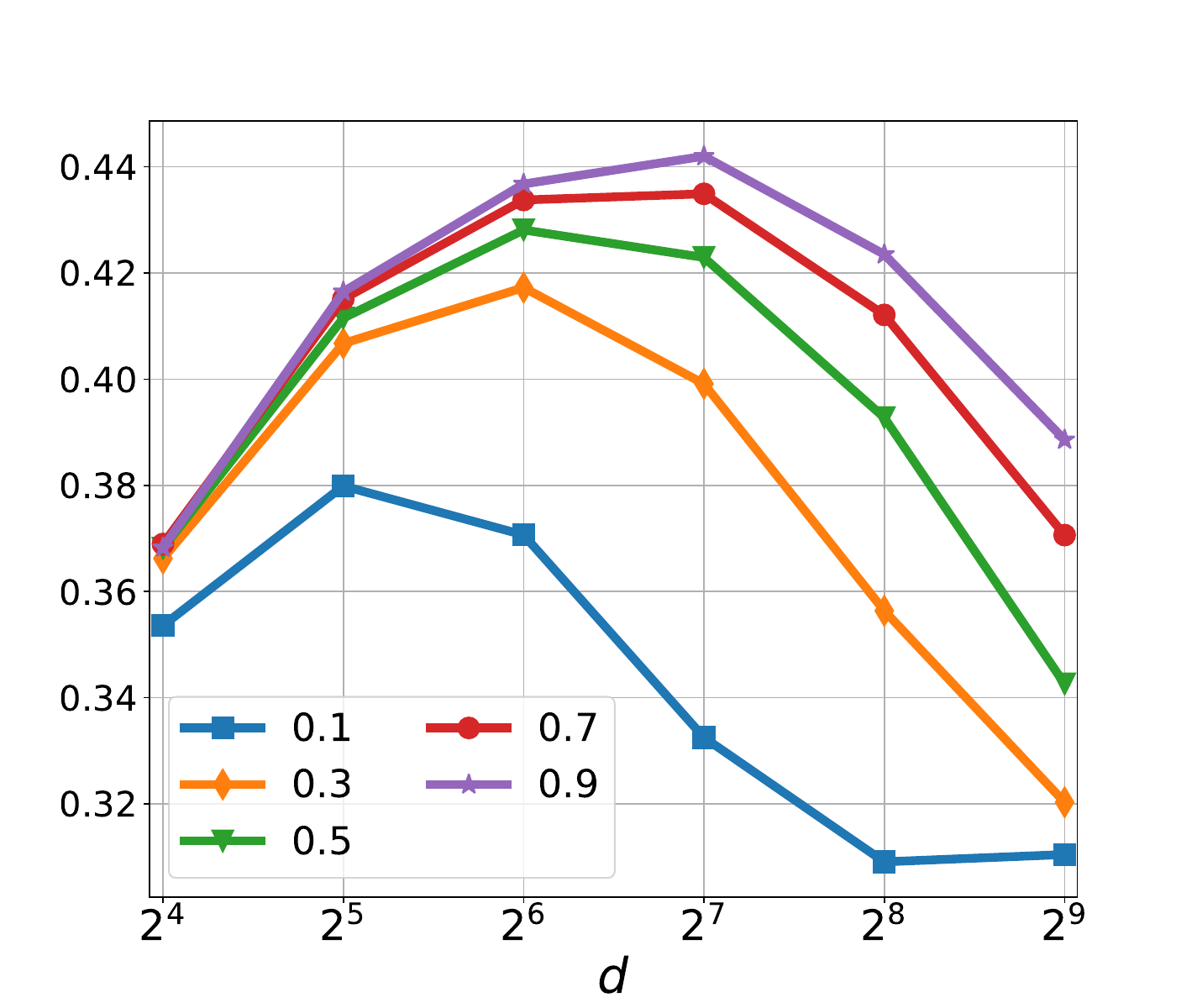}
\caption{}
\end{subfigure}
\begin{subfigure}[t]{0.24\textwidth}
\centering
\includegraphics[scale=0.2]{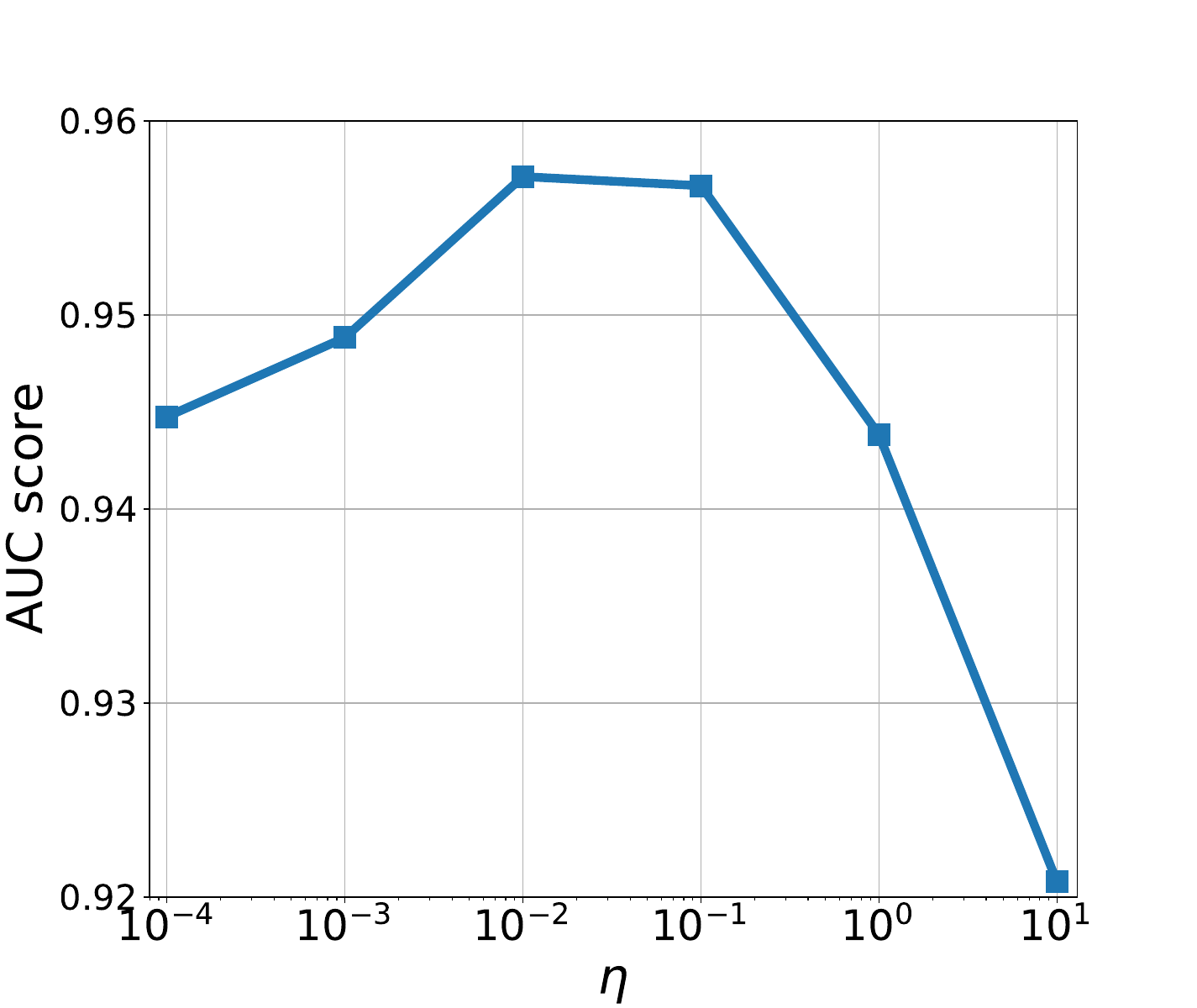}
\caption{}
\end{subfigure}
\begin{subfigure}[t]{0.24\textwidth}
\centering
\includegraphics[scale=0.2]{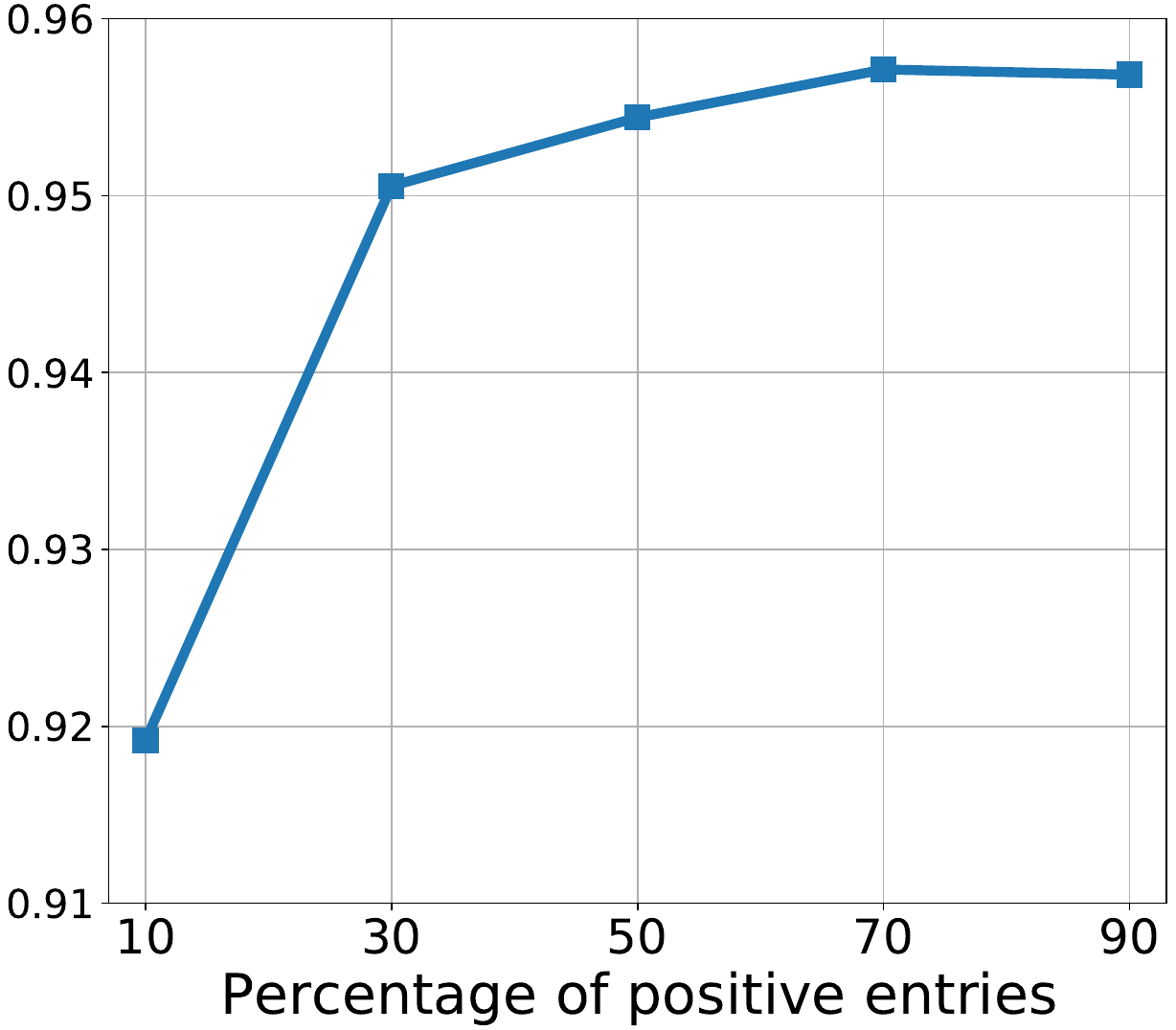}
\caption{}
\end{subfigure}
\begin{subfigure}[t]{0.24\textwidth}
\centering
\includegraphics[scale=0.2]{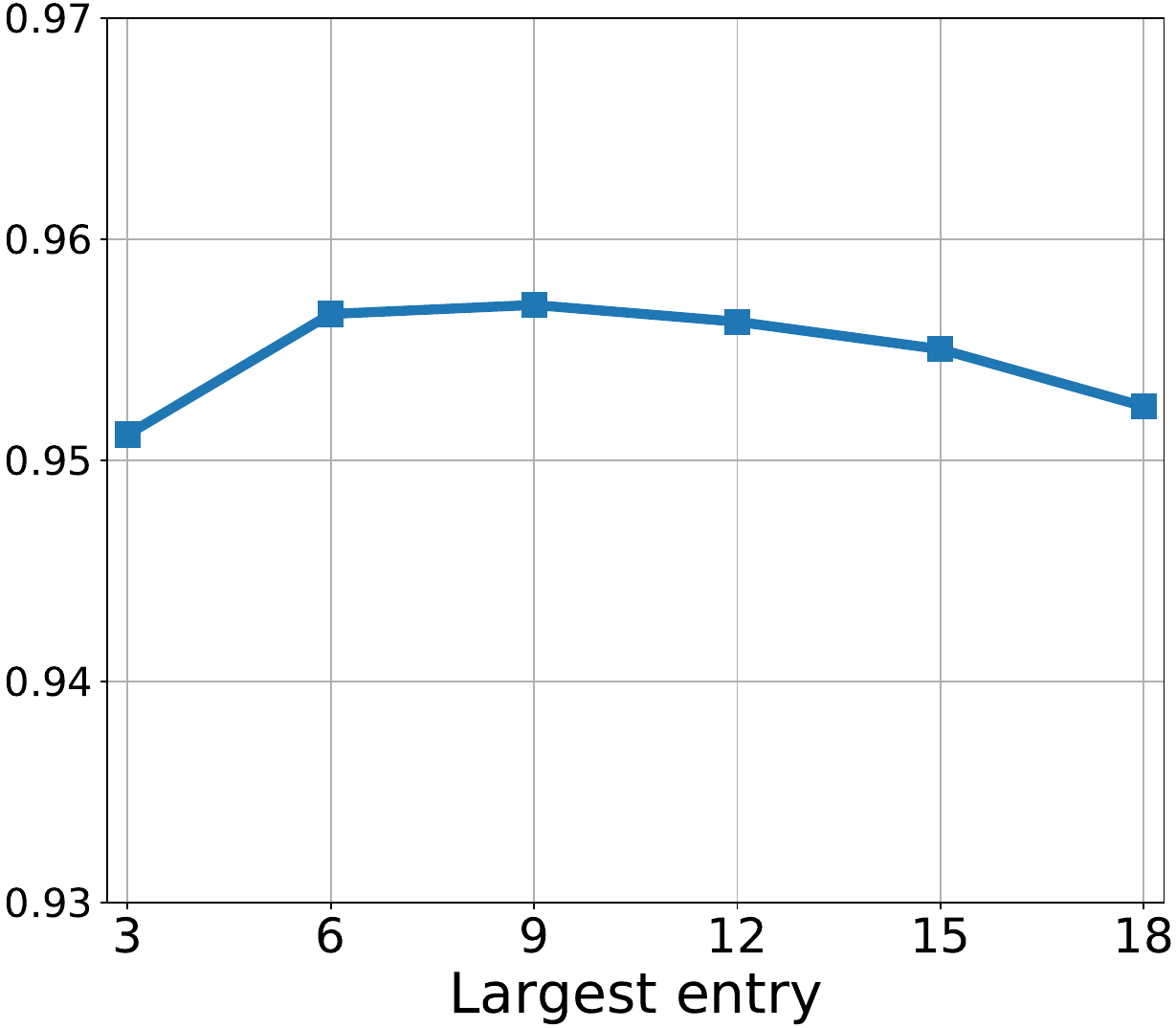}
\caption{}
\end{subfigure}
\begin{subfigure}[t]{0.24\textwidth}
\centering
\includegraphics[scale=0.2]{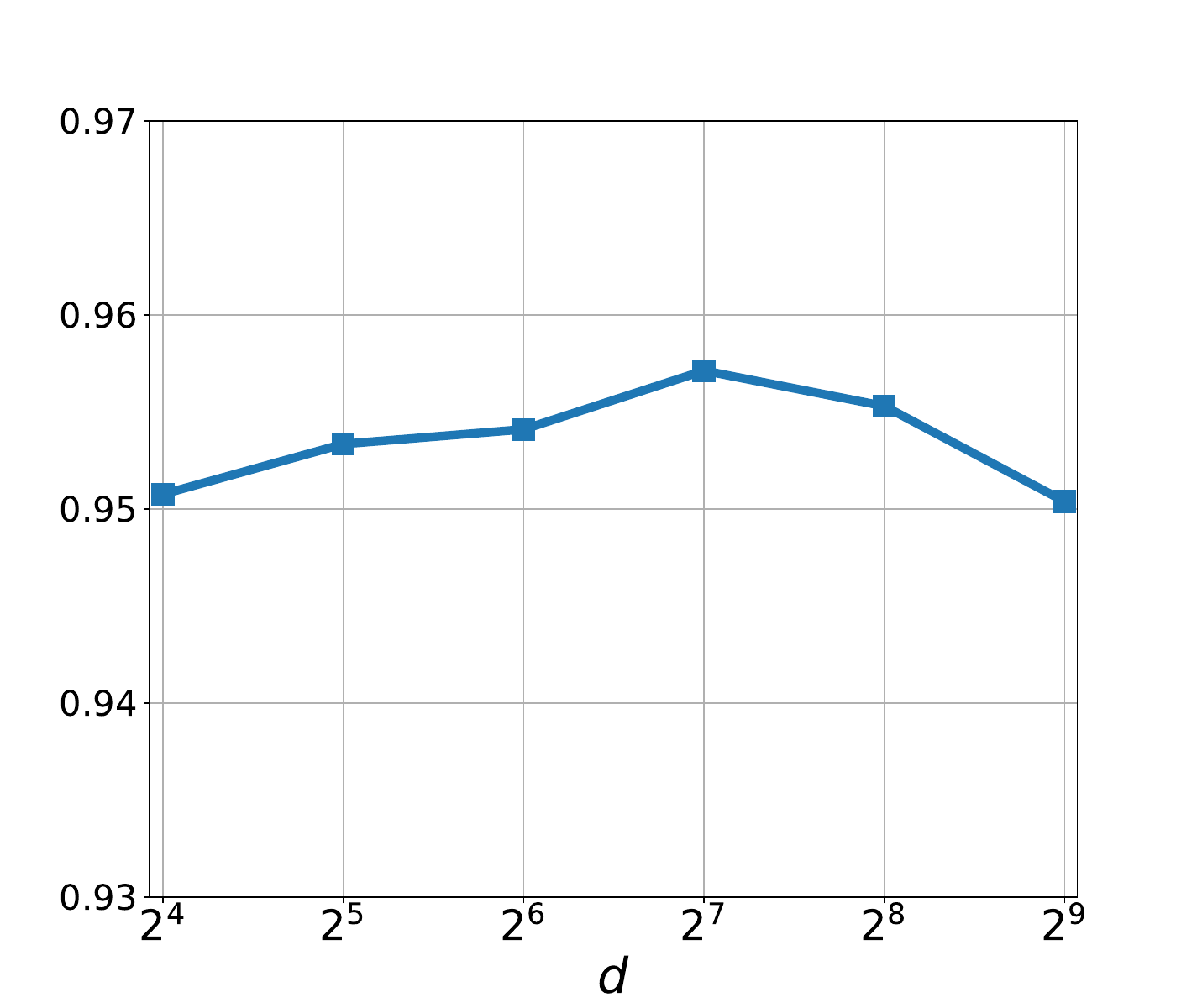}
\caption{}
\end{subfigure}
\caption{\small Parameter sensitivity test on the BlogCatalog network. (a)-(d) are for node classification where the legend denotes the fraction of labeled nodes. (e)-(h) are for link prediction where the Hadamard operator is applied. (i)-(l) are for node clustering.}
\label{fig:blogcatalog_para}
\end{figure*}

\end{titlepage}

\end{document}